%% file: rate_distortion.tex
\documentclass[11pt]{article}

\usepackage{color}
\usepackage{rate-distortion}
\usepackage{hyperref}

\usepackage[numbers]{natbib}

\usepackage{fullpage}

\begin{document}

\vspace{1cm}
\begin{center}
  {\LARGE {\bf{Privacy Aware Learning}}} \\
  \vspace{.5cm}
  {\large
    John C.\ Duchi$^1$ ~~~~~~
    Michael I.\ Jordan$^{1,2}$ ~~~~~~
    Martin J.\ Wainwright$^{1,2}$
  } \\
  \vspace{.2cm}
  {\tt \{jduchi,jordan,wainwrig\}@eecs.berkeley.edu} \\
  \vspace{.2cm}
  {\large
    $^1$Department of Electrical Engineering and Computer Science \\
    \vspace{.1cm}
    $^2$Department of Statistics \\ \vspace{.1cm}
    University of California, Berkeley
  }

\vspace*{.2in}
September 2013
\vspace*{.2in}
\end{center}

\begin{center}
  \begin{abstract}
    We study statistical risk minimization problems under a privacy
    model in which the data is kept confidential even from the learner.
    In this local privacy framework, we establish sharp upper and lower
    bounds on the convergence rates of statistical estimation
    procedures.  As a consequence, we exhibit a precise tradeoff between
    the amount of privacy the data preserves and the utility, as
    measured by convergence rate, of any statistical estimator or
    learning procedure.
  \end{abstract}
\end{center}


\input{introduction}

\input{local-privacy}

\input{statistical-rates}

\input{saddle-points}

\input{statistical-rate-proofs}

\input{open-issues}

\setlength{\bibsep}{1pt}

\bibliographystyle{abbrvnat}
\bibliography{bib}

\appendix

\input{unbiased}

\input{mutual-info-calculations}

\input{sgd-attainability}

\input{background-on-conditional-probability}

\input{saddle-point-proofs}

\input{lone-asymptotic-expansion}

\end{document}

%% file: introduction.tex
\section{Introduction}

Natural tensions between learning and privacy arise whenever a learner
must aggregate data across multiple individuals.  The learner wishes
to make optimal use of each data point, whereas the providers of the
data may wish to limit detailed exposure, either to the learner or to
other individuals.  A characterization of such tensions in the form of
quantitative tradeoffs is of great utility: it can inform public
discourse surrounding the design of systems that learn from data, and
the tradeoffs can be exploited as controllable degrees of freedom
whenever such a system is deployed.

In this paper, we approach this problem from the point of view of
statistical decision theory.  The decision-theoretic perspective
offers a number of advantages.  First, the use of loss functions and
risk functions provides a compelling formal foundation for defining
``learning'', one that dates back to~\citet{Wald39}, and that has seen
continued development in the context of research on machine learning
over the past two decades.  Second, by formulating the goals of a
learning system in terms of loss functions, we make it possible for
individuals to assess whether the goals of a learning system align
with their own personal utility, and thereby determine the extent to
which they are willing to sacrifice some privacy.  Third, an appeal to
decision theory permits abstraction over the details of specific
learning procedures, allowing for the derivation of minimax lower
bounds that apply to any specific procedure.  Fourth, the use of loss
functions---and more specifically, convex loss functions---in the
design of a learning system allows the powerful tools of optimization
theory to be brought to bear.  Not only are optimization-based
learning systems often successful in practice, but they are also often
amenable to theoretical analysis.  Finally, the decision-theoretic
framework is a probabilistic framework, with probabilistic models
definining the transformation from losses to risks.  This connection
provides a natural mechanism for the use of randomization to provide
control over privacy.

In more formal detail, the analysis of this paper takes place within
the following framework. Given a compact convex set $\optdomain
\subset \R^d$, we wish to find a parameter value $\optvar \in
\optdomain$ achieving good average performance under a loss function
$\loss : \statdomain \times \R^d \rightarrow \R_+$.  Here the value
$\loss(\statrv, \optvar)$ measures the performance of the parameter
vector $\optvar \in \optdomain$ on the sample $\statrv \in
\statdomain$, and $\loss(\statsample, \cdot) : \R^d \rightarrow \R_+$
is convex for $\statsample \in \statdomain$.  We measure the expected
performance of $\optvar \in \optdomain$ via the risk function
\begin{equation}
  \label{eqn:objective}
  \optvar \mapsto \risk(\optvar) \defeq \E_\statprob[\loss(\statrv, \optvar)],
\end{equation}
where the expectation is taken over some unknown distribution
$\statprob$ over the space $\statdomain$.  

In the standard formulation of statistical risk minimization, a method
$\method$ is given $n$ samples $\statrv_1, \ldots, \statrv_n$, each
drawn independently from $\statprob$, and its goal to to output an
estimate $\what{\optvar}_n$ that approximately minimizes the risk
function $\risk$.  In this paper, instead of providing the method
$\method$ with access to the samples $\statrv_1, \ldots, \statrv_n$,
however, we study the effect of giving only some disguised view
$\channelrv_i$ of each datum $\statrv_i$.  With $\what{\optvar}_n$ now
denoting an estimator based on the perturbed samples $\channelrv_i$,
we explicitly quantify the rate of convergence of
$\risk(\what{\optvar}_n)$ to $\inf_{\optvar \in \optdomain}
\risk(\optvar)$ as a function of the number of samples $n$ and the
amount of privacy provided by $\channelrv_i$.

\subsection{Prior work}

There is a long history of research at the intersection of privacy and
statistics, going back at least to the 1960s, when~\citet{Warner65}
suggested privacy-preserving methods for survey sampling, and to later
work related to census taking and presentation of tabular data (e.g.,
\cite{Fellegi72}). More recently, there has been a large amount of
computationally-oriented work on
privacy~\cite{DworkMcNiSm06,Dwork08,ZhouLaWa09,WassermanZh10,HallRiWa11,DinurNi03,BlumLiRo08,ChaudhuriMoSa11,RubinsteinBaHuTa12}.
We overview some of the key ideas in this section, but cannot hope to
do justice to the large body of relevant work, referring the reader to
the comprehensive survey by~\citet{Dwork08} and the statistical
treatment by~\citet{WassermanZh10} for background and references.

Most work on privacy attempts to limit disclosure risk: the
probability that some adversary can link a released record to a
particular member of the population or identify that someone belongs
to a dataset that generates a
statistic~\cite{DuncanLa86,DuncanLa89,Reiter05,KarrKoOgReSa06}. In the
statistical literature, work on disclosure limitation and so-called
linkage risk, for example as in the framework of~\citet{DuncanLa86},
has yielded several techniques for maintaining privacy, such as
aggregation, swapping features or responses among different datums, or
perturbation of data.  Other authors have proposed measures for
measuring utility of released data (e.g.,
\cite{KarrKoOgReSa06,CoxKaKi11}). The currently standard measure of
privacy is differential privacy, due to~\citet{DworkMcNiSm06}, which
roughly states that $\what{\optvar}_n$ must not depend too much on the
$n$ samples, and it should be difficult to ascertain whether a vector
$\statsample$ belongs to the set $\{\statrv_1, \ldots, \statrv_n\}$
given $\what{\optvar}_n$.  Formally, paraphrasing the definition of
\citet{WassermanZh10}, the method $\method$ has $\diffp$-differential
privacy if
\begin{equation}
  \label{eqn:differential-privacy}
  \sup_{S \in \sigma(\optdomain)}
  \sup_{x_1, \ldots, x_n} \sup_{x_1', \ldots, x_n'}
  \frac{\channelprob(S \mid \statrv_1 = x_1, \ldots, \statrv_n = x_n)}{
    \channelprob(S \mid \statrv_1 = x_1', \ldots, \statrv_n = x_n')}
  \le \exp(\diffp).
\end{equation}
where the sets $x_1, \ldots, x_n$ and $x_1', \ldots, x_n'$ differ in at most
one element, $\channelprob(\cdot \mid \statrv_1, \ldots, \statrv_n)$ is (a
version of) the conditional probability of the estimator $\what{\optvar}$
constructed by the method $\method$ using the $n$ samples, and
$\sigma(\optdomain)$ is a suitable $\sigma$-algebra on $\optdomain$.

Differentially private algorithms enjoy many desirable
properties~\cite{DworkMcNiSm06,Dwork08,GantaKaSm08} and essentially guarantee
that even if an adversary knows all the entries in a dataset but the $n$th, it
is difficult to discern whether a vector $x$ is equal to $\statrv_n$ given the
output of the method $\method$. Indeed, differential privacy protects against
side information and many adversarial attacks that break previous definitions
of privacy, such as $k$-anonymity~\cite{GantaKaSm08}.  Several researchers
have studied differentially private algorithms for empirical risk
minimization, providing guarantees on the excess risk of differentially
private estimators $\what{\optvar}$. \citet{ChaudhuriMoSa11} use the stability
of the output of regularized empirical risk minimization algorithms to show
that by adding Laplace-distributed noise to an empirical estimator $\optvar$
or by adding an additional random term to the empirical risk
$\frac{1}{n}\sum_{i=1}^n \loss(\statrv_i, \optvar)$, it is possible to obtain
differential privacy and consistency of $\what{\optvar}$.  \citet{DworkLe09}
obtain similar results using robust statistical estimators, and
\citet{Smith11} shows that if one has suitably unbiased estimators, then
differential privacy is possible without compromising asymptotic rates of
convergence. \citet{RubinsteinBaHuTa12} use similar stability and perturbation
techniques to demonstrate that it is possible to obtain differential privacy
when solving support vector machine problems, and also show that if the
desired privacy level $\alpha$ in the
definition~\eqref{eqn:differential-privacy} is too small, it is actually
impossible to obtain a parameter $\what{\optvar}_n$ minimizing the risk
$\risk$.

Our goal is to understand the fundamental tradeoffs between maintaining
privacy while still providing a useful output from the statistical learning
procedure $\method$. Though intuitively there must be some tradeoff,
quantifying it precisely has been difficult. As alluded to above,
\citet{RubinsteinBaHuTa12} are able to show that it is impossible to obtain
what they call an $(\epsilon, \delta)$-useful parameter vector $\optvar$ that
enjoys any differential privacy guarantees; however, it is unknown whether or
not their guarantees might be improvable. \citet{HallRiWa11} show that if a
given histogram, based on a sample $\{\statsample_i\}_{i=1}^n$, has $d$ bins
and we must guarantee $\alpha$-differential
privacy~\eqref{eqn:differential-privacy}, then the (expected) $L^1$-distance
between the sample and released histograms must be at least $d / (n \alpha)$,
and \citet{HardtTa10} give similar lower bounds on the amount of noise
necessary to answer linear database queries.  \citet{NikolovTaZh13} followed
this work with extensions to relaxed notions (so called
$(\diffp,\delta)$-approximate differential privacy) of privacy and providing
higher-dimensional settings, while \citet{KasiviswanathanRuSm13}
give bounds on amounts of additive noise to protect against blatant failures
of privacy in similar linear settings.  \citet{BlumLiRo08} also give lower
bounds on the closeness of certain statistical quantities computed from the
dataset, though their upper and lower bounds do not
match. \citet{SankarRaPo10} provide rate-distortion theorems for utility
models involving information-theoretic quantities, which has some similarity
to our risk-based framework, but it appears somewhat challenging to explicitly
map their setting onto ours.  With the goal of characterizing what it means to
be both useful and private, \citet{GhoshRoSu09} show that for a one-time
computation of counts on a dataset $\statrv_1, \ldots, \statrv_n$ (i.e., the
number of variables satisfying $\statrv_i \in C$ for some set $C$), perturbing
the output of a counting function using geometrically distributed noise is the
unique optimal way to guarantee differential privacy while maximizing a
natural notion of utility.

Much of the work providing sharp lower bounds, however, focuses on showing
that if one wishes to accurately report a statistic
$\what{\optvar}(\statsample_{1:n})$ computed on a sample
$\{\statsample_i\}_{i=1}^n$, then there must be some worst-case sample such
that the error is large (see, e.g.,~\cite{HardtTa10,HallRiWa11,NikolovTaZh13,GhoshRoSu09}). In contrast, we
focus on \emph{population} quantities---which are substantially different---in
that we wish to return a private estimator $\what{\optvar}_n$ approximately
minimizing the population risk $\risk(\optvar) = \E[\loss(\statrv, \optvar)]$
rather than the sample risk $\frac{1}{n} \sum_{i=1}^n \loss(\statrv_i,
\optvar)$.  Providing guarantees on the population risk performance of
$\what{\optvar}_n$, rather than on the observed sample, has been a driving
force behind much of the theoretical work in statistics and machine learning,
and thus provides a natural focus for our work.

\subsection{Our setting}

In contrast to the above work, we study a more local notion of
privacy~\cite{EvfimievskiGeSr03,KasiviswanathanLeNiRaSm11}, in which each
datum $\statrv_i$ is kept private from the method $\method$.  The goal of many
types of privacy is to guarantee that the output $\what{\optvar}_n$ of the
method $\method$ based on the data cannot be used to discover information
about the individual samples $\statrv_1, \ldots, \statrv_n$, but \emph{locally
  private} algorithms only access disguised views of each datum $\statrv_i$.
Local algorithms are among the most classical approaches to privacy, tracing
back to work on randomized response in the statistical
literature~\cite{Warner65}, and rely on communication only of some disguised
view $\channelrv_i$ of each true sample $\statrv_i$. In this setting, for
example, the natural variant of $\diffp$-differential
privacy~\eqref{eqn:differential-privacy} is the non-interactive
(in the sense that $\channelrv_i$ depends only on $\statrv_i$ and
not on any other private variables $\channelrv_j$) local privacy guarantee
\begin{equation}
  \label{eqn:local-differential-privacy}
  \sup_{S} \sup_{\statsample, \statsample'}
  \frac{\channelprob(\channelrv_i \in S \mid \statrv_i = \statsample)}{
    \channelprob(\channelrv_i \in S \mid \statrv_i = \statsample')}
  \le \exp(\diffp).
\end{equation}

Locally private algorithms are natural when the providers of the
data---the population sampled to give $\statrv_1, \ldots,
\statrv_n$---do not even trust the statistician or statistical method
$\method$, but the providers are interested in the parameter vector
$\optvar^*$ that minimizes the risk function. For example, in medical
applications, a participant may be embarrassed about his use of drugs,
or perhaps about his marital status, but if the loss $\loss$ is able
to measure the likelihood of developing cancer, then the participant
has high utility for access to the optimal parameters $\optvar^*$.
Internet applications, where a user's activity is logged across
multiple websites or searches, provide another example: the user has a
utility for a search engine to have a ranking function $\optvar$ that
returns relevant results for web searches, yet may not wish to reveal
his or her search data.  In essence, we would like the statistical
procedure $\method$ to learn \emph{from} the data $\statrv_1, \ldots,
\statrv_n$ but not \emph{about} it.

The work most related to ours seems to be that
of~\citet{KasiviswanathanLeNiRaSm11}, who show that that (in some
settings) locally private algorithms coincide with concepts that can
be learned with polynomial sample complexity in Kearns's statistical
query (SQ) model~\cite{Kearns98}. This result is powerful, but has
some limitations, as the statistical query model relies exclusively on
count queries, and we are interested in measures more precise than
polynomial sample complexity to quantity convergence rates.  In
contrast, our analysis applies to estimators deriving from a broad
class of convex risks~\eqref{eqn:objective}, and it provides sharp
rates of convergence.

We develop our approach to local privacy in the setting of three
related privacy measures. The first is a worst-case measure of mutual
information, where we view privacy preservation as a game between the
providers of the data, who wish to preserve privacy, and nature.  The
second is based on differential privacy, where the provider of each
datum communicates---subject to some constraints we make explicit
later---the \emph{most} differentially private view $\channelrv_i$ of
his or her datum $\statrv_i$.  In this general setting we allow
interactivity (i.e., the mapping between $\channelrv_i$ and
$\statrv_i$ may depend on other $\channelrv_j$ for $j \neq i$).  The
third setting is a non-interactive version of local differential
privacy.

Turning first to the information-theoretic formulation, and recalling
that the method $\method$ sees only the perturbed version
$\channelrv_i$ of $\statrv_i$, we use a uniform variant of mutual
information $\information(\channelrv_i; \statrv_i)$ between the random
variables $\statrv_i$ and $\channelrv_i$ as a measure for privacy.
Using mutual information and related information-theoretic ideas in
the privacy and security context is by no means original; see, for
example, the survey by~\citet{LiangPoSh08}.  It is important to note,
however, that standard mutual information has deficiencies as a
measure of privacy (e.g.~\cite{EvfimievskiGeSr03}).  Accordingly, our
uniform notion of mutual information is as follows: we say that the
distribution $\channelprob$ generating $\channelrv$ from $\statrv$ is
private only if $\information(\statrv; \channelrv)$ is small for all
possible distributions $\statprob$ on $\statrv$, possibly subject to
some constraints.

In this setting, we design procedures that allow consistent estimation
of the parameter $\optvar^*$ minimizing $\risk(\optvar) =
\E_\statprob[\loss(\statrv, \optvar)]$, for any convex loss $\loss$
and distribution $\statprob$ on the data $\statrv$. One central
consequence of our analysis is a sharp characterization of the
\emph{excess risk},
\begin{align}
  \Excess_n(\what{\optvar}; \loss, \optdomain) & \defn \E
  \left[\risk\left(\what{\optvar}(\channelrv_1, \ldots,
    \channelrv_n)\right)\right] - \inf_{\optvar \in \optdomain}
  \risk(\optvar),
\end{align}
associated with any estimator $\what{\optvar}$ that satisfies a pre-specified
privacy constraint. For particular collections $\lossset$ of loss functions
$\loss \in \lossset$, we bound the minimax convergence rate of all estimation
procedures.  More precisely, let us focus on $d$-dimensional problems,
i.e., those for which the domain $\optdomain \subset \R^d$ and observations
$\statrv_i \in \R^d$. If one wishes to uniformly guarantee a level of privacy
$\information(\statrv_i; \channelrv_i) \le \information^*$, then we show that
there exists a constant $a(\lossset, \optdomain) \in \R_+$---dependent only on
the properties of the collection $\lossset$ and domain $\optdomain$---such
that \emph{for any} estimator $\what{\optvar}$ for the family $\lossset$, the
excess risk is lower bounded as
\begin{subequations}
  \label{eqn:main-bounds}
  \begin{align}
    \label{eqn:main-lower-bound} 
    \sup_{\loss \in \lossset}
    \Excess_n(\what{\optvar}; \loss, \optdomain) & \geq
    \frac{\sqrt{d}}{\sqrt{\information^*}} \, \frac{a(\lossset,
      \optdomain)}{\sqrt{n}},
  \end{align}
  where $a(\lossset, \optdomain)$ is a constant characterizing the
  non-private minimax rate of estimation (see,
  e.g.,~\citet{AgarwalBaRaWa12} for such constants).
  Moreover, we also prove that there exists another constant $b(\lossset,
  \optdomain) \ge a(\lossset, \optdomain)$ and provide explicit estimators
  $\what{\optvar}$ with privacy guarantee $\information^*$ such that
  \begin{align}
    \label{eqn:main-upper-bound}
    \sup_{\loss \in \lossset}
    \Excess_n(\what{\optvar}; \loss, \optdomain) & \leq
    \frac{\sqrt{d}}{\sqrt{\information^*}} \, \frac{b(\lossset,
      \optdomain)}{\sqrt{n}}.
  \end{align}
\end{subequations}

Turning to the setting of differential privacy, we are able to show similar
results to the bounds~\eqref{eqn:main-lower-bound}
and~\eqref{eqn:main-upper-bound}.  Namely, there exist constants $b'(\lossset,
\optdomain) \ge a'(\lossset, \optdomain)$ such that if we wish to guarantee
$\diffp$-differential privacy, then for any estimator $\what{\optvar}$, the
risk is lower bounded by
\begin{subequations}
  \label{eqn:main-bounds-diffp}
  \begin{equation}
    \label{eqn:main-lower-bound-diffp}
    \sup_{\loss \in \lossset} \Excess_n(\what{\optvar}; \loss, \optdomain)
    \ge \frac{\sqrt{d}}{\diffp} \,
    \frac{a'(\lossset, \optdomain)}{\sqrt{n}},
  \end{equation}
  while there exist estimators $\what{\optvar}$ such that
  \begin{equation}
    \label{eqn:main-upper-bound-diffp}
    \sup_{\loss \in \lossset} \Excess_n(\what{\optvar}; \loss, \optdomain)
    \le \frac{\sqrt{d}}{\diffp} \,
    \frac{b'(\lossset, \optdomain)}{\sqrt{n}}.
  \end{equation}
\end{subequations}
Here again, the constant $a'(\lossset, \optdomain)$ controls the non-private
minimax rate of estimation.

Finally, we show that stochastic gradient descent is one procedure
that achieves the above upper bounds, and moreover, that the ratios
$b(\lossset, \optdomain) / a(\lossset, \optdomain)$ and $b'(\lossset,
\optdomain) / a'(\lossset, \optdomain)$ are bounded above by a
universal (numerical) constant. The bounds~\eqref{eqn:main-bounds}
and~\eqref{eqn:main-bounds-diffp} thus establish and quantify
explicitly the sharp tradeoff between learning and statistical
estimation and the amount of privacy provided to the population.  More
concretely, we can evaluate the effective sample size of learning
procedures receiving private observations. In the case of
information-based privacy, the sample size of any learning procedure
receiving maximally privatized observations from the data providers is
decreased from $n$ to roughly $n \information^* / d$, while in
differentially private settings, we see that the effective sample size
decreases from $n$ to $n \diffp^2 / d$. The first of these is perhaps
intuitive: in rough terms, a $d$-dimensional observation $\statrv_i$
contains about $d$-bits of information, and so we expect a loss in
(statistical) efficiency by a factor $\information^* / d$. For the
second, recent results suggest scalings of $\information^* \approx
\diffp^2$ (e.g.~\cite[Lemma 3.2]{DworkRoVa10}), so the loss $n \mapsto
n \diffp^2 / d$ is also perhaps intuitive.

Our subsequent analysis will build on this favorable property of
gradient-based methods.  Indeed, in the remainder of the paper, we
will assume that the communication protocol---except in the
non-interactive $\diffp$-differentially private case, which allows any
protocol---by which data is conveyed to the learner $\method$ is based
on (sub)gradients of the loss.  As further motivation for this choice,
note that the subgradient (more generally, a score function) of the
loss $\loss$ is asymptotically sufficient in the sense
of~\citet{LeCam56}. A bit more precisely, gradients (in an asymptotic
sense) contain \emph{all} of the statistical information for risk
minimization problems.  Secondly, estimation procedures based on
stochastic gradient information are asymptotically
efficient~\cite{PolyakJu92}, in the sense of both Bahadur and minimax
efficiency~\cite[Chapter 8]{VanDerVaart98}, and are thus essentially
sample optimal; they also have minimax-optimality guarantees in
finite-sample settings~\cite{AgarwalBaRaWa12}.  Moreover, many
estimation procedures are
gradient-based~\cite{NemirovskiYu83,BoydVa04}, and distributed
optimization procedures that send subgradient information across a
network to a centralized procedure $\method$ are natural
(e.g.~\cite{BertsekasTs89}).  Our arguments also show that in many
settings, disguising subgradients is equivalent to disguising the data
$\statrv$ itself.  Thus, as an additional consequence of our
gradient-based focus, our algorithmic bounds also apply in streaming
and online settings, requiring only a fixed-size memory footprint.

\subsection{Outline and techniques}

We spend the remainder of the paper deriving the
bounds~\eqref{eqn:main-bounds} and \eqref{eqn:main-bounds-diffp}.  Our
route to obtaining these bounds is based on a two-part analysis.
First, we consider saddle points of the mutual information
$\information(\statrv; \channelrv)$, when viewed as a function of the
distribution $\statprob$ of $\statrv$ and the conditional distribution
$\channelprob(\cdot \mid \statrv)$ of $\channelrv$, under natural
constraints that still allow estimation. We consider related saddle
points for differentially private conditional distributions. Having
computed these saddle points, we can apply information-theoretic
techniques for obtaining lower bounds on estimation and
optimization~\cite{YangBa99,AgarwalBaRaWa12} to prove the results of
the form~\eqref{eqn:main-lower-bound}
or~\eqref{eqn:main-lower-bound-diffp}.  Our upper bounds then follow
by application of known convergence rates for computationally
efficient methods, such as the stochastic gradient and mirror descent
algorithms~\cite{NemirovskiYu83,NemirovskiJuLaSh09}.  We provide full
proofs---except for technical results deferred to appendices---and a
more complete outline of our technique in
Section~\ref{sec:optimal-rate-proofs}.

The remainder of the paper is organized as follows. We give a precise
definition of our notions of local privacy in
Section~\ref{sec:optimal-local-privacy}.  Section~\ref{sec:optimal-rates} is
devoted to information-theoretic lower bounds on the convergence rate of any
statistical method $\method$ in terms of the mutual information
$\information^*$ between what the method $\method$ observes and each sample
$\statrv_i$. We characterize the unique privacy guaranteeing distributions in
Section~\ref{sec:saddle-points}, which provides a constructive mechanism for
trading off privacy and learning.  We present our conclusions in
Section~\ref{sec:discussion}.

\paragraph{Notation}
Before continuing, we give our notation and a few standard definitions.  The
\emph{Kullback-Leibler (KL)} divergence between distributions $P$ and $Q$
defined on a set $S$, where $P$ and $Q$ are assumed to have densities $p$ and
$q$ with respect to a base measure $\nu$\footnote{This is no loss of
  generality, as $P$ and $Q$ are absolutely continuous with respect to $\nu =
  \half(P + Q)$.} is given by
\begin{equation*}
  \dkl{P}{Q}
  \defeq \int_S p(s) \log\frac{p(s)}{q(s)} d\nu(s).
\end{equation*}
Similarly, the \emph{total-variation distance} between the distributions
$P$ and $Q$ is defined as
\begin{equation*}
  \tvnorm{P - Q}
  \defeq \sup_{A \subset S}
  \left|P(A) - Q(A)\right|
  = \half \int_S |p(s) - q(s)| d\nu(s).
\end{equation*}
For a convex function $f : \R^d \rightarrow \R \cup \{+\infty\}$, the
subgradient set $\partial f(\optvar)$ of $f$ at the point $\optvar$ is
\begin{equation*}
  \partial f(\optvar) \defeq \left\{ g \in \R^d : f(\optvar') \ge
  f(\optvar) + g^\top(\optvar' - \optvar), ~ \mbox{for~all} ~ \optvar'
  \in \R^d \right\}.
\end{equation*}
We use $\partial \loss(\statsample, \optvar)$ to denote the subgradient
set of the function $\optvar \mapsto \loss(\statsample, \optvar)$, and for a
convex function, $\nabla \loss(\statsample, \optvar)$ denotes an arbitrary
element of $\partial \loss(\statsample, \optvar)$. We say that a function $f$
is $\lipobj$-Lipschitz with respect to the norm $\norm{\cdot}$ over the set
$\optdomain$ if
\begin{equation*}
  \left|f(\optvar) - f(\optvar')\right| \le
  \lipobj \norm{\optvar - \optvar'}
  ~~~ \mbox{for~all}~ \optvar, \optvar' \in \optdomain.
\end{equation*}
The notation $\norm{\cdot}_p$ denotes a standard $\ell_p$-norm.  We use the
abbreviation r.c.d.\ throughout for regular conditional
distribution~\cite{Billingsley86}. The extreme points of a set $C \subset
\R^d$ are denoted by $\extreme(C)$, the convex hull of $C$ is denoted by
$\conv(C)$, and the support of a distribution $\statprob$ is denoted $\supp
\statprob$.  We say values $a_n \asymp b_n$ if $\lim_n (a_n / b_n) = 1$.  The
symbol $e_i$ denotes the $i$th standard basis vector in $\R^d$. Lastly, the
symbol $\rightrightarrows$ denotes a set-valued
mapping~\cite{HiriartUrrutyLe96ab}.



%% file: local-privacy.tex
\section{Problem Formulation}
\label{sec:optimal-local-privacy}

We begin with a formal description of the communication protocol by
which information about the random variables $\statrv$ is communicated
to the procedure $\method$.  We then define the notion of
\emph{optimal local privacy} studied in this paper and the minimax
framework in which we state our main results.

\subsection{Communication protocol}
\label{sec:communication-protocol}

In this paper, we focus on statistical learning procedures that have
access to data through the subgradients $\partial \loss(\statrv,
\optvar)$ of the loss functions.  More formally, at each round, the
method $\method$ is given access to a random vector $\channelrv_i$
such that
\begin{equation}
  \label{eqn:unbiased-subgradient}
  \E[\channelrv_i \mid \statrv_i, \optvar] \in \partial
  \loss(\statrv_i, \optvar),
\end{equation}
where $\optvar \in \optdomain$ is a parameter chosen by the method.
In Appendix~\ref{appendix:biased-subgradient} we present an argument
that shows that the unbiasedness of the subgradient
inclusion~\eqref{eqn:unbiased-subgradient} is not only intuitively
appealing but is, in a certain sense, necessary.

\noindent In detail, our communication protocol consists of the
following three steps:
\begin{itemize}
\item the method $\method$ sends the parameter vector $\optvar$ to the
  owner of the $i$th sample $\statrv_i$;
\item owner $i$ computes a subgradient vector $g \in \partial
  \loss(\statrv_i, \optvar)$ to be communicated privately;
\item the vector $\channelrv_i$ is communicated to $\method$ under the
  constraint that 
\begin{align*}
\E[\channelrv_i \mid \statrv_i, \optvar] = g \in \partial
\loss(\statrv_i, \optvar).
\end{align*}
\end{itemize}

We assume throughout that there is a compact set $C \subset \R^d$ such
that $\partial \loss(x, \optvar) \subseteq C$ for all pairs
\mbox{$(\optvar, x) \in \optdomain \times \statdomain$.}  Our goal is
``disguise'' the subgradient information with a random variable
$\channelrv$ satisfying $\channelrv \in D$, for some compact set $D$
such that $C \subset \interior D \subset \R^d$.  For instance, a
common choice of these sets are norm balls, say of the form
\begin{equation*}
  C = \{ g \in \R^d : \norm{g} \leq \lipobj \},
  \quad \mbox{and} \quad D = \{ g \in \R^d :
  \norm{g} \leq M \},
\end{equation*}
where $\norm{\cdot}$ is a given norm on $\R^d$, and the radius choice $M >
\lipobj$ ensures that $C \subset \interior D$.  This choice covers a variety
of online optimization and stochastic approximation
algorithms~\cite{Zinkevich03,BeckTe03,NemirovskiJuLaSh09,AgarwalBaRaWa12},
for which it is assumed that for any $\statsample \in \statdomain$ and
$\optvar \in \optdomain$, if $g \in \partial \loss(x, \optvar)$ then $\norm{g}
\le \lipobj$ for some norm $\norm{\cdot}$. We may obtain privacy by allowing
perturbation of the subgradient $g$, which is then required to live in a
(larger) norm ball of radius $M > \lipobj$.

\subsection{Optimal local privacy}

Suppose that $\statrv$ has distribution $\statprob$, and for each
$\statsample \in \statdomain$, let $\channelprob(\cdot \mid
\statsample)$ denote the regular conditional probability measure of
$\channelrv$ given that $\statrv = x$.  This pair defines the marginal
distribution $\channelprob(\cdot)$ via $\channelprob(A) =
\E[\channelprob(A \mid \statrv)]$, where the expectation taken with
respect to $\statrv \sim \statprob$.  The mutual information between
$\statrv$ and $\channelrv$ is the expected Kullback-Leibler (KL)
divergence between $\channelprob(\cdot \mid \statrv)$ and
$\channelprob(\cdot)$:
\begin{equation}
  \label{eqn:formal-information-def}
  \information(\statprob, \channelprob) = \information(\statrv;
  \channelrv) \defeq \E_\statprob \left [ \dkl{\channelprob(\cdot \mid
      \statrv)}{\channelprob(\cdot)} \right].
  %
  %
\end{equation}
We view the problem of privacy as a game between the adversary
controlling $\statprob$ and the data owners, who use $\channelprob$ to
obscure the samples $\statrv$. In particular, we say a distribution
$\channelprob$ guarantees a level of privacy $\information^*$ if and
only if $\sup_\statprob \information(\statprob, \channelprob) \le
\information^*$.  Note that this guarantee is worst-case, ensuring
that for any choice of distribution $\statprob$, the publicly
available random variable $\channelrv$ provides at most mutual
information $\information^*$ about the sample $\statrv$.

Our goal is to find a saddle point $\statprob^*, \channelprob^*$ such
that
\begin{equation}
  \label{eqn:saddle-point}
  \sup_\statprob \information(\statprob, \channelprob^*)
  \le \information(\statprob^*, \channelprob^*)
  \le \inf_\channelprob \information(\statprob^*, \channelprob),
\end{equation}
where the first supremum is taken over all distributions $\statprob$ on
$\statrv$ such that $\nabla\loss(\statrv, \optvar) \in C$ with
$\statprob$-probability 1, and the infimum is taken over all regular
conditional distributions $\channelprob$ such that if $\channelrv \sim
\channelprob(\cdot \mid \statrv)$ (meaning that $\channelrv$ is drawn from
$\channelprob$ conditional on $\statrv$), then $\channelrv \in D$ and
$\E_\channelprob[\channelrv \mid \statrv, \optvar] = \nabla\loss(\statrv,
\optvar)$.  Indeed, if we can find $\statprob^*$ and $\channelprob^*$
satisfying the saddle point~\eqref{eqn:saddle-point}, then combination with
the trivial direction of the max-min inequality
yields
\begin{equation*}
  \sup_\statprob \inf_\channelprob \information(\statprob, \channelprob)
  = \information(\statprob^*, \channelprob^*)
  = \inf_\channelprob \sup_\statprob \information(\statprob, \channelprob).
\end{equation*}
To fully formalize this idea and our notions of privacy, we define two
collections of probability measures and associated losses.
For sets $C \subset D \subset
\R^d$, we define the source set
\begin{subequations}
  \begin{equation}
    \sourcedistset[C]
    \defeq \left\{\mbox{Distributions}~ \statprob
    ~\mbox{such~that}~ \supp \statprob \subset C \right\}
    \label{eqn:source-distribution-set}
  \end{equation}
  and the set of communicating
  distributions as the following regular conditional distributions (r.c.d.'s):
  \begin{equation}
    \channeldistset[C, D]
    \defeq \left\{\mbox{r.c.d.'s~}
    \channelprob ~ \mbox{s.t.} ~
    \supp \channelprob(\cdot \mid c) \subset D
    ~ \mbox{and} ~
    \int_D \channelval d\channelprob(\channelval \mid c) = c
    ~  \mbox{for~}c \in C
    \right\}.
    \label{eqn:channel-distribution-set}
  \end{equation}
\end{subequations}
The definitions~\eqref{eqn:source-distribution-set}
and~\eqref{eqn:channel-distribution-set} formally define the sets over which
we may take infima and suprema in the saddle point calculations, and they
capture what may be communicated. The conditional distributions $\channelprob
\in \channeldistset[C, D]$ are defined so that for any loss $\loss$ with
$\nabla \loss(\statsample, \optvar) \in C$, we have
\begin{equation*}
  \E_\channelprob[\channelrv \mid \statrv = \statsample, \optvar]
  \defeq \int_D \channelval d\channelprob\left(\channelval
  \mid \nabla \loss(\statsample, \optvar)\right)
  = \nabla \loss(\statsample,
  \optvar).
\end{equation*}

We now make the following key definition:
\begin{definition}
  \label{def:optimal-local-privacy}
  The conditional distribution $\channelprob^*$ satisfies \emph{\olp} for the
  sets \mbox{$C \subset D \subset \R^d$} if
  \begin{equation*}
    \sup_\statprob
    \information(\statprob, \channelprob^*)
    =
    \inf_\channelprob
    \sup_\statprob
    \information(\statprob, \channelprob)
  \end{equation*}
  where the supremum is taken over distributions $\statprob \in
  \sourcedistset[C]$ and the infimum is taken over regular conditional
  distributions $\channelprob \in \channeldistset[C, D]$.  We say
  $\channelprob^*$ satisfies \emph{\olp at level $\information^*$} if
  in addition $\sup_{\statprob} \information(\statprob, \channelprob^*) =
  \information^*$.
\end{definition}

We also formulate a corresponding notion of local optimality in the 
differentially private setting.  For given sets $C \subset D$, define
the differential privacy measure
\begin{equation}
  \optdiffp(C, D) \defeq
  \inf_\channelprob
  \log\left[
    \sup_{S \in \sigma(D)}
  \sup_{\statsample, \statsample' \in C}
  \frac{\channelprob(S \mid \statrv = \statsample)}{
    \channelprob(S \mid \statrv = \statsample')}\right],
  \label{eqn:optimal-diffp}
\end{equation}
where the infimum is taken over all regular conditional distributions
$\channelprob \in \channeldistset[C, D]$ such that $\E_\channelprob[\channelrv
  \mid \statrv = \statsample] = \statsample$.  We define \olpd as follows:
\begin{definition}
  \label{def:local-diffp}
  The conditional distribution $\channelprob^*$
  satisfies \emph{\olpd} for the sets $C \subset D \subset \R^d$
  if $\channelprob^* \in \channeldistset[C, D]$ and
  \begin{enumerate}[1.]
  \item The distribution $\channelprob^*$ is $\optdiffp(C, D)$-differentially
    private.
  \item We have
    $\sup_\statprob \information(\statprob, \channelprob^*) \le \sup_\statprob
    \information(\statprob, \channelprob)$, for all $\optdiffp(C,
    D)$-differentially private $\channelprob \in \channeldistset[C, D]$, where
    the supremum is taken over all distributions $\statprob \in
    \sourcedistset[C]$.
  \end{enumerate}
\end{definition}

If a distribution $\channelprob^*$ satisfies \olp\ or \olpd, then it
guarantees that even for the worst possible distribution on $\statrv$,
the information communicated about $\statrv$ is limited. (Part of our
results consist in showing that for suitable sets $C \subset D$, it is
possible to attain $\optdiffp(C, D)$, so it is sensible to, in
addition, choose the distribution that minimizes mutual information.)

In a sense, Definitions~\ref{def:optimal-local-privacy}
and~\ref{def:local-diffp} capture the natural competition between
privacy and learnability. The method $\method$ specifies the set $D$
to which the data $\channelrv$ it receives must belong; the
``teachers,'' or owners of the data $\statrv$, choose the distribution
$\channelprob$ to guarantee as much privacy as possible subject to
this constraint.  Using these mechanisms, if we can characterize a
unique distribution $\channelprob^*$ attaining the
infimum~\eqref{eqn:saddle-point} for $\statprob^*$ (and by extension,
for any $\statprob$), then we may study the effects of requiring a
bounded amount of information to be communicated to the method
$\method$ about $\statrv$, which we do in
Section~\ref{sec:optimal-rates}.

\subsection{Minimax error}

Given an estimate $\what{\optvar}$ based on $n$ samples $\statrv$ from a
distribution $\statprob$, we assess its quality in terms of the risk
function~\eqref{eqn:objective}, i.e.\ $\risk(\optvar) = \E[\loss(\statrv,
  \optvar)]$.  In this section, we describe the minimax framework for
obtaining bounds uniformly over all possible estimators.
Let $\method$ denote any statistical procedure or
method that operates on stochastic gradient samples, and let
$\what{\optvar}_n$ denote the output of $\method$ after receiving $n$
such samples. The excess risk of the method $\method$ on the risk
$\risk(\optvar)$ after receiving $n$ sample gradients is
\begin{equation}
  \label{eqn:error-metric}
  \optgap_n(\method, \loss, \optdomain, \statprob)
  \defeq \risk(\what{\optvar}_n) - \inf_{\optvar \in \optdomain}
  \risk(\optvar)
  = \E_\statprob[\loss(\statrv, \what{\optvar}_n)]
  - \inf_{\optvar \in \optdomain}\E_\statprob[\loss(\statrv, \optvar)].
\end{equation}
The excess risk is a random variable, since the output
$\what{\optvar}_n$ of the method is random.

In our settings, in addition to the randomness in the sampling
distribution $\statprob$, there is additional randomness from the
perturbation applied to stochastic gradients of the objective
$\loss(\statrv, \cdot)$ to mask $\statrv$ from the statistitician or
method $\method$. Let $\channelprob$ denote the regular conditional
probability---the channel distribution---whose conditional part is
defined on the range of the (set-valued) subgradient mapping $\partial
\loss(\statrv, \cdot) : \optdomain \rightrightarrows \R^d$. Since the
output $\what{\optvar}_n$ of the statistical procedure $\method$ is a
random function of both $\statprob$ and $\channelprob$, we take the
expectation and measure the expected sub-optimality of the risk
according to $\statprob$ and $\channelprob$. We let $\lossset$ denote
a collection of loss functions, where for a distribution $\statprob$
on $\statdomain$, the set $\lossset(\statprob)$ denotes the losses
\mbox{$\loss : \supp \statprob \times \optdomain \rightarrow \R_+$}
  belonging to $\lossset$.  The \emph{minimax error} is then given by
\begin{equation}
  \label{eqn:minimax-error}
  \optgap_n^*(\lossset, \optdomain) \defeq
  \inf_{\method}  \sup_\statprob \sup_{\loss \in \lossset(\statprob)}
  \E_{\statprob, \channelprob}[
    \optgap_n(\method, \loss, \optdomain, \statprob)],
\end{equation}
where the expectation is taken over the random samples $\statrv \sim
\statprob$ and $\channelrv \sim \channelprob(\cdot \mid \statrv,
\optvar)$.  In this paper, we provide characterizations of the minimax
error~\eqref{eqn:minimax-error} for several classes of loss functions
$\lossset(\statprob)$, giving sharp results when the privacy
distribution $\channelprob$ satisfies \olp\ for any loss function
$\loss \in \lossset(\statprob)$ and distribution $\statprob$.



%% file: statistical-rates.tex
\section{Optimal Learning Rates and Tradeoffs}
\label{sec:optimal-rates}

\providecommand{\real}{\R}

With the basic framework in place, we now turn to statements of our
main results.  We begin by imposing certain (weak) conditions on the
families of loss functions that we consider, and subsequently turn to
the main results of this section (Theorems~\ref{theorem:linf-bound}
and \ref{theorem:lone-bound}, which apply to information-based
privacy, and
Theorems~\ref{theorem:linf-bound-diffp}--\ref{theorem:second-class},
which apply to $\diffp$-differential privacy) as well as some of their
consequences (Corollaries~\ref{corollary:linf-privacy},
\ref{corollary:lone-privacy},
and~\ref{corollary:linf-privacy-diffp}). After describing the optimal
privacy-preserving distributions in Section~\ref{sec:saddle-points},
we provide proofs of the results in this section in
Section~\ref{sec:optimal-rate-proofs}.

\subsection{Families of loss functions and stochastic gradient methods}
\label{sec:loss-families}

We assume that our collection of loss functions obey certain natural
smoothness conditions.  For each $\pval \in [1, \infty]$, we use
$\norm{\cdot}_\pval$ to denote the usual $\ell_\pval$-norm, and we use
$\qval = \frac{\pval}{\pval-1}$ to denote the conjugate exponent
satisfying the relation $1/\pval + 1/\qval = 1$.
With this notation, we have the following definition:
\begin{definition}
  \label{def:lp-loss-class}
  For parameters $\lipobj > 0$ and $p \ge 1$, an \emph{$(\lipobj, p)$-loss
    function} is a measurable function $\loss : \statdomain \times \optdomain
  \rightarrow \R$ such that for $x \in \statdomain$, the function $\optvar
  \mapsto \loss(x, \optvar)$ is convex and $\lipobj$-Lipschitz continuous with
  respect to the norm $\norm{\cdot}_q$.
\end{definition}
\noindent
A convex loss $\loss$ satisfies Definition~\ref{def:lp-loss-class} if
and only if for all $\optvar \in \optdomain$, we have the inequality
$\norm{g}_p \le \lipobj$ for any subgradient $g \in \partial \loss(x,
\optvar)$ (e.g.~\cite{HiriartUrrutyLe96ab}). \\

\noindent To illustrate this definition, let us consider a few
examples:
\begin{example}
  \label{example:median}
  Consider the problem of finding a multi-dimensional median, in which
  case each sample \mbox{$\statsample \in \R^d$}, and the loss
  function takes the form
  \begin{equation*}
    \loss(\statsample, \optvar) = \lipobj \lone{\optvar - \statsample}.
  \end{equation*}
  This loss is $\lipobj$-Lipschitz with respect to the $\ell_1$-norm,
  subgradients belonging to $[-\lipobj, \lipobj]^d$, and hence it
  belongs to the class of $(\lipobj, \infty)$-loss functions.
\end{example}

\begin{example}[Classification]
  We may also consider classification based on either 
  the hinge loss or logistic regression loss. In this setting, the
  data comes in pairs $\statsample = (a, b)$, where $a \in \R^d$ is the
  set of regressors or predictors and $b \in \{-1, 1\}$ is the label;
  the losses are
  \begin{equation*}
    \loss(\statsample, \optvar)
    = \hinge{1 - b \<a, \optvar\>}
    ~~~ \mbox{and} ~~~
    \loss(\statsample, \optvar)
    = \log\left(1 + \exp(-b \<a, \optvar\>)\right).
  \end{equation*}
  By computing (sub)gradients, we may verify that each of these is an
  $(\lipobj, p)$-loss if and only if the covariate
  vector $a \in \R^d$ satisfies $\norm{a}_p \le \lipobj$, which is
  a common
  assumption~\cite{ChaudhuriMoSa11,RubinsteinBaHuTa12}.
\end{example}

Definition~\ref{def:lp-loss-class} is natural given the communication strategy
we outline in Section~\ref{sec:communication-protocol}. Since our loss
functions satisfy $\norm{\partial \loss(\statrv, \optvar)} \le \lipobj$, the
channel distribution $\channelprob$ amounts to perturbing subgradients to
larger norm balls while maintaining the appropriate expectations. \\

Before proceeding, we briefly review standard algorithms for solving
problems of the forms outlined above, since they are essential to our
results: for each of our main results, the optimal convergence rate is
attained by (a variant of) mirror
descent~\cite{NemirovskiYu83,BeckTe03,NemirovskiJuLaSh09}, which is a
non-Euclidean generalization of the stochastic gradient
method~\cite{NemirovskiYu83,PolyakJu92,Zinkevich03}.  Stochastic
gradient methods are iterative methods that update a parameter
$\optvar^t$ over iterations $t$ of an algorithm using stochastic
gradient information. At iteration $t$, the algorithm receives a
vector $g_t \in \R^d$ with conditional expectation $\E[g_t \mid
  \optvar^t] \in \partial \risk(\optvar^t)$, then performs the update
\begin{equation*}
  \optvar^{t + 1}
  = \argmin_{\optvar \in \optdomain} \left\{
  \eta \<g_t, \optvar\> + \Psi(\optvar, \optvar^t) \right\}.
\end{equation*}
Here $\eta$ is a step-size and $\Psi$ is a Bregman divergence, which
keeps $\optvar^{t + 1}$ relatively close to $\optvar^t$.  (See the
papers~\cite{BeckTe03,NemirovskiJuLaSh09} for further details.)  With
appropriate choice of $\Psi$, the mirror descent algorithm enjoys the
following convergence guarantees. Define $\what{\optvar}_n =
\frac{1}{n} \sum_{t=1}^n \optvar^t$.  If $\E[\linf{g_t}^2 \mid
  \optvar^t] \le M_\infty^2$ for all $t$ and $\optdomain$ is contained
in the $\ell_1$-ball of radius $\radius_1$, then with appropriate
choice of $\Psi$ and $\eta$
\begin{subequations}
  \label{eqn:stochastic-bounds}
  \begin{equation}
    \E[\risk(\what{\optvar}_n)] - \risk(\optvar^*)
    = \order\left(\frac{M_\infty \radius_1 \sqrt{\log d}}{\sqrt{n}}\right).
    \label{eqn:mirror-descent-bound}
  \end{equation}
  See, for example, \citet[Section~5]{BeckTe03} or
  \citet[Section~2.3]{NemirovskiJuLaSh09}.  Similarly, with the choice
  $\Psi(\optvar, \optvar') = \ltwo{\optvar - \optvar'}^2$, if
  $\E[\ltwo{g_t}^2 \mid \optvar^t] \le M_2^2$ and $\optdomain$ is
  contained in the $\ell_2$-ball of radius $\radius_2$, then
  \begin{equation}
    \E[\risk(\what{\optvar}_n)] - \risk(\optvar^*)
    = \order\left(\frac{M_2 \radius_2}{\sqrt{n}}\right).
    \label{eqn:sgd-bound}
  \end{equation}
\end{subequations}
For instance, see the references~\cite{Zinkevich03,NemirovskiJuLaSh09}
for results of this type.

%

\subsection{Minimax error bounds under privacy}

We now state our main theorems, and discuss some of their
consequences.  All proofs are deferred to
Section~\ref{sec:optimal-rate-proofs}.


\subsubsection{Minimax errors with mutual information-based privacy}

Our first two main results consider privacy mechanisms $\channelprob$
satisfying \olp, Definition~\ref{def:optimal-local-privacy}. We state
the theorems first focusing on their dependence on the geometry of the
subdifferential sets (in which the subgradients live); in the
corollaries to follow we show how these choices correspond to
particular mutual information guarantees on privacy.

Our first theorem applies to the class of $(\lipobj, \infty)$ loss
functions as given in Definition~\ref{def:lp-loss-class}. For this
theorem, we assume that the set to which the perturbed data
$\channelrv$ must belong is $[-M_\infty, M_\infty]^d$, where $M_\infty
\ge \lipobj$. In the notation of
Definition~\ref{def:optimal-local-privacy}, this corresponds to taking
$C = [-\lipobj, \lipobj]^d$ and $D = [-M_\infty, M_\infty]^d$.  We
state two variants of the first theorem, as one version gives slightly
sharper results for an important special case.

\begin{theorem}
  \label{theorem:linf-bound}
  Let $\lossset$ be the collection of $(\lipobj, \infty)$ loss
  functions, assume the conditions of the preceding paragraph, and let
  $\channelprob$ be optimally locally private
  (Definition~\ref{def:optimal-local-privacy}) for $\lossset$.  Then
  \begin{enumerate}[(a)]
  \item If $\optdomain$ contains the $\ell_\infty$ ball of radius
    $\radius$, then
    \begin{equation*}
      \optgap_n^*(\lossset, \optdomain) \ge
      \frac{1}{20} \,
      \min\left\{\radius \lipobj d,
      \frac{M_\infty \radius d}{9 \sqrt{n}}\right\}.
    \end{equation*}
  \item If $\optdomain = \{\optvar \in \R^d : \lone{\optvar} \le
    \radius\}$, then 
    \begin{equation*}
      \optgap_n^*(\lossset, \optdomain) \ge \frac{1}{8} \,
      \min\left\{\radius \lipobj,
      \frac{M_\infty \, \radius \,\sqrt{\log(2 d)}}{2 \sqrt{n}}\right\}.
    \end{equation*}
  \end{enumerate}
\end{theorem}

Our second main theorem applies to loss functions and objectives with
a different geometry. Now we assume that the
loss functions $\lossset$ consist of $(\lipobj, 1)$ losses, and that
the perturbed data must belong to the $\ell_1$ ball of radius $M_1$,
i.e., $\channelrv \in \{\channelval \in \R^d : \lone{\channelval} \le
M_1\}$. Thus in the notation of
Definition~\ref{def:optimal-local-privacy}, we have $D = (M_1 /
\lipobj) C$, where $C = \{g \in \R^d : \lone{g} \le \lipobj\}$.  If we
define $M = M_1 / \lipobj$, we may define the constants
\begin{equation}
  \gamma \defeq \log\left(\frac{2d - 2 + \sqrt{(2d - 2)^2 + 4(M^2 - 1)}}{
    2(M - 1)}\right)
  ~~~ \mbox{and} ~~~
  \channeldiff \defeq \frac{e^\gamma - e^{-\gamma}}{
    e^\gamma + e^{-\gamma} + 2(d - 1)},
  \label{eqn:lone-normalization-ahead-of-time}
\end{equation}
which are related to the unique distribution achieving \olp\ for the
$(\lipobj, 1)$ losses and the larger $\ell_1$-ball above (see
equation~\eqref{eqn:lone-exp-normalization} and
Proposition~\ref{proposition:lone-saddle-point}).  We have the
following theorem.
\begin{theorem}
  \label{theorem:lone-bound}
  Let $\lossset$ be the collection of $(\lipobj, 1)$ loss functions,
  assume the conditions of the preceding paragraph, and let
  $\channelprob$ be optimally private for the collection
  $\lossset$. If $\optdomain$ contains the $\ell_\infty$-ball of
  \mbox{radius $\radius$,}
  \begin{equation*}
    \optgap_n^*(\lossset, \optdomain) \ge \frac{1}{20} \,
    \min\left\{\radius \lipobj,
    \frac{\radius \lipobj \sqrt{d}}{9 \sqrt{n} \channeldiff}\right\}.
  \end{equation*}
\end{theorem}

\paragraph{Remarks}
We make a few remarks on Theorems~\ref{theorem:linf-bound} and
\ref{theorem:lone-bound}. First, we note that, when reduced to the special
case of having no random distribution $\channelprob$,
Theorems~\ref{theorem:linf-bound} and~\ref{theorem:lone-bound} each yield a
minimax rate for stochastic optimization problems.  Indeed, in
Theorem~\ref{theorem:linf-bound}, we may take $M_\infty = \lipobj$, in which
case (focusing on the second statement of the theorem) we obtain that for
$\optdomain = \{\optvar \in \R^d : \lone{\optvar} \le \radius\}$,
\begin{equation*}
  \optgap_n^*(\lossset, \optdomain) \ge \frac{\radius \lipobj}{16}
  \sqrt{\frac{\log(2d)}{n}}.
\end{equation*}
Mirror descent algorithms~\cite{NemirovskiYu83,NemirovskiJuLaSh09} can be used
to minimize this class of loss functions, and their convergence rate matches
this lower bound up to constant factors (also see our results in the sequel,
as well as the explanation of~\citet{AgarwalBaRaWa12}).  When specialized to
this setting our result is thus unimprovable. Moreover, our analysis is
sharper than previous analyses: none of the existing lower bounds recover
the logarithmic dependence on the dimension $d$, which is evidently necessary.


Our second remark is that while our results appear to require disguising
only gradient information, based on our communication formulation in
Section~\ref{sec:communication-protocol}, this restriction is not actually
substantial. Indeed, when the domain $\optdomain$ is a norm ball 
we can establish each of our lower bounds using
the loss function $\loss(\statsample, \optvar) = \<\statsample, \optvar\>$.
In this case, $\nabla \loss(\statsample, \optvar) = \statsample$, so that
the communication scheme explicitly disguises \emph{exactly} the
individual data $\statrv_i$.

We now turn to some consequences of Theorems~\ref{theorem:linf-bound} and
\ref{theorem:lone-bound}, where we exhibit the tradeoffs between rates of
convergence for any statistical procedure and the desired privacy of a
user. We present two corollaries that characterize this tradeoff.  Looking
ahead to Section~\ref{sec:saddle-points}, we may use
Propositions~\ref{proposition:linf-saddle-point}
and~\ref{proposition:lone-saddle-point} in that section 
to derive a bijection between the sizes $M_\infty$ and $M_1$ of
the perturbation sets and the amount of privacy as measured by the worst case
mutual information $\information^*$. We can then combine the lower bounds of
Theorems~\ref{theorem:linf-bound} and~\ref{theorem:lone-bound} with
results on stochastic approximation (the mirror descent convergence
rates~\eqref{eqn:stochastic-bounds}) to obtain the following tradeoffs.  We provide the
full proofs in Sections~\ref{sec:corollary-rate-linf} and
\ref{sec:corollary-rate-lone}, respectively.

\begin{corollary}
  \label{corollary:linf-privacy}
  Under the conditions of Theorem~\ref{theorem:linf-bound}(b), assume
  moreover that \mbox{$M_\infty \ge 2 \lipobj$}, and that
  $\channelprob^*$ satisfies \olp\ at information level
  $\information^*$ in the sense of
  Definition~\ref{def:optimal-local-privacy}. Then for universal
  constants $0 < c_\ell \le c_u < \infty$, the minimax error is
  sandwiched as
  \begin{equation*}
    c_\ell \frac{\sqrt{d}}{\sqrt{\information^*}}
    \cdot \frac{\radius \lipobj \sqrt{\log (2d)}}{\sqrt{n}}
    \le \optgap_n^*(\lossset, \optdomain) \le c_u
    \frac{\sqrt{d}}{\sqrt{\information^*}} \cdot
    \frac{\radius \lipobj \sqrt{\log (2d)}}{ \sqrt{n}}.
  \end{equation*}
\end{corollary}

Similar upper and lower bounds can be obtained
under the conditions of part~(a) of Theorem~\ref{theorem:linf-bound},
again by using mirror descent, but we lose a factor of $\sqrt{\log d}$
in the lower bound.  (There is an additional factor of $d$ in the
statement~(a), and $\optdomain \supseteq \{\optvar \in \R^d :
\linf{\optvar} \le \radius / d\}$.)  In this case we would not need
to assume that $\optdomain$ is an $\ell_1$-ball for the lower bound. \\

\noindent We now turn to an analogous result based on an
application of Theorem~\ref{theorem:lone-bound} and
Proposition~\ref{proposition:lone-saddle-point}.
\begin{corollary}
  \label{corollary:lone-privacy}
  Under the conditions of Theorem~\ref{theorem:lone-bound}, assume
  that $M_1 \ge 2 \lipobj$ and $\channelprob^*$ satisfies
  \olp\ at information level $\information^*$. Moreover, suppose that
  $\optdomain$ contains an $\ell_\infty$-ball of radius $c_1 \radius$
  and is contained in an $\ell_\infty$-ball of radius $c_2 \radius$, where
  $0 < c_1 \le c_2$ are constants.
  Then for universal
  constants $0 < c_\ell \le c_u < \infty$, the minimax error is sandwiched as
  \begin{equation*}
    c_\ell \frac{\sqrt{d}}{\sqrt{\information^*}}
      \cdot \frac{\radius \lipobj \sqrt{d}}{\sqrt{n}}
    \le \optgap_n^*(\lossset, \optdomain)
    \le c_u \frac{\sqrt{d}}{\sqrt{\information^*}}
    \cdot \frac{\radius \lipobj \sqrt{d}}{\sqrt{n}}.
  \end{equation*}
\end{corollary}

As a final remark, we have stated results that depend on specific
geometric properties of the loss functions $\lossset$. While these
geometric properties are natural, as illustrated by the example
Section~\ref{sec:loss-families}, it is also possible to use our
techniques to derive alternative results.  Such extensions require
computing the optimal distribution attaining local privacy according
to Definitions~\ref{def:optimal-local-privacy}
or~\ref{def:local-diffp}, then applying the lower-bounding techniques
to developed in Section~\ref{sec:optimal-rate-proofs}.

\subsubsection{Minimax errors under differential privacy}

\newcommand{\linearclassone}{\lossset_{\rm lin}(\lipobj, \infty)}
\newcommand{\linearclasstwo}[1][]{
  \ifthenelse{\isempty{#1}}{%
    \lossset_{\rm lin}(\optdomain; \lipobj, p)
  }{%
    \lossset_{\rm lin}({#1}; \lipobj, p)
  }
}
\newcommand{\FIRSTCLASS}{\ensuremath{\lossset(\Ball_1(\radius); \lipobj)}}
\newcommand{\SECONDCLASS}{\ensuremath{\lossset(\optdomain; \lipobj, p)}}
\newcommand{\THIRDCLASS}{\ensuremath{\lossset(\Ball_q(\radius_q); \lipobj,
    p')}}

We now turn to the setting of differentially private algorithms.  We
focus on two settings for differential privacy: in the first
(Theorem~\ref{theorem:linf-bound-diffp}), we assume that
communication respects \olpd, as given by Definition~\ref{def:local-diffp}.
For the second two results, Theorems~\ref{theorem:first-class}
and~\ref{theorem:second-class}, we change the setting slightly, assuming only
that the mechanism by which the private quantity $\channelrv_i$ is
communicated to the method $\method$ is $\diffp$-differentially private and
non-interactive (recall Eq.~\eqref{eqn:local-differential-privacy}).

\paragraph{Optimal local differential privacy}
We begin with the result assuming \olpd. We use the same collection of loss
functions $\lossset$ as in Theorem~\ref{theorem:linf-bound}, that is,
$(\lipobj, \infty)$-loss functions. We also assume that the set to which the
perturbed data $\channelrv$ belong is $[-M_\infty, M_\infty]^d$, though the
specific value of $M_\infty$ is not important for the statement of
the theorem.
\begin{theorem}
  \label{theorem:linf-bound-diffp}
  Let $\lossset$ be the collection of $(\lipobj, \infty)$ loss functions, and
  assume that $\channelrv$ is optimally locally differentially private
  (Definition~\ref{def:local-diffp}), attaining $\diffp$-differential privacy
  for the set $\lossset$.  Let $d \ge 2$ and assume $\diffp \le 5/4$. Then
  \begin{equation*}
    \optgap_n^*(\lossset, \optdomain)
    \ge \frac{1}{8} \, \min\left\{ \radius \lipobj,
    \frac{\sqrt{d}}{\diffp} \,
    \frac{\radius \lipobj \sqrt{\log(2d)}}{4 \sqrt{n}} \right\}.
  \end{equation*}
\end{theorem}

As a corollary to this result, we can show an upper bound on the necessary
magnitude of the gradient bound $M_\infty$ to allow $\diffp$-differential
privacy, again applying the mirror descent
result~\eqref{eqn:mirror-descent-bound}. See
Section~\ref{sec:corollary-rate-linf-diffp} for a proof.
\begin{corollary}
  \label{corollary:linf-privacy-diffp}
  Under the conditions of Theorem~\ref{theorem:linf-bound-diffp},
  assume that $\channelprob^*$ satisfies Definition~\ref{def:local-diffp},
  attaining $\diffp$-differential privacy. Then for universal
  constants $0 < c_\ell \le c_u$, the minimax error is sandwiched as
  \begin{equation*}
    c_\ell \frac{\sqrt{d}}{\diffp} \cdot \frac{\radius
      \lipobj \sqrt{\log (2d)}}{\sqrt{n}}
    \le \optgap_n^*(\lossset, \optdomain)
    \le c_u \frac{\sqrt{d}}{\diffp} \cdot \frac{\radius
      \lipobj \sqrt{\log (2d)}}{\sqrt{n}}.
  \end{equation*}
\end{corollary}

\newcommand{\Ball}{\mathbb{B}}
\newcommand{\ball}{\Ball}
\newcommand{\qfama}{\mc{\channelprob}_\diffp}

\paragraph{Non-interactive local differential privacy}
We turn to our two results under non-interactive differential privacy, where
we no longer assume that the channel is optimally private (the data provider
simply guarantees $\diffp$-differential privacy).  In this setting, we give a
minor refinement of the definition of minimax error~\eqref{eqn:minimax-error}.
We let $\qfama$ denote the family of $\diffp$-differentially private
distributions where the channel $\channelprob$ is $\diffp$-differentially
private and non-interactive, meaning that the private
variable $\channelrv_i$ is conditionally independent of $\channelrv_j$ for $j
\neq i$ given $\statrv_i$; recall the definition~\eqref{eqn:local-differential-privacy}.
With this, the minimax error is defined as
\begin{equation*}
  \optgap_n^*(\lossset, \optdomain, \diffp)
  \defeq \inf_{\method, \channelprob \in \qfama}
  \sup_\statprob \sup_{\loss \in \lossset(\statprob)}
  \E_{\statprob, \channelprob}[
    \optgap_n(\method, \loss, \optdomain, \statprob)],
\end{equation*}
where now the infimum is taken over all $\diffp$-private, non-interactive
local mechanisms $\channelprob$, as well as all methods $\method$. Thus,
the channel $\channelprob$ and $\method$ work together to find the best
possible estimator, subject to the differential privacy constraint.

Our first lower bound applies to a class of functions that are Lipschitz with respect
to the $\ell_1$-norm, where the optimization takes place over the ball
$\Ball_1(\radius) \defeq \{ \optvar \in \R^d \, \mid \, \lone{\optvar} \leq
\radius \}$.  We define the set $\FIRSTCLASS$ to be the collection of convex
$(\lipobj, \infty)$-loss functions defined on $\Ball_1(\radius)$.  By
Example~\ref{example:median}, this loss class covers the problem of the
multi-dimensional median.  In stating our minimax bounds, we use a more
restrictive (i.e., simpler to optimize) class, the collection of $(\lipobj,
\infty)$-\emph{linear} losses:
\begin{equation*}
  \linearclassone \defeq \left\{
  \loss : \statdomain \times \R^d \to \R
  \mid \exists ~ \phi : \statdomain \to \R^d
  ~\mbox{s.t.}~
  \loss(\statsample, \optvar) = \<\phi(\statsample), \optvar\>, 
  \sup_\statsample \linf{\phi(\statsample)} \le \lipobj \right\}.
\end{equation*}
For this class, we have the following minimax rate (see
Section~\ref{sec:proof-first-class} for a proof):

\begin{theorem}
  \label{theorem:first-class}
  For the loss class $\lossset = \linearclassone$ and privacy parameter
  $\diffp = \order(1)$, assuming that the channel $\channelprob$ is
  non-interactive and $\diffp$-differentially private, there are universal
  constants $0 < c_\ell \le c_u < \infty$ such that
  \begin{align}
    \label{EqnFirstClassRate}
    c_\ell \min \left\{\frac{\sqrt{d}}{\diffp} \, \frac{r L
      \sqrt{\log(2d)}}{\sqrt{n}}, \radius \lipobj \right\} \; \leq \;
    \optgap_n^*(\lossset, \Ball_1(r), \diffp)
    \; \le \; c_u \min\left\{\frac{\sqrt{d}}{\diffp} \, \frac{r L
      \sqrt{\log(2d)}}{\sqrt{n}}, \radius \lipobj \right\}.
  \end{align}
\end{theorem}

We can also give a result for a larger class of domains and related
optimization functions. In particular, consider the loss class 
\begin{align}
  \label{EqnDefnSecondClass}
  \SECONDCLASS & \defeq \{ \loss: \statdomain \times \optdomain
  \rightarrow \real \, \mid \mbox{$\loss$ is a convex
    $(\lipobj, p)$-loss function} \big \},
\end{align}
for some $p \in [1, 2]$.  Restricting the set~\eqref{EqnDefnSecondClass}
to the smaller collection of linear functionals, we define
\begin{equation*}
  \linearclasstwo \defeq
  \big\{
  \loss : \statdomain \times \optdomain \to \R
  \mid \exists ~ \phi : \statdomain \to \R^d
  ~\mbox{s.t.}~
  \loss(\statsample, \optvar) = \<\phi(\statsample), \optvar\>, 
  \sup_\statsample \norm{\phi(\statsample)}_{p} \le \lipobj \big\}.
\end{equation*}
We then have the following result, which captures
rates of convergence for optimization of linear functionals over
$\ell_q$-norm balls of the form
\begin{equation*}
  \Ball_q(\radius_q) \defeq \{\optvar \in \R^d : \norm{\optvar}_q \le \radius_q
  \},
  ~~~
  \mbox{where}~ q \in [2, \infty].
\end{equation*}

\begin{theorem}
  \label{theorem:second-class}
  For the loss class $\lossset = \linearclasstwo[\Ball_q(\radius_q)]$ with $q
  \in [2, \infty]$ and non-interactive $\diffp$-differentially private channel
  $\channelprob$ with $\diffp = \order(1)$, there exist universal constants $0
  < c_\ell \le c_u < \infty$ such that
  \begin{align}
    \label{eqn:second-linear-rate}
    c_\ell \, \radius_q \lipobj
    \min\left\{\frac{\sqrt{d}}{\diffp} \frac{ d^{\half
        - \frac{1}{q}}}{\sqrt{n}},
    (\sqrt{n \diffp^2})^{-\frac{1}{q}},
    1 \right\}
    \; \le \;
    \optgap_n^*(\lossset, \Ball_q(\radius_q), \diffp)
    \; \le \; c_u \, \radius_q \lipobj
    \min \left\{\frac{\sqrt{d}}{\diffp} \,
    \frac{d^{\half - \frac{1}{q}}}{\sqrt{n}}, 1 \right\}.
  \end{align}
  For the loss class $\lossset = \SECONDCLASS$ from
  Eq.~\eqref{EqnDefnSecondClass}, if $\optdomain \supset
  \Ball_q(\radius_q)$,
  there exists a universal (numerical) constant $0 < c_\ell$ such that
  \begin{equation}
    \label{eqn:second-class-lower}
    c_\ell \min\left\{\frac{\sqrt{d}}{\diffp} \, \frac{\radius_q \lipobj
      d^{\half - \frac{1}{q}}}{\sqrt{n}}, \radius_q \lipobj \right\} \; \le \;
    \optgap_n^*(\lossset, \optdomain, \diffp).
  \end{equation}
\end{theorem}
\noindent
See Section~\ref{sec:proof-second-class} for a proof.

\paragraph{Remarks}
Each of our theorems and corollaries provide sharp characterizations of the
minimax rate of estimation up to the constant factors $(c_\ell, c_u)$. As
noted in the previous section, the non-private minimax rate for the class
$\FIRSTCLASS$ is $\radius \lipobj \sqrt{\log(2d)} / \sqrt{n}$.  We may compare
this with the rate in Theorems~\ref{theorem:linf-bound-diffp}
and~\ref{theorem:first-class}, as well as
Corollary~\ref{corollary:linf-privacy-diffp}. 
We see that $\diffp$-local
differential privacy has a dimension-dependent effect on the minimax rate: the
effective sample size is reduced from $n$ to $\diffp^2 n / d$. This
is a substantial reduction, as instead of a logarithmic dependence
on the dimension $d$---which one hopes for in high-dimensional settings
such as those specified by the theorem---we have a linear dependence, which
is unavoidable under the conditions of the theorems.

In Theorem~\ref{theorem:second-class} as well, the
inequalities~\eqref{eqn:second-linear-rate} provide a characterization
of the $\diffp$-private minimax rate that is tight up to constant factors.
Again, it is worthwhile to relate this minimax rate to the non-private
setting: from Theorem 1 and Eq.~(11) of \citet{AgarwalBaRaWa12}, the
non-private minimax rate for the function class $\linearclasstwo$ is lower
bounded by $\radius_q \lipobj d^{\half - \frac{1}{q}} / \sqrt{n}$.
Consequently, the price for $\diffp$-privacy is again a reduction in effective
sample size by the dimension-dependent factor $\diffp^2 / d$.

In general, stochastic gradient descent methods require interactivity---they
iteratively process the data and query for gradients at points $\optvar$
depending on the data observed---except in linear settings. We do, however,
obtain matching upper bounds for the general convex case in both
Theorems~\ref{theorem:first-class} and~\ref{theorem:second-class} using
stochastic gradient methods (which is unsurprising, as the linear setting is,
in a sense, the hardest~\cite{NemirovskiYu83,AgarwalBaRaWa12}). This leads to
the intriguing open question of whether interactivity can sharpen the results
of Theorems~\ref{theorem:first-class} and~\ref{theorem:second-class}. (For
Theorems~\ref{theorem:linf-bound}--\ref{theorem:linf-bound-diffp}, the optimal
privacy game played by the data providers allows interactivity, and hence the
results cannot be improved.) It is also interesting to note that in
Theorems~\ref{theorem:linf-bound-diffp}--\ref{theorem:second-class}, in the
$\diffp$-differentially private setting, adding Laplace noise---the most
common mechanism for achieving privacy~\cite{Dwork08}---appears to be
substantially sub-optimal: the magnitude of noise necessary to privatize the
user's data is $\Omega(d)$ larger than that provided by the optimal sampling
mechanisms we develop in the sequel.

Summarizing, each of the preceding results indicates that---no matter
the type of privacy---there is a dimension-dependent increase in sample
complexity. From
Corollaries~\ref{corollary:linf-privacy} and~\ref{corollary:lone-privacy}
we see that incorporating privacy induces a penalty of roughly $\sqrt{d} /
\sqrt{\information^*}$ in convergence rate, or an effective sample
size reduction from $n$ to $n \information^* / d$; in the differential
privacy case we have $n \mapsto n \diffp^2 / d$.
While we do not know of an explicit comparison between these two bounds, work
by \citet[Lemma 3.2]{DworkRoVa10} shows that KL divergence between
$\diffp$-differentially private distributions scales as $\diffp^2$, which
implies roughly that $\information^* \approx \diffp^2$ (though this is
informal). We see roughly similar results, though there does not appear
to be a simple mapping between information-theoretic notions of privacy
and differential privacy.



%% file: saddle-points.tex
\section{Optimal privacy-preserving distributions}
\label{sec:saddle-points}

In this section, we explore conditions for a distribution $\channelprob^*$ to
satisfy \olp as given by Definitions~\ref{def:optimal-local-privacy}
and~\ref{def:local-diffp}.  We give
a few characterizations of necessary (and sometimes sufficient) conditions
based on the compact sets $C \subset D$ for distributions $\statprob^*$ and
$\channelprob^*$ to achieve the saddle point~\eqref{eqn:saddle-point}. Our
results can be viewed as rate distortion
theorems~\cite{Gray90,CoverTh06,CsiszarKo81} (with source $\statprob$ and
channel $\channelprob$) for certain compact alphabets, though as far as we
know, they are all new. Thus, we refer to the conditional distribution
$\channelprob$, which is designed to maintain the privacy of the data
$\statrv$ by communication of $\channelrv$, interchangeably as the
privacy-preserving distribution or the channel distribution.

Note that since we wish to bound $\information(\statrv; \channelrv)$ for
general losses $\loss$, as captured in the definitions of the source 
$\sourcedistset[C]$ and communication set $\channeldistset[C, D]$ in
Eqs.~\eqref{eqn:source-distribution-set}
and~\eqref{eqn:channel-distribution-set}, we must address the case when
$\loss(\statrv, \optvar) = \<\optvar, \statrv\>$, in which case $\nabla
\loss(\statrv, \optvar) = \statrv$; this shows (by the data-processing
inequality~\cite[Chapter 5]{Gray90}) that it is no loss of generality to
assume that $\statrv \in C$ with probability 1 and that we must have
$\E[\channelrv \mid \statrv] = \statrv$. Thus we present each of our results
assuming that $\loss(\statrv, \optvar) = \<\optvar, \statrv\>$, since
a distribution $\channelprob^*$ is optimally locally private or optimally 
differentially locally private if and only if it attains the saddle point 
with \emph{this} choice of loss.

\subsection{General saddle point characterizations}

We begin with a general characterization, first defining the types of sets $C$
and $D$ that we use in our characterization of privacy. Such sets are
reasonable for many applications (recall Section~\ref{sec:loss-families}).  We
focus on the case when the compact sets $C$ and $D$ are (suitably symmetric)
norm balls:
\begin{definition}
  \label{def:rotational-invariance}
  Let $C \subset \R^d$ be a compact convex set with extreme points $u_i \in
  \R^d$, $i \in I$ for some index set $I$. Then $C$ is a \emph{rotationally
    invariant through its extreme points} if $\ltwo{u_i} = \ltwo{u_j}$ for
  each $i, j$, and for any unitary matrix $U$ such that $U u_i = u_j$ for some
  $i \neq j$, then $UC = C$.
\end{definition}
\noindent
Some examples of convex sets rotationally invariant through their extreme
points include $\ell_p$-norm balls for $p = 1, 2, \infty$, though
$\ell_p$-balls for $p \not \in\{1, 2, \infty\}$ are not.

The following theorem gives a general characterization of the minimax mutual
information for such rotationally invariant sets
by providing saddle point distributions $\statprob^*$ and
$\channelprob^*$. We provide the proof of
Theorem~\ref{theorem:finite-rotation-information} in
Section~\ref{appendix:finite-rotation-information-proof}.
\begin{theorem}
  \label{theorem:finite-rotation-information}
  Let $C$ be a compact convex polytope rotationally invariant through its
  $m < \infty$ extreme points $\{u_i\}_{i=1}^m$ and
  $D = (1 + \kappa) C$ for some $\kappa > 0$. Let $\channelprob^*$ be the
  conditional distribution of $\channelrv \mid \statrv$ that maximizes the
  entropy $H(\channelrv \mid \statrv = x)$ subject to the constraints that
  \begin{equation*}
    \E_\channelprob[\channelrv \mid \statrv = \statsample] = \statsample
  \end{equation*}
  for $\statsample \in C$ and that $\channelrv$ is supported on $(1 + \kappa)
  u_i$ for $i = 1, \ldots, m$. Then $\channelprob^*$ satisfies
  Definition~\ref{def:optimal-local-privacy}, \olp, and $\channelprob^*$ is
  (up to measure zero sets) unique.  Moreover, the distribution $\statprob^*$
  that is uniform on $\{u_i\}_{i=1}^m$ attains the saddle
  point~\eqref{eqn:saddle-point}.
\end{theorem}

\paragraph{Remarks:}
We make a few brief remarks here, deferring a somewhat deeper discussion of
the implications of Theorem~\ref{theorem:finite-rotation-information} to
Appendix~\ref{appendix:finite-rotation-information-proof}, as an understanding
of the proof helps. The theorem requires that for $\channelprob^*$
to attain the saddle point guaranteeing \olp, $\channelprob^*(\cdot \mid
\statrv = \statsample)$ should maximize the entropy of $\channelrv$ for each
$\statsample \in C$, but this is not essential. If $\statsample \not \in
\{u_i\}_{i=1}^m$, a two-phase approach still obtains optimal local privacy. We
construct a Markov chain $\statrv \to \statrv' \to \channelrv$, where
$\statrv'$ is supported on the extreme points $\{u_i\}_{i=1}^m$ of $C$. The
distribution $\statrv \to \statrv'$ may then be \emph{any} distribution
satisfying $\E[\statrv' \mid \statrv] = \statrv$; we then take the conditional
distribution $\channelprob^*(\cdot \mid u_i)$ defined for $\statrv' \to
\channelrv$ to be the maximum entropy distribution $\channelprob^*(\cdot \mid
u_i)$ defined in the theorem.  By the data processing
inequality~\cite[Chapter 5]{Gray90}, this Markov chain $\statrv \rightarrow
\statrv' \rightarrow \channelrv$ guarantees the minimax information bound
$\information(\statrv; \channelrv) \le \inf_\channelprob \sup_\statprob
\information(\statprob, \channelprob)$.

\subsection{Specific saddle point computations}

With Theorem~\ref{theorem:finite-rotation-information} in place, we can
explicitly characterize the minimax mutual information for $\ell_1$ and
$\ell_\infty$ balls by computing maximum entropy distributions. That is, we
compute the unique distributions that attain \olp---the distributions that
guarantee as much (of our definition of) privacy as possible subject to
certain constraints.  We present two propositions in this regard, providing
some discussion and giving proofs in
Sections~\ref{appendix:linf-saddle-point-proof}
and~\ref{appendix:lone-saddle-point-proof}.

First, consider the case where $\statrv \in [-L, L]^d$ and $\channelrv \in
[-M, M]^d$, where $M \ge L$. For notational convenience, we define the binary entropy $h(p) =
-p\log p - (1 - p) \log(1 - p)$. We have
\begin{proposition}
  \label{proposition:linf-saddle-point}
  For constants $M \ge L > 0$,
  let $\statrv \in [-L, L]^d$ and $\channelrv \in [-M, M]^d$ be random
  variables such that $\E[\channelrv \mid \statrv] = \statrv$ almost
  surely. Define $\channelprob^*$ to be the conditional distribution on
  $\channelrv \mid \statrv$ such that the coordinates of $\channelrv$ are
  independent, have range $\{-M, M\}$, and satisfy
  \begin{equation*}
    \channelprob^*(\channelrv_i = M \mid \statrv)
    = \half + \frac{\statrv_i}{2M}
    ~~~ \mbox{and} ~~~
    \channelprob^*(\channelrv_i = -M \mid \statrv) =
    \half - \frac{\statrv_i}{2M}.
  \end{equation*}
  Then $\channelprob^*$ satisfies Definition~\ref{def:optimal-local-privacy},
  \olp, and moreover,
  \begin{equation*}
    \sup_\statprob \information(\statprob, \channelprob^*)
    = d - d \cdot h\left(\half + \frac{L}{2M}\right).
  \end{equation*}
\end{proposition}

Before continuing, we give a slightly more intuitive understanding of
Proposition~\ref{proposition:linf-saddle-point}. Let $L = 1$ for simplicity
(this is no loss of generality by scaling).
Concavity implies that for $a, b > 0$, $\log(a) \le \log b + b^{-1}(a - b)$,
or $-\log(a) \ge -\log(b) + b^{-1}(b - a)$, so
\begin{equation*}
  -\log\left(\half - \frac{1}{2M}\right)
  \ge -\log\half + 2\cdot\frac{1}{2M}
  ~~~ \mbox{and} ~~~
  -\log\left(\half + \frac{1}{2M}\right)
  \ge -\log\half - 2 \cdot \frac{1}{2M}.
\end{equation*}
In particular, we see that
\begin{equation*}
  h\left(\half + \frac{1}{2M}\right)
  \ge - \left(\half + \frac{1}{2M}\right)
  \left(-\log 2 - \frac{1}{M}\right)
  - \left(\half - \frac{1}{2M}\right)
  \left(-\log 2 + \frac{1}{M}\right) 
  = \log 2 - \frac{1}{M^2}.
\end{equation*}
That is, we have for \emph{any} distribution $\statprob$ on $\statrv$,
where $\statrv \in [-L, L]^d$, that (in natural logarithms)
\begin{equation*}
  \information(\statprob, \channelprob^*) \le \frac{dL^2}{M^2},
\end{equation*}
and this bound is tight to $\order((M/L)^{-3})$.

We now consider the case when $\statrv \in \left\{x \in \R^d : \lone{x} \le
1\right\}$ and $\channelrv \in \left\{z \in \R^d : \lone{z} \le
M\right\}$. Here the arguments are slightly more complicated, as the
coordinates of the random variables are no longer independent, but
Theorem~\ref{theorem:finite-rotation-information} still allows us to
explicitly characterize the saddle point of the mutual information.  Before
stating the proposition, we recall that if $e_i \in \R^d$ are the standard
basis vectors, then the extreme points of the $\ell_1$-ball of radius 1 are
the $2d$ vectors $\{\pm e_i\}_{i=1}^d$.
\begin{proposition}
  \label{proposition:lone-saddle-point}
  For a constant $M > 1$,
  let $\statrv \in \{x \in \R^d : \lone{x} \le 1\}$ and $\channelrv \in \{z
  \in \R^d : \lone{z} \le M\}$ be random variables.
  Define the parameter
  \begin{equation}
    \gamma \defeq
    \log\left(\frac{2d - 2 + \sqrt{(2d - 2)^2 + 4(M^2 - 1)}}{2(M - 1)}\right),
    \label{eqn:lone-exp-normalization}
  \end{equation}
  and let $\channelprob^*$ be the conditional distribution
  on $\channelrv \mid \statrv$ such that $\channelrv$ is supported on
  $\{\pm M e_i\}_{i=1}^d$, and
  \begin{subequations}
    \label{eqn:lone-conditional}
    \begin{align}
      \channelprob^*(\channelrv = M e_i \mid \statrv = e_i)
      & = \frac{e^{\gamma}}{e^{\gamma} + e^{-\gamma} + (2 d - 2)},
      \label{eqn:lone-conditional-plus} \\
      \channelprob^*(\channelrv = -M e_i \mid \statrv = e_i)
      & = \frac{e^{-\gamma}}{e^{\gamma} + e^{-\gamma} + (2 d - 2)},
      \label{eqn:lone-conditional-minus} \\
      \channelprob^*(\channelrv = \pm M e_j \mid \statrv = e_i, j \neq i)
      & = \frac{1}{e^{\gamma} + e^{-\gamma} + (2 d - 2)}.
      \label{eqn:lone-conditional-zero}
    \end{align}
  \end{subequations}
  (For $\statrv \not \in \{\pm e_i\}$, define $\statrv'$ to be randomly
  selected in any way from among $\{\pm e_i\}$ such that $\E[\statrv' \mid
    \statrv] = \statrv$, then sample $\channelrv$ from $\statrv'$ according
  to~\eqref{eqn:lone-conditional-plus}--\eqref{eqn:lone-conditional-zero}.)
  Then $\channelprob^*$ satisfies Definition~\ref{def:optimal-local-privacy},
  \olp, and
  \begin{equation*}
    \sup_\statprob \information(\statprob, \channelprob^*)
    = \log(2d) - \log\left(e^\gamma + e^{-\gamma}
    + 2d - 2\right)
    + \gamma \frac{e^\gamma}{e^\gamma + e^{-\gamma} + 2d - 2}
    - \gamma \frac{e^{-\gamma}}{e^\gamma + e^{-\gamma} + 2d - 2}.
  \end{equation*}
\end{proposition}
\noindent
By scaling, if we have $\statrv \in \{\statsample \in \R^d :
\lone{\statsample} \le L\}$ and $\channelrv \in \{\channelval \in \R^d :
\lone{\channelrv} \le M_1\}$, then the theorem holds with $M = M_1 / L$ in
Eq.~\eqref{eqn:lone-exp-normalization} and with $\statrv = e_i$ replaced by
$\statrv = L e_i$ in Eq.~\eqref{eqn:lone-conditional}.

Proposition~\ref{proposition:lone-saddle-point} is somewhat more complex than
the $\ell_\infty$ case. We remark that the additional sampling to guarantee
that $\statrv' \in \{\pm e_i\}$ (where the conditional distribution
$\channelprob^*$ is defined) can be accomplished simply: define the random
variable $\statrv'$ so that $\statrv' = e_i \sign(x_i)$ with probability
$|x_i| / \lone{x}$. Evidently $\E[\statrv' \mid \statrv] = x$, and $\statrv
\rightarrow \statrv' \rightarrow \channelrv$ for $\channelrv$ distributed
according to $\channelprob^*$ defines a Markov chain as in our remarks
following Theorem~\ref{theorem:finite-rotation-information}. An
asymptotic expansion allows us to gain a somewhat clearer picture of the
values of the mutual information, though we do not derive upper bounds as we
did for Proposition~\ref{proposition:linf-saddle-point}.  
We have the following corollary, proved in
Appendix~\ref{appendix:lone-asymptotic-expansion}.
\begin{corollary}
  \label{corollary:lone-asymptotic-expansion}
  Let $\channelprob^*$ denote the conditional distribution in
  Proposition~\ref{proposition:lone-saddle-point}, where
  $\statrv$ and $\channelrv$ lie in $\ell_1$-balls of radii $L$ and $M$,
  respectively. Then
  \begin{equation*}
    \sup_{\statprob} \information(\statprob, \channelprob^*)
    = \frac{d L^2}{2 M^2} + \Theta\left(\min\left\{\frac{d^3L^4}{M^4},
    \frac{\log^4(d)}{d}\right\}\right).
  \end{equation*}
\end{corollary}

\subsection{Saddle points for differentially private communication}

Our final result in this section characterizes saddle points for
distributions satisfying Definition~\ref{def:local-diffp}. Such calculations
are, in general, non-trivial, so we restrict our attention to results
necessary for the setting of Theorem~\ref{theorem:linf-bound-diffp}.
To that end, we focus on the case where $C$ and $D$ are $\ell_\infty$ balls,
which is relevant for high-dimensional statistical and optimization
settings. Without loss of generality (by scaling), we may take $C = [-1, 1]^d$
and $D = [-M, M]^d$. We have
\begin{proposition}
  \label{proposition:linf-diffp-saddle-point}
  For a constant $M \ge 1$, let $\statrv \in [-1, 1]^d$ and $\channelrv \in
  [-M, M]^d$ be random variables such that $\E[\channelrv \mid
    \statrv] = \statrv$ almost surely.  Fix any $\statsample \in \{-1, 1\}^d$
  and for $k = \{0, 2, 4, \ldots, 2\ceil{d/2} - 2\}$ define the constants
  $\channeldensity_k^+ \ge 0$ and $\channeldensity_k^- \ge 0$ to satisfy the
  linear equations
  \begin{equation*}
    M \channeldensity_k^+ \!\!\!\!\!\!\!\!
    \sum_{\channelval \in \{-1, 1\}^d
      : \<\channelval, \statsample\> > k} \!\!\!\!\!\!\!\!
    \channelval
    + M \channeldensity_k^- \!\!\!\!
    \!\!\!\! \sum_{\channelval \in \{-1, 1\}^d :
      \<\channelval, \statsample\> \le k} \!\!\!\!\!\!\!\! \channelval
    ~ = x
    ~~~ \mbox{and} ~~~
    \channeldensity_k^+ + \channeldensity_k^- = 1.
  \end{equation*}
  Set $k^* \in \argmin_k\{\channeldensity_k^+ / \channeldensity_k^-\}$.
  The optimal differential privacy~\eqref{eqn:optimal-diffp}
  for the sets $C = [-1, 1]^d$ and $D = [-M, M]^d$ is
  \begin{equation*}
    \optdiffp(C, D) = \log
    \frac{\channeldensity_{k^*}^+}{\channeldensity_{k^*}^-}
    = \min_{k \in \{0, 2, \ldots, 2 \ceil{d/2} - 2\}} \log
    \frac{\channeldensity_{k}^+}{\channeldensity_{k}^-}.
  \end{equation*}

  The distributions attaining \olpd are characterized as follows.
  Define $\channelprob_k^*$ to be the distribution supported on
  $\{-M, M\}^d$ with probability mass function
  defined by
  \begin{equation}
    \channelprob_k^*(\channelrv = M \channelval
    \mid \statrv = \statsample)
    = \begin{cases} \channeldensity_{k}^+ & \mbox{if} ~
      \<\channelval, \statsample\> > k \\
      \channeldensity_{k}^- & \mbox{if} ~
      \<\channelval, \statsample\> \le k
    \end{cases}
    \label{eqn:k-based-diffp-linf}
  \end{equation}
  for $\channelval, \statsample \in \{-1, 1\}^d$.  (For $\statrv \not\in \{-1,
  1\}^d$, define $\statrv'$ to be randomly chosen from $\{-1, 1\}^d$ such that
  $\E[\statrv' \mid \statrv] = \statrv$, then sample $\channelrv$ according to
  the above p.m.f.)  A distribution $\channelprob^*$ satisfies
  Definition~\ref{def:local-diffp}, \olpd, if and only if it can be written as
  a convex combination of those $\channelprob_k^*$ for which $k \in \argmin_k
  \{\channeldensity_k^+, \channeldensity_k^-\}$, that is,
  \begin{equation*}
    \channelprob^*
    = \!\!\! \sum_{k \in \argmin_k \{\channeldensity_k^+,
      \channeldensity_k^-\}}
    \!\!\!\!\!\!\!\!\!
    \beta_k \channelprob_k^*,
    ~~ \mbox{where} ~~
    \beta_k \ge 0 ~~ \mbox{and}
    \sum_{k \in \argmin_k \{\channeldensity_k^+, \channeldensity_k^-\}}
    \!\!\!\!\!\! \beta_k = 1.
  \end{equation*}
\end{proposition}

The proof of Proposition~\ref{proposition:linf-diffp-saddle-point} is
technical, and we defer it to
Section~\ref{appendix:proof-linf-diffp-saddle-point}.
We make a few remarks, however. First, we provide a simplified
explanation of the linear equations in the proposition.  By
symmetry, no matter the value of $\statsample \in \{-1, 1\}^d$ chosen, the
same $\channeldensity_k^+$ and $\channeldensity_k^-$ solve the linear
equations.
Proposition~\ref{proposition:linf-diffp-saddle-point} shows the
\emph{structure} of the distribution attaining \olpd.  That is, the
proposition shows that the distribution $\channelprob^*(\cdot \mid
\statsample)$ assigns mass only on the points $\channelval \in \{-M, M\}^d$,
and moreover, it assigns one of two masses: either $\channeldensity^+$ or
$\channeldensity^-$. To sample from this distribution given an initial point
$\statsample$, one simply flips a coin with bias probabilities
$\{\channeldensity_k^+, \channeldensity_k^-\}$, and depending on the result of
the coint flip, samples $\channelrv \in \{-M, M\}^d$ uniformly from one of the
sets $\{\channelval : \<\channelval, \statsample\> > k/M\}$ or $\{\channelval
: \<\channelval, \statsample\> \le k/M\}$.



%% file: statistical-rate-proofs.tex
\section{Proofs of Statistical Rates}
\label{sec:optimal-rate-proofs}

In this section, we prove
Theorems~\ref{theorem:linf-bound}--\ref{theorem:second-class} as well
as Corollaries~\ref{corollary:linf-privacy},
\ref{corollary:lone-privacy}, and~\ref{corollary:linf-privacy-diffp}.
Our proofs are based on classical information-theoretic techniques
from statistical minimax theory~\cite{YangBa99,Yu97}, and also exploit
some additional results due to~\citet{AgarwalBaRaWa12}. At a high
level, our approach consists of the following steps.  Beginning with
an appropriately chosen finite set $\hypercubeset$, we assign a risk
function $\risk_\cubecorner$ to each member $\cubecorner \in
\hypercubeset$.  The resulting collection
$\{\risk_\cubecorner\}_{\cubecorner \in \hypercubeset}$ of risk
functions is chosen so that they ``separate'' points in the set
$\hypercubeset$, meaning that if $\optvar \in \optdomain$ is a point
that approximately minimizes the function $\risk_\cubecorner$, then
for any $\altcubecorner \neq \cubecorner$, the point $\optvar$
\emph{cannot} also be an approximate minimizer of
$\risk_\altcubecorner$.  This separation property allows us to deduce
that statistical estimation implies the existence of a testing
procedure that distinguishes $\cubecorner$ from $\altcubecorner$ for
$\altcubecorner \neq \cubecorner$.  We then use lower bounds on the
error probability in tests, such asFano's and Le Cam's
inequalities~\cite{Yu97}, to obtain a lower bound on the testing
error.  These inequalities depend on the mutual information between
the random variable $\statrv_i$ and the vector $\channelrv_i$
communicated, so the final step is to obtain good upper bounds on this
mutual information.  In the next section, we describe in more detail
this reduction, finishing with an outline of the proofs to follow.

\subsection{Reduction to testing}
\label{sec:testing-reduction-outline}

We begin by describing the reduction that lower bounds the minimax error by
the error of a
testing problem.  It assumes a given collection of risk functions
$\{\risk_\cubecorner \}_{\cubecorner \in \hypercubeset}$ indexed by a finite
set $\hypercubeset$; see the individual theorem proofs to follow for
constructions of the particular collections used in our analysis.  For each
$\cubecorner \in \hypercubeset$, we choose some representative
$\optvar^*_\cubecorner \in \argmin_{\optvar \in \optdomain}
\risk_\cubecorner(\optvar)$ of the set of all minimizing vectors.  Our
reduction is based on a discrepancy measure between pairs of risk functions,
first introduced by~\citet{AgarwalBaRaWa12}, defined as
\begin{equation*}
  \discrepancy(\risk_\cubecorner, \risk_\altcubecorner) \defeq
  \inf_{\optvar \in \optdomain} \left[ \risk_\cubecorner(\optvar) +
    \risk_\altcubecorner(\optvar) - \risk_\cubecorner
    (\optvar_\cubecorner^*) - \risk_\altcubecorner
    (\optvar_\altcubecorner^*)\right].
\end{equation*}
The $\discrepancy$-separation of the set $\hypercubeset$ is defined as
\begin{equation}
  \label{eqn:min-discrepancy}
  \risksep(\hypercubeset) \defeq \min \left \{\discrepancy
  (\risk_\cubecorner, \risk_\altcubecorner) : \cubecorner,
  \altcubecorner \in \hypercubeset, \cubecorner \neq \altcubecorner
  \right \}.
\end{equation}
When the set $\hypercubeset$ is clear from context, we use $\risksep$ as
shorthand for this separation.

The key to the definition~\eqref{eqn:min-discrepancy} is that the separation
allows us to lower bound the expected optimality gap of a statistical method
$\method$ by the probability of error in a hypothesis test. That is, with the
family $\{\risk_\cubecorner\}_{\cubecorner \in \hypercubeset}$, we can define
the \emph{canonical hypothesis testing} problem: nature chooses a uniformly
random $\cuberv \in \hypercubeset$, and conditional on $\cuberv =
\cubecorner$, the $n$ observations $\statrv_i$, $i = 1, \ldots, n$, are drawn
independently from a distribution $\statprob_\cubecorner$ satisfying
$\risk_\cubecorner(\optvar) = \E_{\statprob_\cubecorner}[\loss(\statrv,
  \optvar)]$. In the private setting, instead of observing $\statrv_i$, the
learner observes the privatized vector $\channelrv_i$, and given the set
$\{\channelrv_i\}_{i=1}^n$, the goal is to determine the underlying index
$\cubecorner$.

Recall the definition~\eqref{eqn:error-metric} of the excess risk
$\optgap_n$.  The previously described hypothesis testing problem can
be used to establish lower bounds on the estimation error, as
demonstrated in the following:
\begin{lemma}
  \label{lemma:expectation-lower-bound}
  Let $\statprob$ be a joint distribution over $\statrv \in \R^d$ and
  $\cuberv \in \hypercubeset$ such that $\statrv$ are i.i.d.\ given
  $\cuberv$ and
  \begin{equation*}
    \E_\statprob\left[\loss(\statrv, \optvar) \mid \cuberv =
      \cubecorner\right] = \risk_\cubecorner(\optvar).
  \end{equation*}
  Let $\channelprob$ be the conditional distribution of $\channelrv$
  given the observations $\statrv$.  Then for any minimization
  procedure $\method$, there exists a hypothesis test
  $\what{\cubecorner}_\method : (\channelrv_1, \ldots, \channelrv_n)
  \rightarrow \hypercubeset$ such that
  \begin{equation*}
    \E_{\statprob, \channelprob} \left[\optgap_n(\method, \loss,
      \optdomain, \statprob)\right] \ge
    \frac{\risksep(\hypercubeset)}{2} \P_{\statprob,
      \channelprob}\left[\what{\cubecorner}_\method(\channelrv_1,
      \ldots, \channelrv_n) \neq \cuberv\right].
  \end{equation*}
\end{lemma}
\noindent
This result is a variant of Lemma 2 due to~\citet{AgarwalBaRaWa12}) It
shows that if we can bound the probability of error of any hypothesis
test for identifying $\cuberv$ based on the sample $\channelrv_1,
\ldots, \channelrv_n$, we have lower bounded the rate at which it is
possible to minimize the risk $\risk$.

The remaining challenge is to provide a lower bound on the error of
hypothesis testing problems.  To do so, we apply one of two well-known
inequalities: Fano's inequality~\cite{CoverTh06}, which applies when
$|\hypercubeset| > 2$, or Le Cam's method~\cite{LeCam73,Yu97}, which
we apply when $|\hypercubeset| = 2$. Let $\cuberv \in \hypercubeset$
be chosen uniformly at random from $\hypercubeset$. If a procedure
observes random variables $\channelrv_1, \ldots, \channelrv_n$, Fano's
inequality ensures that for any estimate $\what{\cubecorner}$ of
$\cuberv$---that is, any measurable function $\what{\cubecorner}$ of
$\channelrv_1, \ldots, \channelrv_n$---the test error probability
satisfies the lower bound
\begin{equation}
  \label{eqn:fano}
  \P(\what{\cubecorner}(\channelrv_1, \ldots, \channelrv_n) \neq
  \cuberv) \ge 1 - \frac{\information(\channelrv_1, \ldots,
    \channelrv_n; \cuberv) + \log 2}{\log |\hypercubeset|}.
\end{equation}
By contrast, Le Cam's method
provides lower bounds on the probability of error in binary hypothesis testing
problems. In this setting, assume that $\hypercubeset = \{-1, 1\}$ has two
elements, and let $\cuberv \in \hypercubeset$ be chosen uniformly at random
from $\hypercubeset$. If a procedure observes random variables $\channelrv_1,
\ldots, \channelrv_n$ distributed according to $\channelprob^n_1$ if $\cuberv
= 1$ and $\channelprob^n_{-1}$ if $\cuberv = -1$, then any estimate
$\what{\cubecorner}$ of $\cuberv$ satisfies the lower bound
\begin{equation}
  \label{eqn:le-cam}
  \P\left(\what{\cubecorner}(\channelrv_1, \ldots,
  \channelrv_n) \neq \cuberv\right)
  \ge \half - \half \tvnorm{\channelprob^n_1 - \channelprob^n_{-1}}.
\end{equation}
See, for example, \citet[Section 2]{LeCam73} or \citet[Lemma 1]{Yu97}.

Using the lower bound provided by
Lemma~\ref{lemma:expectation-lower-bound} and Fano's
inequality~\eqref{eqn:fano} or Le Cam's inequality~\eqref{eqn:le-cam},
the structure of our remaining proofs becomes more apparent. Each
lower bound argument proceeds in three steps:
\begin{enumerate}[(1)]
\item \label{item:loss-def-step} We construct a collection of loss functions
  satisfying Definition~\ref{def:lp-loss-class}, computing the minimal
  separation~\eqref{eqn:min-discrepancy} so that we may apply
  Lemma~\ref{lemma:expectation-lower-bound}.
\item \label{item:fano-bound-step} The second step is to provide an upper
  bound on the appropriate information theoretic quantity in order to apply
  Fano's inequality~\eqref{eqn:fano}, in which case we bound
  $\information(\channelrv_1, \ldots, \channelrv_n; \cuberv)$, or Le Cam's
  inequality~\eqref{eqn:le-cam}, where we bound $\tvnorm{\channelprob^n_1 -
    \channelprob_{-1}^n}$.  This step requires the most work and constitutes
  the major arguments in this section. We provide these bounds using one of two
  techniques:
  \begin{enumerate}[(a)]
  \item In the proofs of Theorems~\ref{theorem:linf-bound},
    \ref{theorem:lone-bound}, and~\ref{theorem:linf-bound-diffp}, we use a
    distribution $\channelprob$ that satisfies
    Definition~\ref{def:optimal-local-privacy} of \olp
    (Theorems~\ref{theorem:linf-bound} and~\ref{theorem:lone-bound}) or
    Definition~\ref{def:local-diffp} of \olpd
    (Theorem~\ref{theorem:linf-bound-diffp}). We can then explicitly upper
    bound the mutual information $\information(\channelrv_{1:n}; \cuberv)$
    using the definition of $\channelprob$ and the losses $\loss$ from
    step~\ref{item:loss-def-step}. (See
    Lemmas~\ref{lemma:linear-linf-mutual-information},
    \ref{lemma:linf-mutual-information}, \ref{lemma:lone-mutual-information},
    and \ref{lemma:diffp-linear-information} in the subsequent sections.)
  \item In the proofs of Theorems~\ref{theorem:first-class}
    and~\ref{theorem:second-class}, we use a bound on mutual
    information in non-interactive locally differentially private schemes,
    which we recently presented~\cite{DuchiJoWa13_parametric}.
    This requires a careful packing construction in conjunction with
    the loss choice from step~\ref{item:loss-def-step}.
  \end{enumerate}
\item The final step is to use the results of
  Steps~\ref{item:loss-def-step} and~\ref{item:fano-bound-step} in the
  application of Lemma~\ref{lemma:expectation-lower-bound} and Fano's
  inequality~\eqref{eqn:fano} (when the dimension $d$ is low, we use
  Le Cam's inequality~\eqref{eqn:le-cam}). This then yields the
  theorems.
\end{enumerate}

\noindent
We now turn to the proofs of the theorems.  We provide each proof in turn,
following the steps in the preceding outline.



\subsection{Proof of Theorem~\ref{theorem:linf-bound}}
\label{sec:proof-theorem-linf-bound}

We provide the most detail in the proof of this theorem, as it
closely exhibits the blueprint by which we prove the other results.

\subsubsection{Constructing well-separated losses}
\label{sec:linear-losses}

The first step in proving our minimax lower bounds is to construct a family of
well-separated risks.  For Theorem~\ref{theorem:linf-bound}, we use one of
two families of loss functions: linear losses and median-based losses.  Each
of these gives a well-separated family with subgradients
bounded in $\ell_\infty$-norm.

\paragraph{Linear losses}
Our first collection of risk functionals is relatively simple, based on
families of linear loss functions; we describe the sampling scheme
for $\statrv$ to generate them.  Let $\hypercubeset = \{\pm e_i\}_{i=1}^d$,
where the vectors $e_i$ are the standard basis vectors in $\R^d$, whence
$|\hypercubeset| = 2d$. Fix $\delta \in \openleft{0}{1}$,
which we specify later. We choose the distribution
$\statprob$ on $\statrv$ to be nearly uniform on $\statrv \in \{-1, 1\}^d$.
Conditional on the parameter $\cubecorner \in \hypercubeset$,
we use the following sampling distribution for $\statrv$:
\begin{equation}
  \label{eqn:lone-linear-samples}
  \mbox{Choose~} \statrv \in \{-1, 1\}^d ~
  \mbox{with~independent~coordinates,~where}~ \statrv_j =
  \begin{cases}
    1 & \mbox{w.p.}~ \frac{1 + \delta \cubecorner_j}{2} \\
    -1 & \mbox{w.p.}~ \frac{1 - \delta \cubecorner_j}{2}.
  \end{cases}
\end{equation}
Now for $\lipobj \in \R_+$ we may define the linear loss functions
\begin{equation}
  \label{eqn:linear-loss}
  \loss(\statrv, \optvar) \defeq \lipobj \<\statrv, \optvar\> \; = \;
  \lipobj \sum_{j=1}^d \statrv_j \optvar_j.
\end{equation}
By inspection, the final risk is
$\risk_\cubecorner(\optvar) = \E_\statprob[\<\optvar, \statrv\>] =
\lipobj \delta \<\cubecorner, \optvar\>$.
We obtain the following result on the separation of the risks.
\begin{lemma}
  \label{lemma:linear-linf-risk-separation}
  Given the sampling scheme~\eqref{eqn:lone-linear-samples},
  \begin{enumerate}[(a)]
  \item The loss~\eqref{eqn:linear-loss} is $\lipobj$-Lipschitz with
    respect to the $\ell_1$-norm.
  \item For the optimization domain $\optdomain = \{\optvar \in \R^d :
    \lone{\optvar} \le \radius\}$, the $\discrepancy$-separation of
    the set $\hypercubeset = \{\pm e_i\}_{i=1}^d$ is
    $\risksep(\hypercubeset) = \lipobj \radius \delta$.
  \end{enumerate}
\end{lemma}
\begin{proof}
  The first statement of the lemma is immediate, since $\nabla \loss(\statrv,
  \theta) = \lipobj \statrv$ and $\linf{\statrv} \le 1$
  (cf.~\cite{HiriartUrrutyLe96ab}). For the second, we verify that
  $\inf_{\optvar \in \optdomain} [\risk_\cubecorner(\optvar) +
    \risk_\altcubecorner(\optvar)] - \risk_\cubecorner(\optvar_\cubecorner^*)
  - \risk_\altcubecorner(\optvar_\altcubecorner^*) \ge \lipobj \radius
  \delta$. To do so, we compute the minimizers of $\risk_\cubecorner$: since
  $\ell_1$ and $\ell_\infty$ are dual norms, we see that for $\cubecorner \in
  \hypercubeset$,
  \begin{equation*}
    \inf_{\lone{\optvar} \le \radius} \risk_\cubecorner(\optvar) =
    \inf_{\lone{\optvar} \le \radius} \lipobj \delta
    \<\cubecorner, \optvar\> = -\lipobj \delta \radius
    \linf{\cubecorner} = -\lipobj \delta \radius,
  \end{equation*}
  and the minimizer is uniquely attained at $\optvar_\cubecorner^* =
  -\radius \cubecorner$. Then we have for any $\altcubecorner \neq
  \cubecorner$ that
  \begin{equation*}
    \inf_{\lone{\optvar} \le 1}\left[\<\cubecorner + \altcubecorner,
      \optvar\> \right] + \linf{\cubecorner} + \linf{\altcubecorner} =
    -\linf{\cubecorner + \altcubecorner} + \linf{\cubecorner} +
    \linf{\altcubecorner} \ge -1 + 1 + 1 = 1,
  \end{equation*}
  since any non-zero coefficients of $\cubecorner$ and
  $\altcubecorner$ have differing signs. Multiplying the result by $\lipobj
  \radius \delta$ completes the proof.
\end{proof}

\paragraph{Median-type losses}

We now describe a class of median-type losses, one with more general
applicability than the linear losses of Section~\ref{sec:linear-losses}.  Let
$\hypercubeset \subset \{-1, 1\}^d$ be a subset of the binary hypercube such
that for all distinct pairs $\cubecorner \neq \cubecorner'$, we have
$\lone{\cubecorner - \cubecorner'} \ge d / 2$, or equivalently
$\norm{\cubecorner - \cubecorner'}_0 \ge d/4$.  From the Gilbert-Varshamov
bound~\cite[Lemma 4]{Yu97} there are sets of this form with cardinality
at least $\card(\hypercubeset) \ge \exp(d/8)$.  We define the distribution on
$\statrv$, conditional on $\cubecorner \in \hypercubeset$, as
\begin{equation}
  \label{eqn:full-samples}
  \mbox{Choose~} \statrv \in \{-1, 1\}^d
  ~\mbox{with~independent~coordinates,~where~} \statrv_j =
  \begin{cases}
    1 & \mbox{w.p.~} \frac{1 + \delta\cubecorner_j}{2} \\
    -1 & \mbox{w.p.~} \frac{1 - \delta\cubecorner_j}{2}.
  \end{cases}
\end{equation}
For $\lipobj > 0$, we then define the median-type loss function
\begin{equation}
  \label{eqn:lone-loss}
  \loss(\statrv, \optvar) = \lipobj \lone{\radius \statrv - \optvar},
\end{equation}
which under the sampling scheme~\eqref{eqn:full-samples}
gives rise to the risk functional
\begin{equation*}
  \risk_\cubecorner(\optvar)
  = \lipobj \sum_{j=1}^d \frac{1 +
    \delta \cubecorner_j}{2} \left|\optvar_j - \radius\right| + \frac{1
    - \delta \cubecorner_j}{2} \left|\optvar_j + \radius\right| =
  \lipobj \left(\frac{1 + \delta}{2} \lone{\optvar - \radius
    \cubecorner} + \frac{1 - \delta}{2} \lone{\optvar + \radius
    \cubecorner} \right).
\end{equation*}
By construction, whenever $\optdomain$ contains the $\ell_\infty$ ball
of radius $\radius$, this risk function has the unique minimizer
\begin{equation*}
  \optvar^*_\cubecorner \defeq \argmin_{\optvar \in \optdomain}
  \risk_\cubecorner(\optvar) = \radius \cubecorner \in r \{-1, 1\}^d
  \subset \optdomain.
\end{equation*}
The following lemma, due to~\citet{AgarwalBaRaWa12}, captures the
separation properties of the collection
$\{\risk_\cubecorner\}_{\cubecorner \in \hypercubeset}$ of risk
functionals:
\begin{lemma}
  \label{lemma:median-risk-separation}
  Assume that $\optdomain$ contains $[-\radius, \radius]^d$ and let
  $\risk_\cubecorner$ be defined as in the preceding paragraph.  If
  $\cubecorner, \altcubecorner \in \hypercubeset$ with $\cubecorner \neq
  \altcubecorner$, the discrepancy $\discrepancy(\risk_\cubecorner,
  \risk_\altcubecorner) \ge \radius \lipobj d \delta / 2$.
\end{lemma}

As a final remark, for random variables $\statrv \in \R^d$, the loss
function~\eqref{eqn:lone-loss} is Lipschitz continuous (for appropriate choice
of $\lipobj$) for \emph{any} distribution $\statprob$ on $\statrv$.
Specifically, defining the $\sign(\cdot)$ function coordinate-wise, we have
the subgradient equality $\partial \loss(x, \optvar) = \lipobj \sign(\optvar -
\radius x)$. Thus, for any $p \in [1,\infty]$ and $\lipobj_p \ge
0$, setting $\lipobj = \lipobj_p d^{-1/p}$ yields a member of the collection
of $(\lipobj_p, p)$-loss functions.

\subsubsection{Bounding the mutual information}

As outlined in Section~\ref{sec:testing-reduction-outline}, the second step in
our lower bound proofs is to bound the mutual information
$\information(\channelrv_1, \ldots, \channelrv_n; \cuberv)$, where
$\channelrv_i$ are the private views available to the learning method.  Here
we provide mutual information bounds for the family of linear losses
(Lemma~\ref{lemma:linear-linf-mutual-information}) and median-based losses
(Lemma~\ref{lemma:linf-mutual-information}).  Each of these mutual information
bounds---and our subsequent bounds on mutual information---proceed by using
independence to reduce the problem to estimating the mutual information
$\information(\channelrv; \cuberv \mid \optvar)$ for a single randomized
gradient sample $\channelrv$. Then, careful calculation of the distribution of
$\channelrv \mid \cuberv$ yields the final inequalities.  As the proofs are
somewhat long and technical, we defer them to
Appendix~\ref{appendix:mutual-information-computation}.

\begin{lemma}
  \label{lemma:linear-linf-mutual-information}
  Let $\cuberv$ be drawn uniformly at random from $\hypercubeset = \{\pm
  e_i\}_{i=1}^d$.  Let $\statrv$ have the
  distribution~\eqref{eqn:lone-linear-samples} conditional on $\cuberv =
  \cubecorner$ and assume $\loss(\statrv, \optvar) = \lipobj \<\statrv,
  \optvar\>$.  Let $\channelrv$ be constructed according to the conditional
  distribution specified by Proposition~\ref{proposition:linf-saddle-point}
  given a subgradient $\partial \loss(\statrv_i; \optvar)$ with
  $\channelrv \in [-M_\infty, M_\infty]^d$, where $M_\infty \ge \lipobj$.  Then
  \begin{equation*}
    \information(\channelrv_1, \ldots, \channelrv_n; \cuberv)
    \le n \frac{\delta^2 \lipobj^2}{M_\infty^2}.
  \end{equation*}
\end{lemma}
\noindent
See Appendix~\ref{appendix:mutual-information-computation-linear-linf}
for a proof of Lemma~\ref{lemma:linear-linf-mutual-information}.

\begin{lemma}
  \label{lemma:linf-mutual-information}
  Let $\cuberv$ be drawn uniformly at random from a set $\hypercubeset \subset
  \{-1, 1\}^d$. Let $\statrv$ have the distribution~\eqref{eqn:full-samples}
  conditional on $\cuberv = \cubecorner$ and assume $\loss(\statrv, \optvar) =
  \lipobj \lone{\radius \statrv - \optvar}$, where $\radius > 0$ is a
  constant.
  Let $\channelrv$ be constructed according to the distribution specified by
  Proposition~\ref{proposition:linf-saddle-point} conditional on a subgradient
  $\partial \loss(\statrv_i; \optvar)$, where $\channelrv \in [-M_\infty,
    M_\infty]^d$ and $M_\infty \ge \lipobj$.  Then
  \begin{equation*}
    \information(\channelrv_1, \ldots, \channelrv_n; \cuberv)
    \le n \frac{\delta^2 \lipobj^2 d}{M_\infty^2}.
  \end{equation*}
\end{lemma}
\noindent
See Appendix~\ref{appendix:mutual-information-computation-linf}
for a proof of Lemma~\ref{lemma:linf-mutual-information}.

\subsubsection{Applying testing inequalities}

Having established the two families of loss functions we consider and
the resultant mutual information bounds, it remains to apply
Lemma~\ref{lemma:expectation-lower-bound} and a testing inequality.
We begin by proving part (a) of the theorem.

\paragraph{Proof of Theorem~\ref{theorem:linf-bound}(a)}
We divide the proof of part (a) of the theorem into two parts: one assuming
the dimension $d \ge 9$ and the other assuming $d < 9$. For the
first, we use Fano's inequality~\eqref{eqn:fano}, while for the second,
an application of Le Cam's method~\eqref{eqn:le-cam} completes the result.
For both results, we use the median-type loss $\loss(\statrv, \optvar)
= \lipobj \lone{\radius \statrv - \optvar}$.
We first recall the beginning of the previous section, stating the following
application of Lemma~\ref{lemma:expectation-lower-bound} and Fano's
inequality~\eqref{eqn:fano}:
\begin{align}
  \frac{2}{\risksep(\hypercubeset)} \E_{\statprob, \channelprob}
  \left[\optgap_n(\method, \loss, \optdomain, \statprob)\right] \ge
  \P_{\statprob, \channelprob}\left(\what{\cubecorner}(\method) \neq
  \cuberv\right) \ge 1 - \frac{\information(\channelrv_1, \ldots,
    \channelrv_n; \cuberv) + \log 2}{ \log |\hypercubeset|}.
  \label{eqn:risk-to-fano}
\end{align}

Now we give the proof of the first statement of the theorem
in the case that $d \ge 9$.
Applying Lemmas~\ref{lemma:median-risk-separation}
and~\ref{lemma:linf-mutual-information}, we immediately have the
following specialization of the inequality~\eqref{eqn:risk-to-fano}:
\begin{align*}
  \frac{4}{\radius \lipobj d \delta} \E_{\statprob, \channelprob}\left[
    \optgap_n(\method, \loss, \optdomain, \statprob)\right] & \ge 1 -
  \frac{\log 2}{\log |\hypercubeset|} - n \frac{\delta^2
    \lipobj^2 d}{M_\infty^2 \log |\hypercubeset|}.
\end{align*}
Taking the set $\hypercubeset \subset \{-1, 1\}^d$ to be a $d/4$
packing of the hypercube $\{-1, 1\}^d$ satisfying $|\hypercubeset| \ge
\exp(d/8)$, as in our construction of median-type losses
in Section~\ref{sec:linear-losses}, we see that
\begin{equation*}
  \frac{4}{\radius \lipobj d \delta} \E_{\statprob, \channelprob} \left[
    \optgap_n(\method, \loss, \optdomain, \statprob)\right] \ge 1 -
  \frac{8 \log 2}{d} - n\frac{8 \delta^2
    \lipobj^2}{M_\infty^2}.
\end{equation*}
The numerical inequality $8 \log 2 < 6$
coupled with the preceding bound implies
\begin{equation*}
  \frac{4}{\radius d \lipobj \delta} \E_{\statprob, \channelprob}\left[
    \optgap_n(\method, \loss, \optdomain, \statprob)\right]
  > 1 - \frac{6}{d}
  - 8 n \frac{\delta^2 \lipobj^2}{M_\infty^2}.
\end{equation*}
By our assumption that $d \ge 9$, if we choose $\delta = \min\{M_\infty / 8
\lipobj \sqrt{n}, 1\}$, we are guaranteed the lower bound
$\frac{4}{\radius d \lipobj \delta} \E_{\statprob, \channelprob}
\left[\optgap_n(\method, \loss, \optdomain, \statprob)\right] >
\frac{1}{5}$, or equivalently
\begin{align*}
  \E_{\statprob, \channelprob} \left[\optgap_n(\method, \loss,
    \optdomain, \statprob) \right] & > \frac{\radius d \lipobj
    \delta}{20} = \frac{1}{20} \cdot \min\left\{
  \radius d \lipobj,
  \frac{M_\infty \radius
    d}{8\sqrt{n}}\right\}.
\end{align*}

When $d < 9$, we may reduce to the case that $d = 1$, since a lower
bound in this setting extends to higher dimensions (though
we may lose dimension dependence).
For this case, we use the packing set
$\hypercubeset = \{-1, 1\}$ with the linear loss function from
Lemma~\ref{lemma:linear-linf-risk-separation}, which has
$\risksep(\hypercubeset) = \lipobj \radius \delta$.  In this case, the
marginal distribution $\channelprob(\cdot \mid \cuberv)$ is given by
\begin{equation*}
  \channelprob(\channelrv = \channelval \mid \cuberv = 1)
  = \half + \begin{cases}
    \frac{\delta \lipobj}{2 M} & \mbox{if~} \channelval = M \\
    -\frac{\delta \lipobj}{2 M} & \mbox{otherwise, i.e.~if~}
    \channelval = -M.
  \end{cases}
\end{equation*}
Now, let $\channelprob^n(\cdot \mid \cuberv)$ denote the distribution of
$\channelrv_1, \ldots, \channelrv_n$ conditional on $\cuberv$. Then applying
Lemma~\ref{lemma:expectation-lower-bound} and Le Cam's lower
bound~\eqref{eqn:le-cam}, we obtain the inequality
\begin{equation*}
  \frac{2}{\radius \lipobj \delta}
  \E_{\statprob, \channelprob}[\optgap_n(\method, \loss,
    \optdomain, \statprob)]
  \ge \P_{\statprob, \channelprob}
  \left(\what{\cubecorner}(\method) \neq \cuberv\right)
  \ge \half - \half \norm{\channelprob^n(\cdot \mid \cuberv = 1)
    - \channelprob^n(\cdot \mid \cuberv = -1)}_{\rm TV}.
\end{equation*}
A standard result on the total variation distance of Bernoulli distributions
(see Lemma~\ref{lemma:tv-norm-bound} in
Appendix~\ref{appendix:tv-norm-bounding}) implies that
\begin{equation*}
  \tvnorm{\channelprob^n(\cdot \mid \cuberv = 1)
    - \channelprob^n(\cdot \mid \cuberv = -1)}
  \le \frac{\delta \lipobj}{M} \sqrt{(3/2) n}
\end{equation*}
if $\delta \le M / (3 \lipobj)$. Thus we have the bound
\begin{equation}
  \label{eqn:apply-le-cam}
  \frac{2}{\radius \lipobj \delta}
  \E_{\statprob, \channelprob}[\optgap_n(\method, \loss,
    \optdomain, \statprob)]
  \ge 
  \frac{1}{2} - \frac{\sqrt{3 n}}{2 \sqrt{2}}
  \cdot \frac{\delta \lipobj}{M}.
\end{equation}
Multiplying both sides by $\radius \lipobj \delta$, then setting
$\delta = \min\{M / (3 \lipobj \sqrt{n}), 1\} \le M/(3\lipobj)$, we have
\begin{equation*}
  \E_{\statprob, \channelprob}[\optgap_n(\method, \loss,
    \optdomain, \statprob)]
  \ge 
  \left(\half - \frac{1}{2 \sqrt{6}}\right) \frac{\radius \lipobj \delta}{2}
  \ge \frac{\sqrt{6} - 1}{4 \sqrt{6}} \radius \lipobj
  \min\left\{\frac{M}{3 \lipobj \sqrt{n}}, 1\right\}.
\end{equation*}
In turn, for any $d \le 8$, we immediately find that
$(\sqrt{6} - 1) / 4 \sqrt{6} \ge d / (9 \cdot 20)$,
which completes the proof of Theorem~\ref{theorem:linf-bound}(a).

\paragraph{Proof of Theorem~\ref{theorem:linf-bound}(b)}
For the second statement of the theorem, we use the linear losses of
Section~\ref{sec:linear-losses} and apply
Lemmas~\ref{lemma:linear-linf-risk-separation}
and~\ref{lemma:linear-linf-mutual-information} with the choice $\hypercubeset
= \{\pm e_i\}_{i = 1}^d$. In this case, the lower
bound~\eqref{eqn:risk-to-fano} and
Lemma~\ref{lemma:linear-linf-risk-separation}'s
separation guarantee imply that
\begin{equation*}
  \frac{2}{\lipobj \radius \delta} \E_{\statprob, \channelprob}
  \left[\optgap_n(\method, \loss, \optdomain, \statprob)\right] \ge 1
  - \frac{\log 2}{\log(2 d)} - \frac{\information(\channelrv_1,
    \ldots, \channelrv_n; \cuberv)}{\log(2d)}.
\end{equation*}
We may assume that $d \ge 2$ (using the result of part (a) for $d = 1$), and
we have $\log 2 / \log(2d) \le 1/2$, which, after an application of
Lemma~\ref{lemma:linear-linf-mutual-information}, yields
\begin{equation*}
  \frac{2}{\lipobj \radius \delta} \E_{\statprob, \channelprob}
  \left[\optgap_n(\method, \loss, \optdomain, \statprob)\right]
  \ge \half - n \frac{\delta^2 \lipobj^2}{M_\infty^2 \log (2d)}.
\end{equation*}
If we choose $\delta = \min\{M_\infty \sqrt{\log(2d)} / 2 \lipobj \sqrt{n},
1\}$, we see that we have
\begin{equation*}
  \frac{2}{\lipobj \radius \delta} \E_{\statprob, \channelprob}
  \left[\optgap_n(\method, \loss, \optdomain, \statprob)\right]
  \ge \frac{1}{4},
\end{equation*}
which is equivalent in this case to
\begin{equation*}
  \E_{\statprob, \channelprob} \left[\optgap_n(\method, \loss,
    \optdomain, \statprob)\right] \ge \frac{\radius \lipobj \delta}{8}
  = \frac{1}{8}
  \min\left\{ \lipobj \radius, \frac{M_\infty \radius \sqrt{\log
      (2d)}}{2 \sqrt{n}}\right\}.
\end{equation*}


\subsection{Proof of Theorem~\ref{theorem:lone-bound}}
\label{sec:proof-theorem-lone-bound}

The proof of Theorem~\ref{theorem:lone-bound} is quite similar to that
of Theorem~\ref{theorem:linf-bound}, again following our outline
from Section~\ref{sec:testing-reduction-outline}. In this section, however,
we construct a different family of loss functions, necessitating
a new mutual information bound.

\subsubsection{Constructing well-separated losses}
\label{sec:hinge-losses}

We construct families of losses that are useful for analyzing the case of
stochastic subgradients bounded in $\ell_1$-norm.  As was the case with median
losses (recall Section~\ref{sec:linear-losses}), let $\hypercubeset \subset
\{-1, 1\}^d$ be a $d/4$-packing of the hypercube in $\ell_0$-norm; we know
there is such a set with cardinality $|\hypercubeset| \ge \exp(d/8)$.  As our
sampling process for the data, we choose $\statrv$ from among the $2d$
positive and negative standard basis vectors $\pm e_j$, namely
\begin{equation}
  \label{eqn:coord-samples}
  \mbox{Choose~index~} j \in \{1, \ldots, d\} ~
  \mbox{uniformly~at~random, and~set~} \statrv =
  \begin{cases}
    e_j & \mbox{w.p.}~ \frac{1 + \delta \cubecorner_j}{2} \\ -e_j &
    \mbox{w.p.}~ \frac{1 - \delta \cubecorner_j}{2},
  \end{cases}
\end{equation}
where $\delta \in \openleft{0}{1}$ is fixed.  For a fixed $\lipobj > 0$, we
define the hinge loss, common in classification problems,
\begin{equation}
  \label{eqn:svm-loss}
  \loss(\statsample, \optvar) =
  \lipobj \hinge{\radius - \<\statsample, \optvar\>}.
\end{equation}
The combination of
hinge loss~\eqref{eqn:svm-loss} and sampling
strategy~\eqref{eqn:coord-samples} yields the risk functional
\begin{equation*}
  \risk_\cubecorner(\optvar) =
  \frac{\lipobj}{d} \sum_{j=1}^d \frac{1 + \delta \cubecorner_j}{2}
  \hinge{\radius - \<e_j, \optvar\>} +
  \frac{\lipobj}{d} \sum_{j=1}^d \frac{1 - \delta \cubecorner_j}{2}
  \hinge{\radius + \<e_j, \optvar\>}.
\end{equation*}
Assuming that $\optdomain$ contains the $\ell_\infty$ ball of radius
$\radius$, the (unique) minimizer of the risk over $\optdomain$ is
\begin{equation*}
  \optvar^*_\cubecorner \defeq \argmin_{\optvar \in \optdomain}
  \risk_\cubecorner(\optvar) = \radius \cubecorner \in \radius \{-1,
  1\}^d \subset \optdomain.
\end{equation*}
Moreover, this risk has the following properties:
\begin{lemma}
  \label{lemma:hinge-loss-separation}
  For any set $\optdomain \supseteq [-\radius, \radius]^d$, we have:
  \begin{enumerate}[(a)]
  \item For $\statprob$ with support $\supp \statprob \subseteq \{\statsample
    \in \R^d : \lone{x} \le 1\}$, the loss function~\eqref{eqn:svm-loss} is
    $\lipobj$-Lipschitz with respect to the $\ell_\infty$-norm.
  \item If $\cubecorner, \altcubecorner \in \hypercubeset$ with
    $\cubecorner \neq \altcubecorner$, the discrepancy
    $\discrepancy(\risk_\cubecorner, \risk_\altcubecorner) \ge \radius
    \lipobj \delta / 2$.
  \end{enumerate}
\end{lemma}
\begin{proof}
  The first claim is immediate (cf.~\cite{HiriartUrrutyLe96ab}), since
  $\lone{\partial \loss(\statsample, \optvar)} \le \lipobj \lone{\statsample}
  \le \lipobj$. For the second statement of the lemma, as in the proof of
  Lemma~\ref{lemma:linear-linf-risk-separation}
  we verify the separation condition directly by computing the
  minimizers of $\risk_\cubecorner$.
  The minimum of
  \begin{equation*}
    \risk_\cubecorner(\optvar) + \risk_\altcubecorner(\optvar) =
    \frac{\lipobj}{d}\sum_{j=1}^d \left(\hinge{\radius - \<e_j, \optvar\>} +
    \hinge{\radius + \<e_j, \optvar\>}\right)
    + \frac{\lipobj \delta}{d} \sum_{j
      : \cubecorner_j = \altcubecorner_j} \!\! \cubecorner_j
    \left(
    \hinge{\radius - \<e_j, \optvar\>} - \hinge{\radius + \<e_j, \optvar\>}
    \right)
  \end{equation*}
  is attained by any $\optvar \in \R^d$ with $\optvar_j \in [-\radius,
    \radius]$ for $j$ such that $\cubecorner_j \neq \altcubecorner_j$ and
  $\optvar_j = \radius \cubecorner_j$ for $j$ such that $\cubecorner_j =
  \altcubecorner_j$; a minimizer of $\risk_\cubecorner$ is
  $\optvar_\cubecorner^* = \radius \cubecorner$. Thus we have
  \begin{align*}
    \lefteqn{\inf_{\optvar \in \optdomain}
      \left\{\risk_\cubecorner(\optvar) +
      \risk_\altcubecorner(\optvar)\right\} -
      \risk_\cubecorner(\optvar_\cubecorner^*) -
      \risk_\altcubecorner(\optvar_\altcubecorner^*) = \frac{\lipobj}{d}
      \sum_{j = 1}^d 2 \radius - \frac{2\lipobj}{d} \sum_{j : \cubecorner_j
        = \altcubecorner_j} \radius \delta - \lipobj \radius (1 - \delta) -
      \lipobj \radius (1 - \delta)} \\ &
    \qquad\qquad\qquad\qquad\qquad\qquad\qquad = 2\lipobj r - 2\lipobj r
    + 2\lipobj r \delta - \frac{2\lipobj r \delta}{d}
    \left( d - \norm{\cubecorner - \altcubecorner}_0\right)
    = \frac{2\lipobj \radius
      \delta}{d}\norm{\cubecorner - \altcubecorner}_0.
  \end{align*}
  Since $\norm{\cubecorner - \altcubecorner}_0 \ge d/4$ by
  construction, we have $\discrepancy(\risk_\cubecorner,
  \risk_\altcubecorner) \ge \radius \lipobj \delta / 2$, as desired.
\end{proof}

\subsubsection{Bounding the mutual information}

For Theorem~\ref{theorem:lone-bound}, we require a somewhat careful
bound on the mutual information between the subgradients and the
unknown index.  We have the following lemma, whose proof we provide in
Appendix~\ref{appendix:mutual-information-computation-lone}.
\begin{lemma}
  \label{lemma:lone-mutual-information}
  Let $\cuberv$ be drawn uniformly at random from a set $\hypercubeset
  \subset \{-1, 1\}^d$. Define the distribution $\statprob(\cdot \mid
  A)$ on $\statrv$ as in the random sampling
  scheme~\eqref{eqn:coord-samples} and use the
  loss~\eqref{eqn:svm-loss}. Let $\channelrv$ be constructed according
  to the conditional distribution specified by
  Proposition~\ref{proposition:lone-saddle-point}, where $\channelrv
  \in \{\channelval \in \R^d : \lone{\channelval} \le M_1\}$, and
  define $M = M_1 / \lipobj$.  Then
  \begin{equation*}
    \information(\channelrv_1, \ldots, \channelrv_n; \cuberv) \le n
    \delta^2 \channeldiff^2,
  \end{equation*}
where
  \begin{equation*}
    \gamma \defeq \log\left(\frac{2d - 2 + \sqrt{(2d - 2)^2
        + 4(M^2 - 1)}}{2(M - 1)}\right)
    ~~~ \mbox{and} ~~~
    \channeldiff \defeq
    \frac{e^\gamma - e^{-\gamma}}{e^\gamma + e^{-\gamma}
      + 2(d - 1)}.
  \end{equation*}
\end{lemma}

\subsubsection{Applying testing inequalities}

The remainder of the proof is similar to that of
Theorem~\ref{theorem:linf-bound}, except that we apply
Lemma~\ref{lemma:lone-mutual-information} in place of
Lemmas~\ref{lemma:linear-linf-mutual-information}
or~\ref{lemma:linf-mutual-information}. Indeed, following identical steps to
those in the proof of Theorem~\ref{theorem:linf-bound}, we see that with the
specified packing $\hypercubeset \subset \{-1, 1\}^d$ of size $|\hypercubeset|
\ge \exp(d/8)$, we have (recall Eq.~\eqref{eqn:risk-to-fano})
\begin{align*}
  \frac{4}{\radius \lipobj \delta}
  \E_{\statprob, \channelprob}[
    \optgap_n(\method, \loss, \optdomain, \statprob)]
  & \ge 1 - \frac{\log 2}{\log |\hypercubeset|}
  - n \frac{\delta^2 \channeldiff^2}{\log |\hypercubeset|}
  \ge 1 - \frac{6}{d} - 8 n \frac{\delta^2 \channeldiff^2}{d}.
\end{align*}
Consequently, if we choose $\delta = \min\{\sqrt{d} / (8 \channeldiff
\sqrt{n}), 1\}$, then for all $d \ge 9$, we have the lower bound
$\frac{4}{\radius \lipobj \delta} \E_{\statprob, \channelprob}[
  \optgap_n(\method, \loss, \optdomain, \statprob)] \ge \frac{1}{5}$,
or equivalently
\begin{equation*}
  \E_{\statprob, \channelprob}[ \optgap_n(\method, \loss, \optdomain,
    \statprob)] \ge \frac{\radius \lipobj \delta}{20} = \frac{1}{20}
  \min\left\{\radius \lipobj, \frac{\radius \lipobj \sqrt{d}}{
    8 \sqrt{n} \channeldiff}\right\},
\end{equation*}
which completes the proof (the case $d \le 8$ is identical to that in
Theorem~\ref{theorem:linf-bound}).


\subsection{Proof of Theorem~\ref{theorem:linf-bound-diffp}}
\label{sec:proof-theorem-linf-bound-diffp}

We are somewhat more terse in our proof of
Theorem~\ref{theorem:linf-bound-diffp} than the previous two,
though we repeat the same steps to emphasize our technique.

\subsubsection{Constructing well-separated losses}

We begin by choosing the family of loss functions we require: since our
optimization domain $\optdomain = \{\optvar \in \R^d : \lone{\optvar} \le
\radius\}$, we use the linear losses of Section~\ref{sec:linear-losses} with
the sampling scheme~\eqref{eqn:lone-linear-samples} as in
Theorem~\ref{theorem:linf-bound}. Thus, using the packing set $\hypercubeset =
\{\pm e_i\}_{i=1}^d$, we find that
$\risksep(\hypercubeset) = \lipobj \radius \delta$, and consequently
\begin{equation*}
  \frac{2}{\lipobj \radius \delta}
  \E_{\statprob, \channelprob}\left[
    \optgap_n(\method, \loss, \optdomain, \statprob)\right]
  \ge 1 - \frac{\log 2}{\log(2d)}
  -\frac{\information(\channelrv_1, \ldots, \channelrv_n; \cuberv)}{
    \log(2d)}
\end{equation*}
as earlier.

\subsubsection{Bounding the mutual information}
The mutual information bound in this theorem
is somewhat more complicated than the previous bounds, as
the optimal privacy-preserving distribution (recall
Proposition~\ref{proposition:linf-diffp-saddle-point}) is more complex.
We begin by stating a lemma.
\begin{lemma}
  \label{lemma:diffp-linear-information}
  Let $\cuberv$ be drawn uniformly at random from
  $\hypercubeset = \{\pm e_i\}_{i=1}^d$.
  Let $\statrv \mid \cuberv$ be sampled according to the
  distribution~\eqref{eqn:lone-linear-samples}, and let
  $\channelrv \mid \statrv = \statsample$ have support on
  $\{-1, 1\}^d$ and have p.m.f.
  \begin{equation*}
    \channeldensity(\channelval \mid \statsample) \propto
    \begin{cases}
      \exp(\diffp) & \mbox{if~} \channelval^\top \statsample > k \\
      1 & \mbox{if~} \channelval^\top \statsample \le k
    \end{cases}
  \end{equation*}
  for some $k \ge 0$. Define the constants $C_d(k)$ and $\Delta(\delta,
  \diffp, d, k)$ by
  \begin{equation*}
    C_d(k) \defeq \card\left\{\channelval \in \{-1, 1\}^d
    : \<\channelval, \statsample\> > k\right\}
    = \sum_{i = 0}^{\ceil{(d - k)/2} - 1} \binom{d}{i}.
  \end{equation*}
  and
  \begin{equation*}
    \Delta(\delta, \diffp, d, k)
    \defeq \delta \cdot \frac{e^\diffp - 1}{(e^\diffp + 1)
      C_d(k) + 2^d} \binom{d - 1}{\ceil{(d - k)/2} - 1}.
  \end{equation*}
  Then
  \begin{equation*}
    \information(\channelrv; \cuberv) \le
    \Delta(\delta, \diffp, d, k)^2.
  \end{equation*}
\end{lemma}
\noindent
We provide the proof of the lemma in
Appendix~\ref{appendix:mutual-information-computation-diffp}. \\

For any $\diffp \le 5/4$, we have $e^\diffp - 1 \le 2\diffp$, and by
properties of binomial coefficients and Stirling's approximation we have
\begin{equation*}
  \frac{1}{2^d} \binom{d - 1}{\ceil{(d - k)/2} - 1}
  \le \frac{1}{2^d} \binom{d-1}{\ceil{d/2} - 1}
  \le \frac{1}{\sqrt{d}}
\end{equation*}
for any $k$.  For any distribution $\channelprob$ satisfying \olpd\ at a
differential privacy level $\diffp$,
Proposition~\ref{proposition:linf-diffp-saddle-point} implies $\channelprob$
is a convex combination of distributions with p.m.f.s of the form in
Lemma~\ref{lemma:diffp-linear-information}. That is,
we sample with a channel $\channelprob$ whose p.m.f.\ is a convex
combination of p.m.f.s of the form
\begin{equation*}
  \channeldensity_k(\channelval \mid \statsample)
  = \frac{1}{e^\diffp C_d(k) + (2^d - C_d(k))}
  \cdot \begin{cases} \exp(\diffp) & \mbox{if}~ \channelval^\top \statsample
    / M
    > k \\
    1 & \mbox{if~} \channelval^\top \statsample / M \le k
  \end{cases}
  ~~ \mbox{for~} \channelval \in \{-M, M\}^d,
\end{equation*}
so $\channelprob(\channelrv = \channelval \mid \statsample)
= \sum_k \beta_k \channeldensity_k(\channelval \mid \statsample)$
for some $\beta_k \ge 0$ with $\sum_k \beta_k = 1$.
Applying the convexity of mutual
information---taking a convex combination of channel distributions
$\channelprob$ can only reduce mutual information---and
Lemma~\ref{lemma:diffp-linear-information}, we thus obtain
\begin{align}
  \lefteqn{\information(\channelrv_1, \ldots, \channelrv_n; \cuberv)
    \le n \max_{k \ge 0} \Delta(\delta, \diffp, d, k)^2}
  \label{eqn:opt-diffp-information} \\
  & \qquad\qquad ~ \le n \delta^2 (e^\diffp - 1)^2 \max_k \left(
  \frac{1}{(e^\diffp + 1) C_d(k) + 2^d} \binom{d - 1}{
    \ceil{(d - k)/2} - 1}\right)^2
  \le 4n \delta^2 \diffp^2 \cdot \frac{1}{d}.
  \nonumber
\end{align}

\subsubsection{Applying testing inequalities}
As a consequence of the display~\eqref{eqn:opt-diffp-information}, we have the
lower bound
\begin{equation*}
  \frac{2}{\lipobj \radius \delta}
  \E_{\statprob, \channelprob}\left[
    \optgap_n(\method, \loss, \optdomain, \statprob)\right]
  \ge \half - \max_k \frac{n\Delta(\delta, \diffp, d, k)^2}{\log(2d)}
  \ge \half - \frac{4 n \delta^2 \diffp^2}{d \log(2d)}.
\end{equation*}
By choosing $\delta = \min\{\sqrt{d \log(2d)}/ 4\diffp \sqrt{n}, 1\}$, we find
that
\begin{equation*}
  \frac{2}{\lipobj \radius \delta}
  \E_{\statprob, \channelprob}[\optgap_n(
    \method, \loss, \optdomain, \statprob)]
  \ge \frac{1}{4},
\end{equation*}
which is equivalent to the bound given in the theorem.



\subsection{Proof of Theorem~\ref{theorem:first-class}}
\label{sec:proof-first-class}

\newcommand{\channel}{\channelprob}
\newcommand{\marginprob}{M}
\newcommand{\packset}{\hypercubeset}
\newcommand{\packrv}{\cuberv}
\newcommand{\packval}{\cubecorner}
\newcommand{\altpackval}{\altcubecorner}
\newcommand{\optdens}{\gamma} 
\newcommand{\linfset}{\mc{B}_1} 
\newcommand{\channeldomain}{\mc{\channelrv}}
\newcommand{\meanstatprob}{\overline{\statprob}}

In our proofs of Theorems~\ref{theorem:first-class}
and~\ref{theorem:second-class}, we exploit some of our own recent
results~\cite{DuchiJoWa13_parametric} on the contractive properties of
mutual information and KL-divergence under local differential
privacy. A few definitions are required in order to state these
results.  As usual, we have an indexed set of probability measures
$\{\statprob_\packval\}_{\packval \in \packset}$, and we let
$\channelprob^n(\cdot \mid \statsample_1, \ldots, \statsample_n)$
denote the joint probability of the $n$ released private random
variables $\channelrv_1, \ldots, \channelrv_n$. For each $\packval \in
\packset$, we then define the \emph{marginal} distribution
\begin{equation}
  \marginprob_\packval^n(A) \defeq
  \int \channelprob^n(A \mid \statsample_1, \ldots, \statsample_n)
  d \statprob^n_\packval(\statsample_1, \ldots, \statsample_n)
  ~~~ \mbox{for}~ A \in \sigma(\channeldomain^n).
  \label{eqn:marginal-channel}
\end{equation}
\citet{DuchiJoWa13_parametric} establish two results that provide
bounds on $\tvnorm{\marginprob_\packval^n -
  \marginprob_{\altpackval}^n}$ and $\information(\channelrv_1,
\ldots, \channelrv_n; \packrv)$ as a function of the amount of privacy
provided and the distances between the underlying distributions
$\statprob_\packval$.


The bounds apply to any channel distribution $\channelprob$ that is
$\diffp$-locally differentially private (for the first result) and to any
non-interactive $\diffp$-locally differentially private channel (for the
second result; recall the
definition~\eqref{eqn:local-differential-privacy}). Let $\statprob_\packval$
be the distribution of $\statrv$ conditional on the random packing element
$\packrv = \packval$, and let $\marginprob_\packval^n$ be the marginal
distribution~\eqref{eqn:marginal-channel} induced by passing $\statrv_i$
through $\channelprob$.  Define the mixture distribution \mbox{$\meanstatprob
  = \frac{1}{|\packset|} \sum_{\packval \in \packset} \statprob_{\packval}$,}
and let $\sigma(\statdomain)$ denote the $\sigma$-field on $\statdomain$ over
which $\statprob_\packval$ are defined.  With this notation we can state the
following proposition, which summarizes the results we need
from~\citet[Theorems 1--2, Corollaries 1--2]{DuchiJoWa13_parametric}:
\begin{proposition}[Information bounds]
  \label{proposition:information-bounds}
  Let the conditions of the previous paragraph hold and assume that
  $\channelprob$ is $\diffp$-locally differentially private.
  \begin{itemize}
  \item[(a)] For all $\diffp \ge 0$,
    \begin{equation}
      \label{eqn:one-d-kl}
      \dkl{\marginprob_\packval^n}{\marginprob_{\altpackval}^n}
      \le 4 n
      (e^\diffp - 1)^2 \tvnorm{\statprob_{\packval} -
        \statprob_{\altpackval}}^2.
    \end{equation}
  \item[(b)] If $\channelprob$ is non-interactive, then for all $\diffp \ge 0$,
    \begin{equation}
      \label{eqn:super-fano}
      \information(\channelrv_1, \ldots, \channelrv_n; \packrv)
      \le e^\diffp n \left(e^\diffp - e^{-\diffp}\right)^2
      \sup_{S \in \sigma(\statdomain)}\frac{1}{|\packset|}
      \sum_{\packval \in \packset}
      \left(\statprob_\packval(S) - \meanstatprob(S)\right)^2.
    \end{equation}
  \end{itemize}
\end{proposition}
\noindent
With Proposition~\ref{proposition:information-bounds} in place, we
proceed with the proof of Theorem~\ref{theorem:first-class}. To
establish the lower bound, we follow the outline of
Section~\ref{sec:testing-reduction-outline}.  Establishing the upper
bound requires a few additional steps. We defer the formal proofs of
achievability to Appendix~\ref{appendix:sgd-achievability} (see
\ref{appendix:sgd-linf-achievability}), as the proofs are similar to
Corollaries~\ref{corollary:linf-privacy}--\ref{corollary:linf-privacy-diffp}.

\subsubsection{Constructing well-separated losses}
Our lower bound uses an identical construction to that in
Section~\ref{sec:linear-losses}. We let the loss function $\loss(\statsample,
\optvar) = \lipobj \<\statsample, \optvar\>$, and we use the
distribution~\eqref{eqn:lone-linear-samples} on $\statsample$; that is, we
have $\packset = \{\pm e_j\}_{j=1}^d$ and for $\delta \in [0, 1/2]$, we sample
vectors from $\statdomain = \{-1, 1\}^d$ with probability
$\statprob_\packval(\statrv = \statsample) = (1 + \delta \packval^\top
\statsample) / 2^d$. We then have $\risksep(\packset) = \lipobj \radius
\delta$ and $|\packset| = 2d$ (recall
Lemma~\ref{lemma:linear-linf-risk-separation}).

\subsubsection{Bounding the mutual information}

With the construction of the (nearly) uniform sampling scheme, we have the
following lemma~\cite[Lemma 7]{DuchiJoWa13_parametric}.
\begin{lemma}
  \label{lemma:l1-information-bound}
  Under the conditions of the previous paragraph, let $\delta \le 1$ and
  $\packrv$ be sampled uniformly from $\{\pm e_j\}_{j=1}^d$.  For any
  non-interactive $\diffp$-differentially private channel $\channel$,
  \begin{equation*}
    \information(\channelrv_1, \ldots, \channelrv_n; \packrv) \le
    n \frac{e^\diffp}{4d} \left(e^{\diffp} - e^{-\diffp}\right)^2
    \delta^2.
  \end{equation*}
\end{lemma}

\subsubsection{Applying testing inequalities}

Using Lemma~\ref{lemma:l1-information-bound}, we can give an almost immediate
proof of the lower bound in Theorem~\ref{theorem:first-class}. Indeed,
using Fano's inequality~\eqref{eqn:fano},
Lemmas~\ref{lemma:expectation-lower-bound}
and~\ref{lemma:l1-information-bound}, and the separation $\risksep(\packset) =
\lipobj \radius \delta$ from Lemma~\ref{lemma:linear-linf-risk-separation}(b),
we obtain
\begin{equation*}
  \optgap_n^*(\lossset, \optdomain, \diffp) \ge \frac{\lipobj \radius
    \delta}{2} \left(1 - \frac{n e^\diffp (e^\diffp - e^{-\diffp})^2
    \delta^2 / 4 d + \log 2}{\log(2d)}\right).
\end{equation*}
So long as $d \ge 2$, setting
\begin{equation*}
  \delta = \min\left\{\frac{\sqrt{d \log(2d)}}{\sqrt{e^\diffp n} (e^\diffp -
    e^{-\diffp})}, 1 \right\}
\end{equation*}
and noting that $e^\diffp = \order(1)$ and $e^\diffp - e^{-\diffp} \le c
\diffp$ for a universal constant $c$ if $\diffp = \order(1)$ completes the
proof.

When $d = 1$, an argument via Le Cam's method~\eqref{eqn:le-cam} yields an
identical result.  Given Proposition~\ref{proposition:information-bounds}, the
argument is quite similar to that used in the proof of
Theorem~\ref{theorem:linf-bound}.  We use the packing set $\packset = \{\pm
1\}$ and conditional on $\packrv = \packval$, set $\statrv = 1$ with
probability $(1 + \packval \delta) / 2$ and $\statrv = -1$ with probability
$(1 - \packval \delta) / 2$, which yields separation $\risksep(\packset) =
\lipobj \radius \delta$ by Lemma~\ref{lemma:linear-linf-risk-separation}. We
also have the marginal contraction
\begin{equation*}
  \tvnorm{\marginprob_1^n - \marginprob_{-1}^n}^2 \le \half
  \dkl{\marginprob_1^n}{\marginprob_{-1}^n} \le 2 (e^\diffp - 1)^2 n
  \tvnorm{\statprob_1 - \statprob_{-1}}^2
\end{equation*}
by Pinsker's inequality and
Proposition~\ref{proposition:information-bounds}.  By
construction, the total variation $\tvnorm{\statprob_1 -
  \statprob_{-1}} = \delta$, whence we find that
$\tvnorm{\marginprob_1^n - \marginprob_{-1}^n}^2 \le 2(e^\diffp -
1)^2 n \delta^2$.  Applying Le Cam's method~\eqref{eqn:le-cam} and
Lemma~\ref{lemma:expectation-lower-bound}, we obtain
\begin{equation*}
  \optgap_n^*(\lossset, \optdomain, \diffp) \ge \frac{\lipobj \radius
    \delta}{2} \left(\half - \frac{\sqrt{n} (e^\diffp - 1)
    \delta}{\sqrt{2}}\right).
\end{equation*}
Take $\delta = \min\{\left(2 \sqrt{2} \sqrt{n} (e^\diffp - 1)\right)^{-1},
1\}$ to complete the proof in this case.


\subsection{Proof of Theorem~\ref{theorem:second-class}}
\label{sec:proof-second-class}

The proof of this theorem follows the outline established in
Section~\ref{sec:testing-reduction-outline}, as did the previous results. We
defer the attainability results to
Appendix~\ref{appendix:sgd-ltwo-achievability}.

\subsubsection{Constructing well-separated losses}
Before constructing the well-separated loss functions, we
exhibit the set $\packset$ that underlies our sampling distributions.
The following lemma exhibits the existence of a special packing of
the Boolean hypercube~\cite[Lemma 5]{DuchiJoWa13_parametric}:
\begin{lemma}
  \label{lemma:hypercube-packing}
  There exists a packing $\packset$ of the $d$-dimensional hypercube
  $\{-1, 1\}^d$ with $\lone{\packval - \altpackval} \ge d/2$ for each
  $\packval, \altpackval \in \packset$ with $\packval \neq
  \altpackval$ such that the cardinality of $\packset$ is at least
  $\ceil{\exp(d/16)}$ and
  \begin{equation*}
    \frac{1}{|\packset|} \sum_{\packval \in \packset} \packval
    \packval^\top \preceq 25 I_{d \times d}.
  \end{equation*}
\end{lemma}
\noindent
With this packing $\packset$, we 
as usual let $\packrv \in \packset$, and conditional on
$\packrv = \packval$, we sample
$\statrv \in \{\pm e_j\}_{j=1}^d$ according the
sampling scheme~\eqref{eqn:coord-samples}.

\paragraph{Linear losses (for the bound~\eqref{eqn:second-linear-rate})}
We first consider the case that
the loss functions are linear functionals that are $\lipobj$-Lipschitz with
respect to the $\ell_{p'}$-norm for some $p' \ge 2$, and we optimize over
$\ell_q$ balls $\Ball_q(\radius_q)$.  In this case, we define the loss
$\loss(\statsample, \optvar) = \lipobj \<\statsample, \optvar\>$, and we note
that since $\lone{\statsample} \le 1$ for $\statsample \in \statdomain$,
the function $\loss(\statsample, \cdot)$ is $\lipobj$-Lipschitz with
respect to the $\ell_\infty$-norm. With the sampling
scheme~\eqref{eqn:coord-samples}, $\risk_\packval(\optvar) = \lipobj
\delta \<\packval,\optvar\> / d$.
Since $\inf_{\optvar \in \ball_q(\radius_q)}
\<\cubecorner, \optvar\> = -\norm{\cubecorner}_p$
when $p = (1 - 1/q)^{-1}$ is the conjugate of $q$, we have
the pairwise separation
\begin{equation*}
  \discrepancy(\risk_\packval, \risk_{\altpackval})
  = -\frac{\radius_q \lipobj \delta}{d}
  \left(\norms{\packval + \altpackval}_p - \norm{\packval}_p
  - \norms{\altpackval}_p\right).
\end{equation*}
With the packing set $\packset$ exhibited by
Lemma~\ref{lemma:hypercube-packing}, we have
\begin{equation*}
  \norms{\packval + \altpackval}_p^p
  = \sum_{j=1}^d (\packval_j + \altpackval_j)^p
  = \sum_{j : \packval_j = \altpackval_j} 2^p
  \le \frac{3d}{4} 2^p,
\end{equation*}
and $\norm{\packval}_p = d^{1/p}$ for each $\packval \in \packset$. This
implies the following lower bound on the discrepancy:
\begin{equation}
  \label{eqn:lp-lq-separation}
  \risksep(\packset) \ge
  \max_{\packval \neq \altpackval}
  \discrepancy(\risk_\packval, \risk_{\altpackval})
  \ge \frac{\radius \lipobj \delta}{d}
  \left(2 d^{1/p} (1 - (3/4)^{1/p})\right)
  \ge \frac{3}{5} \radius \lipobj \delta d^{\frac{1}{p} - 1}.
\end{equation}

\paragraph{General loss functions
  (for the bound~\eqref{eqn:second-class-lower})} For the general lower bound
of the theorem, we use the hinge loss $\loss(\statsample, \optvar) = \lipobj
\hinge{\radius - \<\statsample, \optvar\>}$ as our loss function. In this
case, as in Theorem~\ref{theorem:lone-bound} (recall
Lemma~\ref{lemma:hinge-loss-separation}), our sampling strategy yields that
the loss $\loss(\statsample, \optvar)$ is an $(\lipobj, 1)$ loss, since
$\lone{\statsample} \le 1$, and we have the discrepancy bound
$\risksep(\packset) \ge \frac{\radius \lipobj \delta}{2}$, since the
separation depends only on the distance $\lone{\packval - \altpackval}$, which
Lemma~\ref{lemma:hypercube-packing} lower bounds.

\subsubsection{Bounding the mutual information}

Using Lemma~\ref{lemma:hypercube-packing}, we may bound the mutual
information between samples $\channelrv$ from a particular
distribution and a random sample $\packrv$ from a set $\packset$ of
the form in the lemma. Indeed, let $\packset$ be a packing of the
$d$-dimensional hypercube specified in
Lemma~\ref{lemma:hypercube-packing}. Conditional on $\packrv =
\packval \in \{-1, 1\}^d$, we sample the random vector $\statrv \in
\{\pm e_j\}_{j=1}^d$ according to the sampling
scheme~\eqref{eqn:coord-samples}.  Then we have the following
lemma~\cite[Lemma 6]{DuchiJoWa13_parametric}, which applies as long as
the channel $\channelprob$ is non-interactive and $\diffp$-locally
differentially private.
\begin{lemma}
  \label{lemma:linf-info-bound}
  Let $\channelrv_i$ be $\diffp$-locally differentially
  private for $\statrv_i$ and the conditions of the
  previous paragraph hold. Then
  \begin{equation*}
    \information(\channelrv_1, \ldots, \channelrv_n; \packrv) \le n
    \frac{25 e^\diffp}{16} \frac{\delta^2}{d} (e^\diffp -
    e^{-\diffp})^2.
  \end{equation*}
\end{lemma}

\subsubsection{Applying testing inequalities}

Our last step is to apply the usual testing inequalities.
We first prove the lower bound in inequality~\eqref{eqn:second-linear-rate}.
Let $\lossset_{\rm lin} = \linearclasstwo[\Ball_q(\radius_q)]$.
Then by applying Lemma~\ref{lemma:linf-info-bound} and Fano's
inequality~\eqref{eqn:fano} to
Lemma~\ref{lemma:expectation-lower-bound}---using the
separation~\eqref{eqn:lp-lq-separation}---we obtain
\begin{equation*}
  \optgap_n^*(\lossset_{\rm lin}, \optdomain, \diffp)
  \ge \frac{3 \radius_q \lipobj \delta d^{\frac{1}{p} - 1}}{10}
  \left(1 - \frac{25 n e^\diffp \delta^2 (e^\diffp - e^{-\diffp})^2 / 16d
    + \log 2}{d / 16}\right).
\end{equation*}
So long as $d \ge 16$, we have $16 \log 2 / d \le \log 2 < 7/10$. Thus
choosing
\begin{equation*}
  \delta =
  \min\left\{\frac{d}{10 \sqrt{n C_\diffp} (e^\diffp - e^{-\diffp})},
  1 \right\}
\end{equation*}
and noting that $e^\diffp - e^{-\diffp} = \order(\diffp)$ and $e^\diffp
= \order(1)$ for $\diffp = \order(1)$ we obtain
\begin{equation*}
  \optgap_n^*(\lossset, \optdomain, \diffp)
  \ge \frac{3 \radius_q \lipobj \delta d^{\frac{1}{p} - 1}}{10}
  \cdot \frac{1}{20}
  \ge c
  \min\left\{\frac{\radius_q \lipobj d^{\frac{1}{p}}}{\sqrt{n \diffp^2}},
  \radius_q \lipobj d^{\frac{1}{p} - 1}\right\}
  = c \radius_q \lipobj \min\left\{
  \frac{d^{1 - \frac{1}{q}}}{\sqrt{n \diffp^2}},
  d^{\frac{1}{p} - 1}\right\}.
\end{equation*}
for a universal constant $c$.
As in the proof of Theorem~\ref{theorem:first-class}, when $d < 16$, we apply
an essentially similar argument but with Le Cam's method~\eqref{eqn:le-cam},
which gives the desired result. (The proof of the case $d = 1$ from
Theorem~\ref{theorem:first-class} applies here as well.)  Finally, we remark
that we may repeat a completely identical proof as that above replacing $d$
with any $k \le d$, which \emph{mutatis mutandis} implies the lower bound
\begin{equation*}
  \optgap_n^*(\lossset, \optdomain, \diffp)
  \ge c \lipobj \radius_q \, \max_{k \in [d]}
  \min\left\{\frac{k^{1 - \frac{1}{q}}}{\sqrt{n \diffp^2}},
  k^{\frac{1}{p} - 1}\right\}
  \ge c \lipobj \radius_q
  \min\left\{\frac{d^{1 - \frac{1}{q}}}{\sqrt{n \diffp^2}},
  \frac{(n \diffp^2)^{\frac{1}{2p}}}{\sqrt{n \diffp^2}},
  1\right\},
\end{equation*}
where for the rightmost bound we have taken $k = \max\{1, \min\{\sqrt{n
  \diffp^2}, d\}\}$.  Noting that $1/p = 1 - 1/q$ completes the proof of the
lower bound~\eqref{eqn:second-linear-rate}.

To prove inequality~\eqref{eqn:second-class-lower},
we apply the same reasoning as in the proof of the
inequality~\eqref{eqn:second-linear-rate}. Use
Lemma~\ref{lemma:linf-info-bound} and Fano's inequality~\eqref{eqn:fano} (via
Lemma~\ref{lemma:expectation-lower-bound} and the separation from
Lemma~\ref{lemma:hinge-loss-separation}) to obtain
\begin{equation*}
  \optgap_n^*(\lossset, \optdomain, \diffp) \ge \frac{\radius \lipobj
    \delta}{4} \left(1 - \frac{25 n e^\diffp (e^{\diffp} -
    e^{-\diffp})^2 \delta^2 / 16d + \log 2}{d / 16}\right).
\end{equation*}
Then choosing $\delta \asymp d / (\sqrt{n}(e^\diffp - e^{-\diffp}))$ as in the
proof of inequality~\eqref{eqn:second-linear-rate} gives the desired result.


\subsection{Proof of Corollary~\ref{corollary:linf-privacy}}
\label{sec:corollary-rate-linf}

Since $\optdomain \subseteq \{\optvar \in \R^d : \lone{\optvar} \le
\radius\}$, the bound~\eqref{eqn:mirror-descent-bound}
guarantees that mirror descent obtains convergence rate
$\order(M_\infty \radius \sqrt{\log (2d)} / \sqrt{n})$.  This matches the
second statement of Theorem~\ref{theorem:linf-bound}.  Now fix our desired
amount of mutual information $\information^*$. From the remarks following
Proposition~\ref{proposition:linf-saddle-point}, if we must guarantee that
$\information^* \ge \sup_\statprob \information(\statprob, \channelprob)$
for any distribution $\statprob$ and loss function $\loss$ whose gradients
are bounded in $\ell_\infty$-norm by $\lipobj$, we \emph{must} (because of
the uniqueness of the optimal privacy distribution $\channelprob$) have
\begin{equation}
  \label{eqn:information-bijection}
  \information^* \asymp \frac{d \lipobj^2}{M_\infty^2}.
\end{equation}
Up to higher order terms, to guarantee a level of privacy with mutual
information $\information^*$, we must allow gradient noise up to a level
$M_\infty = \lipobj\sqrt{d / \information^*}$.
The equality~\eqref{eqn:information-bijection} establishes that for a
given level of allowed mutual information $\information^*$, if
\olp\ holds, then we must have $M_\infty \asymp \lipobj \sqrt{d} /
\sqrt{\information^*}$.  That is, we have a bijection between
$\information^*$ and $M_\infty$ whenever \olp\ holds, so substituting
$M_\infty = \lipobj \sqrt{d} / \sqrt{\information^*}$ into our upper
and lower bounds yields the claim.


\subsection{Proof of Corollary~\ref{corollary:lone-privacy}}
\label{sec:corollary-rate-lone}


According to the conditions of \olp, if we must guarantee that
$\information^* \ge \sup_\statprob \information(\statprob,
\channelprob)$ for any loss function $\loss$ whose gradients are
bounded in $\ell_1$-norm by $\lipobj$, we must have
\begin{equation*}
  \information^* \asymp \frac{d \lipobj^2}{2 M_1^2},
\end{equation*}
using Corollary~\ref{corollary:lone-asymptotic-expansion} after the statement of
Proposition~\ref{proposition:lone-saddle-point}. Rewriting this, we see that
we must have $M_1 = \lipobj \sqrt{d / 2\information^*}$ (to higher order
terms) to be able to guarantee an amount of privacy $\information^*$.  As in
the $\ell_\infty$ case, we have a bijection between the multiplier $M_1$ and
the amount of information $\information^*$ and can apply similar techniques.
Now recall the convergence guarantee~\eqref{eqn:sgd-bound} provided by
stochastic gradient descent. Since the $\ell_\infty$-ball of
radius $\radius$ is contained in the $\ell_2$-ball of radius $\radius_2
= \radius \sqrt{d}$, and $\lone{g} \le \ltwo{g}$ for all $g \in \R^d$,
stochastic gradient descent guarantees that
$\optgap_n^*(\lossset, \optdomain) \le C M_1 \radius \sqrt{d} / \sqrt{n}$.
Applying the lower bound provided by Theorem~\ref{theorem:lone-bound}
and substituting for $M_1$ completes the proof.

\subsection{Proof of Corollary~\ref{corollary:linf-privacy-diffp}}
\label{sec:corollary-rate-linf-diffp}

Without loss of generality (by scaling), we assume that $\lipobj = 1$.
Now we consider Proposition~\ref{proposition:linf-diffp-saddle-point},
which characterizes the distributions satisfying \olpd. We use the
proposition to find an upper bound on $M_\infty$ in terms of the
differential privacy level $\diffp$, which in turn allows us to apply
the bound from mirror descent~\eqref{eqn:mirror-descent-bound}.
Instead of directly using
Proposition~\ref{proposition:linf-diffp-saddle-point}, it is simpler
to use the linear program~\eqref{eqn:single-diffp-lp} in its proof,
and note that finding a lower bound on $t$ (in the LP) as a function
of $\diffp$ provides an upper bound on $M_\infty$ since $M_\infty =
1/t$.  Now, in the linear program~\eqref{eqn:single-diffp-lp}, we
choose the values for $q(z)$ specified by
Lemma~\ref{lemma:two-level-diffp-solution}.  Let $q_+$ and $q_-$
denote the larger and smaller probabilities, respectively. Fix an $x
\in \{-1, 1\}^d$, and let $z$ range over $\{-1, 1\}^d$. With those
choices, we note that for $d$ odd,
\begin{align*}
  \sum_{z : \<z, x\> > 0} z
  & = \sum_{z : \<z, x\> = 1} z
  + \sum_{z : \<z, x\> = 3} z
  + \ldots + \sum_{z : \<z, x\> = d} z \\
  & = \left[\binom{d - 1}{\frac{d - 1}{2}}
    - \binom{d - 1}{\frac{d + 1}{2}}\right] x
  + \left[\binom{d - 1}{\frac{d + 1}{2}}
    - \binom{d - 1}{\frac{d + 3}{2}}\right] x
  + \ldots 
  = \binom{d - 1}{\frac{d - 1}{2}} x.
\end{align*}
For $d$ even, a similar calculation yields
$\sum_{z : \<z, x\> > 0} z = \binom{d-1}{d/2} x$.
As a consequence, we find that
\begin{equation*}
  \sum_z z q(z \mid x)
  = q_+ \!\!\!\! \sum_{z : \<z, x\> > 0} z
  + q_- \!\!\!\! \sum_{z : \<z, x\> \le 0} z
  = x (q_+ - q_-) \cdot
  \begin{cases} \binom{d - 1}{\frac{d - 1}{2}} & d~\mbox{odd} \\
    \binom{d - 1}{d / 2} & d~\mbox{even}. \end{cases}
\end{equation*}
Focusing on the odd case for simplicity---identical bounds
hold in the even case---we have
for a universal constant $c > 0$ that
\begin{equation*}
  (q_+ - q_-) \binom{d - 1}{\frac{d - 1}{2}} = \frac{e^\diffp -
    1}{2^{d-1} (e^\diffp + 1)} \binom{d - 1}{\frac{d - 1}{2}} \ge c
  \frac{e^\diffp - 1}{e^\diffp + 1} \frac{1}{\sqrt{d}} \ge c
  \frac{\diffp}{\sqrt{d}},
\end{equation*}
the first inequality following from Stirling's approximation and the second
from convexity of the function $\diffp \mapsto e^\diffp$.  In particular, we
see that the minimizing value $t$ in the linear
program~\eqref{eqn:single-diffp-lp} will satisfy $t \ge c \diffp / \sqrt{d}$,
which in turn yields $M_\infty = 1 / t \le \sqrt{d} / (c \diffp)$.  Noting
that the lower bound in the corollary is given by
Theorem~\ref{theorem:linf-bound-diffp},
applying the convergence guarantee~\eqref{eqn:mirror-descent-bound} of
mirror descent based
on $M_\infty$ completes the proof.




%% file: open-issues.tex
\section{Discussion}
\label{sec:discussion}

We have studied methods for protecting privacy in general statistical
risk minimization problems, and have described general techniques for
obtaining sharp tradeoffs between privacy protection and estimation
rates. The latter are a natural measure of utility for statistical
problems.

We believe that there are a number of interesting open issues and
areas for future work. First, we studied procedures that access each
datum only once, and through a perturbed view $\channelrv_i$ of the
subgradient $\partial \loss(\statrv_i, \optvar)$, which is natural in
the context of convex risk minimization.  A natural question is
whether there are restrictions of the class of loss functions so that
a transformed version $(\channelrv_1, \ldots, \channelrv_n)$ of the
data are sufficient for inference. For instance, other
researchers~\cite{ZhouLaWa09,ZhouLiWa09} have studied applications in
which a data matrix $X = [\statrv_1 ~ \cdots ~ \statrv_n]^\top \in
\R^{n \times d}$ is pre-multiplied by a normal matrix $\Phi \in \R^{m
  \times n}$, where $m \ll n$, and statistical inference is performed
using $\Phi X$. For problems such as linear regression and PCA, the
resulting estimators enjoy good statistical properties. This
transformation, however, cannot be computed without the entire dataset
at one's disposal. Nonparametric data releases, such as those studied
by~\citet{HallRiWa11}, could provide insights here, though again,
current approaches require the data to be aggregated by a trusted
curator before release.

Our constraints on the privacy-inducing channel distribution
$\channelprob$ require that its support lie in some compact set. We
find this restriction useful, but perhaps it possible to achieve
faster estimation rates if all we require are moment conditions, for
example, $\E_{\channelprob}[\norm{\channelrv - \statrv}_p^2 \mid
  \statrv] \le M^2$. A better understanding of general
privacy-preserving channels $\channelprob$ for alternative constraints
to those we have proposed is also desirable.  Moreover, one might
consider attempting only to guarantee that $\phi(\statrv)$ is private,
where $\phi$ is some (known) function. For example, members of a
dataset may not care if their genders are known, but more personal
features of $\statrv$ may be more sensitive.

These questions do not appear to have easy answers, especially when we wish to
allow each provider of a single datum to be able to guarantee his or her own
privacy. Nevertheless, we hope that our view of privacy and
the techniques we have developed herein prove fruitful,
and we hope to investigate some of the above issues in future work.



%% file: unbiased.tex
\section{Unbiasedness}
\label{appendix:biased-subgradient}

In this appendix, we show that if an optimization procedure receives
biased subgradients it is possible to be arbitrarily wrong. We do so
by constructing a simple problem instance.  Fix a bias $b > 0$ and
consider the following one-dimensional problem:
\begin{equation*}
  \minimize ~ f(\optvar) \defeq \frac{b\optvar}{2}
  ~~ \subjectto ~~ \optvar \in [-c, c].
\end{equation*}
If a gradient oracle returns biased gradients of the form $-b/2$ at
each point $\optvar \in [-c, c]$, it is impossible to distinguish the
objective from $-b \optvar/2$. The minimizer of this objective is
\mbox{$\optvar_{\rm bias} = \sign(b) c$.} The true optimal point is
$\optvar^* = -\sign(b) c$, yielding the worst possible error
\begin{equation*}
  f(\optvar_{\rm bias}) - f(\optvar^*)
  = \sup_{\optvar \in [-c, c]} f(\optvar)
  - \inf_{\optvar \in [-c, c]} f(\optvar).
\end{equation*}
We can show this more formally using an information theoretic
derivation similar to that in
Section~\ref{sec:optimal-rate-proofs}. Omitting details, the argument
is as follows. In the notation of
Section~\ref{sec:optimal-rate-proofs}, if a bias is chosen
independently of the parameters $\cubecorner \in \hypercubeset$ of the
risk $\risk_\cubecorner$, then there is a \emph{bounded} amount of
mutual information that can be communicated to any optimization
procedure.  Consequently, Fano's inequality~\eqref{eqn:fano}
guarantees that the estimation accuracy of any procedure must be
bounded away from zero.



%% file: mutual-info-calculations.tex
\section{Calculation of the Mutual Information for Sampling Strategies}
\label{appendix:mutual-information-computation}

This appendix is devoted to the proofs of our bounds on mutual information:
Lemma~\ref{lemma:linear-linf-mutual-information},
Lemma~\ref{lemma:linf-mutual-information},
Lemma~\ref{lemma:lone-mutual-information},
and Lemma~\ref{lemma:diffp-linear-information}.
Before proving the lemmas, we make an observation that allows us to tensorize
the mutual information, making our arguments simpler (we need only compute the
single observation information $\information(\channelrv_1; \cuberv)$).  For
each of Lemmas~\ref{lemma:linear-linf-mutual-information},
\ref{lemma:linf-mutual-information}, \ref{lemma:lone-mutual-information}, and
\ref{lemma:diffp-linear-information}, recall that $\channelrv_1, \ldots,
\channelrv_n$ are constructed based on an evaluation of the subgradient set
$\partial_\optvar \loss(\statrv_i, \optvar)$, where $\statrv_i$ are
independent samples according to a distribution $\statprob(\cdot \mid
\cuberv)$. Then the samples $\channelrv_i$ are conditionally independent of
$\cuberv$ given $\statrv_i$ and the parameters $\optvar$, since $\channelrv$
is a random function of $\partial \loss(\statrv_i, \optvar)$.  Our goal is to
upper bound the mutual information between the sequence $\channelrv_1, \ldots,
\channelrv_n$ of observed (stochastic) gradients and the random element
$\cuberv \in \hypercubeset$.


By the general definition of mutual information~\cite[Chapter 5]{Gray90}, it
is no loss of generality to assume (temporarily) that the random variable
$\channelrv_i$ are supported on finite sets.  Thus (using the chain rule for
mutual information~\cite[Chapter 5]{CoverTh06,Gray90}) we have the
decomposition
\begin{align*}
  \information(\channelrv_1, \ldots, \channelrv_n; \cuberv)
  & = \sum_{i=1}^n \left[
    H(\channelrv_i \mid \channelrv_1, \ldots, \channelrv_{i-1})
    - H(\channelrv_i \mid \cuberv, \channelrv_1, \ldots, \channelrv_{i-1})
    \right].
\end{align*}
Let $\optvar_i$ denote the point at which the $i$th gradient is computed.
Then by inspection, we must have
$\optvar_i \in \sigma(\channelrv_1, \ldots, \channelrv_{i-1})$. Since
$\channelrv_i$ is conditionally independent of $\channelrv_1, \ldots,
\channelrv_{i-1}$ given $\cuberv$ and $\optvar_i$ and conditioning
decreases entropy, we have
\begin{align*}
  H(\channelrv_i \mid \channelrv_1, \ldots, \channelrv_{i-1}) -
  H(\channelrv_i \mid \cuberv, \channelrv_1, \ldots, \channelrv_{i-1}) &
  = H(\channelrv_i \mid \channelrv_1, \ldots, \channelrv_{i-1}) -
  H(\channelrv_i \mid \cuberv, \optvar_i) \\
  & \leq H(\channelrv_i \mid \optvar_i) - H(\channelrv_i \mid \cuberv,
  \optvar_i) \\
  & = \information(\channelrv_i; \cuberv \mid \optvar_i).
\end{align*}
In particular, letting $F_i$ denote the distribution of $\optvar_i$, we have
\begin{align}
  \information(\channelrv_1, \ldots, \channelrv_n; \cuberv) & \leq
  \sum_{i=1}^n \int_{\optdomain} \information(\channelrv_i; \cuberv
  \mid \optvar) dF_i(\optvar)
  \label{eqn:observed-cube-information}
  \leq \sum_{i=1}^n \sup_{\optvar \in \optdomain}
  \information(\channelrv_i; \cuberv \mid \optvar).
\end{align}
The representation~\eqref{eqn:observed-cube-information} is the key
to our calculations in this appendix. \\

In addition, the proofs of Lemmas~\ref{lemma:linf-mutual-information}
and~\ref{lemma:lone-mutual-information} require a minor lemma, which we
present here before giving the proofs proper.

\begin{lemma}
  \label{lemma:entropy-minimization}
  Let $1 > p > \delta > 0$ and $p + \delta \le 1$. Then
  \begin{equation*}
    (p + \delta) \log(p + \delta) + (p - \delta) \log (p - \delta)
    > 2 p \log p.
  \end{equation*}
\end{lemma}
\begin{proof}
  Since the function $p \mapsto f(p) = p \log p$ is strictly convex
  over $\openright{0}{\infty}$, we may apply convexity. Indeed,
  $p = \half(p + \delta) + \half (p - \delta)$, so
  \begin{equation*}
    p \log p = f\left(\half(p + \delta) + \half(p - \delta)\right)
    < \half f(p + \delta) + \half f(p - \delta),
  \end{equation*}
  which is the desired result.
\end{proof}

\subsection{Proof of Lemma~\ref{lemma:linear-linf-mutual-information}}
\label{appendix:mutual-information-computation-linear-linf}

The subgradient set $\partial \loss(X_i; \optvar)$ is
independent of $\optvar$, so we may use the
inequality~\eqref{eqn:observed-cube-information} to bound the mutual
information of $\cuberv$ and a single sample $\channelrv$ while ignoring
the dependence on $\optvar$. Define $M =
M_\infty / \lipobj$.  Since the sampling
scheme~\eqref{eqn:lone-linear-samples} is independent per-coordinate, we see
immediately that if $\channelrv_j$ denotes the $j$th coordinate of
$\channelrv$ then
\begin{equation*}
  \information(\channelrv; \cuberv)
  = H(\channelrv) - H(\channelrv \mid \cuberv)
  \le d \log(2) - \sum_{j = 1}^d H(\channelrv_j \mid \cuberv).
\end{equation*}
Since $\cuberv$ is uniformly chosen from one of $2d$ vectors, we
additionally find that
\begin{equation*}
  \information(\channelrv; \cuberv)
  \le d \left[\log 2 - \frac{1}{2d}\sum_{\cubecorner \in \hypercubeset}
    H(\channelrv \mid \cuberv = \cubecorner)\right].
\end{equation*}
By the choice of our sampling scheme for $X$ and $\channelrv$, we see that $H(\channelrv \mid
\cuberv = \cubecorner)$ is identical for each $\cubecorner \in
\hypercubeset$, and we have
\begin{equation*}
  \channelprob(\channelrv_j = M_\infty \mid \cuberv_j = \cubecorner_j = 0)
  = \channelprob(\channelrv_j = -M_\infty \mid \cuberv_j = \cubecorner_j = 0)
  = \half.
\end{equation*}
On the other hand, by our choice of sampling scheme,
for the ``on'' index in $\cuberv$, we have
\begin{align*}
  \channelprob(\channelrv_j = -M_\infty \mid \cuberv_j = \cubecorner_j = -1)
  & = \channelprob(\channelrv_j = M_\infty
    \mid \cuberv_j = \cubecorner_j = 1) \\
  & = \channelprob(\channelrv_j = M_\infty \mid X_j = 1)
  \statprob(X_j = 1 \mid \cuberv_j = \cubecorner_j = 1) \\
  & \qquad ~ +
  \channelprob(\channelrv_j = M_\infty \mid X_j = -1)
  \statprob(X_j = -1 \mid \cuberv_j = \cubecorner_j = 1) \\
  & = \left(\frac{M + 1}{2 M}\right) \left(\frac{1 + \delta}{2}\right)
  + \left(\frac{M - 1}{2 M}\right) \left(\frac{1 - \delta}{2} \right)
  = \half + \frac{\delta}{2M}.
\end{align*}
Consequently, defining the Bernoulli entropy $h(p) = -p \log p - (1 - p)
\log(1 - p)$, then
\begin{align*}
  \information(\channelrv; \cuberv)
  & \le d \left[ \log 2 - \frac{1}{2d}
    \left((2 d  - 2) \log 2 + 2 h \left(\half + \frac{\delta}{2M}\right)
      \right)\right] \\
  & = \log 2 + \left(\half + \frac{\delta}{2M}\right)
  \log\left(\half + \frac{\delta}{2M}\right)
  + \left(\half - \frac{\delta}{2M}\right)
  \log\left(\half - \frac{\delta}{2M}\right).
\end{align*}
The concavity of the function $p \mapsto \log(p)$ yields that
$\log(1/2 + p) \le \log(1/2) + 2p$, so
\begin{align*}
  \information(\channelrv; \cuberv)
  & \le \log 2 + \left(\half + \frac{\delta}{2M}\right)
  \left(-\log 2 + \frac{\delta}{M}\right)
  + \left(\half - \frac{\delta}{2M}\right)
  \left(-\log 2 - \frac{\delta}{M}\right)
  = \frac{\delta^2}{M^2}.
\end{align*}
Making the substitution $M = M_\infty / \lipobj$ completes the proof.

\subsection{Proof of Lemma~\ref{lemma:linf-mutual-information}}
\label{appendix:mutual-information-computation-linf}

By using the inequality~\eqref{eqn:observed-cube-information}, a bound on the
mutual information $\information(\channelrv; \cuberv \mid \optvar)$ implies a
bound on the joint information in the statement of the lemma, so we focus on
bounding the mutual information of a single sample $\channelrv$. In addition,
it is no loss of generality to assume that $\radius = 1$.

Define $M = M_\infty / \lipobj$ to be the multiple of the $\ell_\infty$-norm
of the subgradients that we take, and let $\channelrv_j$ denote the $j$th
coordinate of $\channelrv$.  Using the coordinate-wise independence of the
sampling, we have
\begin{equation*}
  \information(\channelrv; \cuberv \mid \optvar)
  = H(\channelrv \mid \optvar) - H(\channelrv \mid \cuberv, \optvar)
  \le d \log(2)
  - \sum_{j=1}^d H(\channelrv_j \mid \cuberv_j, \optvar_j).
\end{equation*}
Now consider the distribution of $\channelrv_j$ given $\cuberv_j$ and
$\optvar_j$.  By symmetry, the distribution has identical entropy for
any value of $\cuberv_j$, so we may fix $\cuberv = \cubecorner$ and
assume $\cubecorner_j = $ without loss of generality. Then for
$\optvar_j \in (-1, 1)$, the $j$th component of the subgradient
$\partial \loss(X; \optvar)$ is $-X_j$, whence we see that
\begin{align*}
  \lefteqn{\channelprob(\channelrv_j = M_\infty \mid \cubecorner_j = 1, \optvar_j)} \\
  & = \channelprob(\channelrv_j = M_\infty \mid X_j = 1, \optvar_j)
  \statprob(X_j = 1 \mid \cubecorner_j = 1)
  + \channelprob(\channelrv_j = M_\infty \mid X_j = -1, \optvar_j)
  \statprob(X_j = -1 \mid \cubecorner_j = 1) \\
  & = \left(\frac{M - 1}{2M}\right)\left(
  \frac{1 + \delta}{2}\right)
  + \left(\frac{M + 1}{2M}\right)\left(\frac{1 - \delta}{2}\right) \\
  & = \frac{2M - 2 \delta}{4 M}
  = \half - \frac{\delta}{2M}.
\end{align*}
Similarly, $\channelprob(\channelrv_j = -M_\infty \mid \cubecorner_j = 1,
\optvar_j) = \half + \frac{\delta}{2M}$.  If $\optvar_j \ge 1$, then we have
that the subgradient $\partial |\optvar_j - \statrv_j| = 1$ with probability
1, and thus
\begin{equation*}
  Q(\channelrv_j = M_\infty \mid \cubecorner_j = 1, \optvar_j)
  = \left(\frac{M + 1}{2M}\right)\left(\frac{1 + \delta}{2}\right)
  + \left(\frac{M + 1}{2M}\right)\left(\frac{1 - \delta}{2}\right)
  = \half,
\end{equation*}
which increases the entropy $H(\channelrv_j \mid \cuberv_j, \optvar_j)$ by
Lemma~\ref{lemma:entropy-minimization}. Thus we see that the value $\optvar_j$
minimizing the entropy $H(\channelrv_j \mid \cuberv_j, \optvar_j)$ is given by
any $\optvar_j \in (-1, 1)$, yielding Bernoulli marginal $(\half + \delta /
2M, \half - \delta / 2M)$ on $\channelrv_j \mid \cuberv_j$.  Summarizing, we
have
\begin{equation*}
  \information(\channelrv; \cuberv \mid \optvar)
  \le d \log(2) + d\left[
    \left(\half + \frac{\delta}{2M}\right)
    \log\left(\half + \frac{\delta}{2M}\right)
    + \left(\half - \frac{\delta}{2M}\right)
    \log\left(\half - \frac{\delta}{2M}\right)\right].
\end{equation*}
As in the proof of Lemma~\ref{lemma:linear-linf-mutual-information},
we use the concavity of $\log$ to see that
\begin{align*}
  \information(\channelrv; \cuberv \mid \optvar)
  & \le
  d \log(2) + d\left[
    \left(\half + \frac{\delta}{2M}\right)
    (-\log(2) + \delta / M)
    + \left(\half - \frac{\delta}{2M}\right)
    (-\log(2) - \delta/ M)\right] \\
  & = d \left(\half + \frac{\delta}{2M}\right) \left(\frac{\delta}{M}\right)
  + d \left(\half - \frac{\delta}{2M}\right)
  \left(-\frac{\delta}{M}\right)
  = \frac{d \delta^2}{M^2}.
\end{align*}
Applying the bound~\eqref{eqn:observed-cube-information} and replacing $M =
M_\infty / \lipobj$ completes the proof.

\subsection{Proof of Lemma~\ref{lemma:lone-mutual-information}}
\label{appendix:mutual-information-computation-lone}

\newcommand{\denominator}{\ensuremath{D_\gamma}}

Letting $\channelrv$ denote a single subgradient sample using the
conditional distribution $\channelprob$ specified by
Proposition~\ref{proposition:lone-saddle-point},
we prove that
\begin{equation}
  \information(\channelrv; \cuberv \mid \optvar) \le \delta^2
  \channeldiff^2 \quad \mbox{for any $\optvar \in \R^d$},
  \label{eqn:lone-mutual-info-claim}
\end{equation}
which implies the lemma by the
representation~\eqref{eqn:observed-cube-information}.  Recall the SVM risk
defined using the individual hinge losses~\eqref{eqn:svm-loss}: by
construction, whenever $\statrv = e_i$, then the loss is equal to $\lipobj
\hinge{\radius - \optvar_i}$. We have
\begin{equation*}
  \partial \loss(e_i, \optvar)
  = \lipobj \left\{\begin{array}{ll}
  0 & \mbox{if~} \optvar_i > \radius \\
  -e_i & \mbox{otherwise} \end{array}\right.
  ~~~ \mbox{and} ~~~
  \partial \loss(-e_i, \optvar)
  = \lipobj \left\{\begin{array}{ll}
  0 & \mbox{if~} \optvar_i < -\radius \\
  e_i & \mbox{otherwise.} \end{array}\right.
\end{equation*}
For the remainder of this proof, we use the shorthand
\begin{equation*}
  \denominator \defeq e^\gamma + e^{-\gamma} + 2(d - 2)
\end{equation*}
for the denominator in many of our expressions.
If $\statrv = e_i$, then $\partial \loss(e_i, \optvar) = \lipobj e_i$
or $0$ as $\optvar_i \le \radius$ or $\optvar_i > \radius$.
Therefore, as we wish to communicate $\lipobj e_i$ or $0$,
the construction in Proposition~\ref{proposition:lone-saddle-point} implies
\begin{equation}
  \channelprob(\channelrv = M_1 e_i \mid \statrv = e_i, \optvar)
  = \begin{cases}
    \frac{e^{-\gamma}}{\denominator} & \mbox{if~} \optvar_i \le \radius \\
    \frac{1}{2d} & \mbox{if~} \optvar_i > \radius,
  \end{cases}
  \label{eqn:l1-conditional-distribution}
\end{equation}
and similarly we have for $j \neq i$ that
\begin{equation}
  \channelprob(\channelrv = M_1 e_j \mid \statrv = e_i, \optvar)
  = \begin{cases}
    \frac{1}{\denominator}
    & \mbox{if~} \optvar_i \le \radius \\
    \frac{1}{2d} & \mbox{if~} \optvar_i > \radius.
  \end{cases}
  \label{eqn:off-l1-conditional-distribution}
\end{equation}
For $\statrv = -e_i$, we have the conditional distribution
parallel to~\eqref{eqn:l1-conditional-distribution}:
\begin{equation*}
  \channelprob(\channelrv = M_1 e_i \mid \statrv = -e_i, \optvar)
  = \begin{cases}
    \frac{e^{\gamma}}{\denominator} &
    \mbox{if~} \optvar_i \ge -\radius \\
    \frac{1}{\denominator} & \mbox{if~} \optvar_i < -\radius.
  \end{cases}
\end{equation*}

For any given $\optvar$, we have that
\begin{equation}
  \information(\channelrv; \cuberv \mid \optvar)
  = H(\channelrv \mid \optvar) - H(\channelrv \mid \cuberv, \optvar)
  \le \log(2d) - \frac{1}{|\hypercubeset|}
  \sum_{\cubecorner \in \hypercubeset} H(\channelrv \mid \optvar,
  \cuberv = \cubecorner)
  \label{eqn:intermediate-lone-information-bound}
\end{equation}
since the choice of $\cuberv$ is uniform and $\channelrv$ takes on at most
$2d$ values.  We thus use the conditional
distributions~\eqref{eqn:l1-conditional-distribution}
and~\eqref{eqn:off-l1-conditional-distribution} to compute the entropy
$H(\channelrv \mid \optvar, \cuberv)$ (specifically, the minimal such entropy
across all values of $\optvar$). To do this, we compute the marginal
distribution $\channelprob(\channelval \mid \cubecorner)$, arguing that
$H(\channelrv \mid \optvar, \cuberv)$ is minimal for $\optvar \in \interior
[-\radius, \radius]^d$. When $\optvar_j \in (-\radius, \radius)$ for all $j$,
we have
\begin{align*}
  \channelprob(\channelrv = M_1 e_i \mid \cuberv = \cubecorner, \optvar)
  & = \sum_{j = 1}^d \channelprob(\channelrv = M_1 e_i \mid
  X = e_j, \optvar) \statprob(X = e_j \mid \cuberv = \cubecorner) \\
  & \quad ~ + \sum_{j = 1}^d \channelprob(\channelrv = M_1 e_i \mid
  X = -e_j, \optvar) \statprob(X = -e_j \mid \cuberv = \cubecorner).
\end{align*}
When $\cubecorner_i = 1$, we thus have that
\begin{subequations}
  \begin{align}
    \channelprob(\channelrv = M_1 e_i \mid \cuberv = \cubecorner, \optvar)
    & = 
    \frac{1 + \delta}{2d} \frac{e^{-\gamma}}{\denominator}
    + \frac{1 - \delta}{2d} \frac{e^\gamma}{\denominator}
    + \sum_{j \neq i}
    \frac{1}{\denominator}
    \left(\frac{1 + \delta \cubecorner_j}{2 d}
    + \frac{1 - \delta \cubecorner_j}{2 d}\right) \nonumber \\
    & = \frac{e^\gamma + e^{-\gamma} + \delta (e^{-\gamma} - e^\gamma)}{
      2d \denominator}
    + \frac{d - 1}{d \denominator}
    = \frac{1}{2d} +
    \frac{\delta (e^{-\gamma} - e^\gamma)}{2d \denominator}
    \label{eqn:l1-marginal-conditional-a},
  \end{align}
  and under the same condition,
  \begin{align}
    Q(\channelrv = -M_1 e_i \mid A = \cubecorner, \optvar)
    = \frac{e^\gamma + e^{-\gamma} +
      \delta(e^\gamma + e^{-\gamma})}{2d \denominator}
    + \frac{d - 1}{d \denominator}
    = \frac{1}{2d}
    + \frac{\delta (e^\gamma - e^{-\gamma})}{2d \denominator}.
    \label{eqn:l1-marginal-conditional-b}
  \end{align}
\end{subequations}
If for any (possibly multiple) indices $j$ we have $\optvar_j \not \in
(-\radius, \radius)$, then via a bit of algebra and the conditional
distributions~\eqref{eqn:l1-conditional-distribution}
and~\eqref{eqn:off-l1-conditional-distribution}, we see that there exists an
$\epsilon \in (0, 1)$ such that
\begin{equation*}
  Q(\channelrv = M_1 e_i \mid \cuberv = \cubecorner, \optvar)
  = \epsilon \frac{1}{2 d} + (1 - \epsilon)
  \left(\frac{1}{2d}
  + \frac{\delta (e^{-\gamma} - e^\gamma)}{2d \denominator} \right).
\end{equation*}
Lemma~\ref{lemma:entropy-minimization} then implies that if $\optvar \in
\interior[-\radius, \radius]^d$ while $\optvar' \not \in \interior[-\radius,
  \radius]^d$, then
\begin{equation*}
  H(\channelrv \mid \optvar, \cuberv = \cubecorner)
  < H(\channelrv \mid \optvar', \cuberv = \cubecorner).
\end{equation*}
Since we seek an upper bound on the mutual information, we may thus assume
without loss of generality that $\optvar \in \interior[-\radius, \radius]^d$.

Now we compute the entropy $H(\channelrv \mid \optvar, \cubecorner)$ using the
marginal conditional distributions~\eqref{eqn:l1-marginal-conditional-a}
and~\eqref{eqn:l1-marginal-conditional-b}, which describe $\channelrv \mid
\cuberv$ when $\optvar \in \interior[-\radius, \radius]^d$. Indeed, recall the
definition in the statement of the lemma of the difference $\channeldiff$.
For $\channelval \in \{\pm M_1 e_j\}_{j=1}^d$, define the relation
$\channelval \sim \cubecorner$ to mean that if $\channelval = M_1 e_i$, then
$\cubecorner_i = 1$ and if $\channelval = -M_1 e_i$ then $\cubecorner_i = -1$.
We then see that the entropy is
\ifdefined\usejournalmargins
\begin{align*}
  \lefteqn{H(\channelrv \mid \optvar, \cuberv = \cubecorner)
    = -\sum_{\channelval \sim \cubecorner}
    \channelprob(\channelval \mid \cubecorner, \optvar)
    \log \channelprob(\channelval \mid \cubecorner, \optvar)
    - \sum_{\channelval \not\sim \cubecorner}
    \channelprob(\channelval \mid \cubecorner, \optvar)
    \log \channelprob(\channelval \mid \cubecorner, \optvar)} \\
  & \qquad\quad ~
  = -d\left(\frac{1}{2d} + \frac{\delta\channeldiff}{2d}\right)
  \log\left(\frac{1}{2d} + \frac{\delta\channeldiff}{2d}\right)
  - d \left(\frac{1}{2d} - \frac{\delta\channeldiff}{2d}\right)
  \log\left(\frac{1}{2d} - \frac{\delta\channeldiff}{2d}\right).
\end{align*}
\else
\begin{align*}
  H(\channelrv \mid \optvar, \cuberv = \cubecorner)
  & = -\sum_{\channelval \sim \cubecorner}
  \channelprob(\channelval \mid \cubecorner, \optvar)
  \log \channelprob(\channelval \mid \cubecorner, \optvar)
  - \sum_{\channelval \not\sim \cubecorner}
  \channelprob(\channelval \mid \cubecorner, \optvar)
  \log \channelprob(\channelval \mid \cubecorner, \optvar) \\
  & = -d\left(\frac{1}{2d} + \frac{\delta\channeldiff}{2d}\right)
  \log\left(\frac{1}{2d} + \frac{\delta\channeldiff}{2d}\right)
  - d \left(\frac{1}{2d} - \frac{\delta\channeldiff}{2d}\right)
  \log\left(\frac{1}{2d} - \frac{\delta\channeldiff}{2d}\right).
\end{align*}
\fi
As in the proofs of Lemmas~\ref{lemma:linear-linf-mutual-information}
and~\ref{lemma:linf-mutual-information}, we use the concavity
of $\log(\cdot)$ to see that
\begin{align*}
  -H(\channelrv \mid \optvar, \cuberv = \cubecorner)
  & = \left(\half + \frac{\delta \channeldiff}{2}\right)
  \log\left(\frac{1}{2d} + \frac{\delta \channeldiff}{2d}\right)
  + \left(\half - \frac{\delta \channeldiff}{2}\right)
  \log\left(\frac{1}{2d} - \frac{\delta \channeldiff}{2d}\right) \\
  & \le \left(\half + \frac{\delta \channeldiff}{2}\right)
  \left(-\log(2d) + \delta \channeldiff\right)
  + \left(\half - \frac{\delta \channeldiff}{2}\right)
  \left(-\log(2d) - \delta \channeldiff\right) \\
  & = -\log(2d) + \delta^2 \channeldiff^2.
\end{align*}
Invoking the earlier bound~\eqref{eqn:intermediate-lone-information-bound}
and adding $\log(2d)$ to the above expression completes the proof of
the claim~\eqref{eqn:lone-mutual-info-claim}.

\subsection{Proof of Lemma~\ref{lemma:diffp-linear-information}}
\label{appendix:mutual-information-computation-diffp}

Let $\channelrv_j$ denote the $j$th coordinate of $\channelrv$. We first
argue that conditional on $\cuberv$, the random variable $\channelrv$ has
independent coordinates. Indeed, let $\channeldensity^+ =
\channeldensity(\channelval \mid \statsample)$ for $\channelval$ such that
$\channelval^\top \statsample > k$ and $\channeldensity^- = e^{-\diffp}
\channeldensity^+$. Without loss of generality, we
may take $\cuberv = e_1$, the first basis vector, and hence
\begin{align}
  \channelprob(\channelrv = \channelval
  \mid \cuberv = e_1)
  & = \sum_{\statsample \in \{-1, 1\}^d}      
  \channelprob(\channelrv = \channelval \mid \statrv = \statsample)
  \statprob(\statrv = \statsample \mid \cuberv = e_1) \nonumber \\
  & = \frac{1}{2^{d-1}} \sum_{\statsample \in \{-1, 1\}^d}      
  \channelprob(\channelrv = \channelval \mid \statrv = \statsample)
  \cdot \frac{1 + \statsample_1 \delta}{2} \nonumber \\
  & = \frac{1}{2^{d-1}} \bigg[
    \sum_{\statsample : \<\channelval, \statsample\> > k} \channeldensity^+
    \frac{1 + \statsample_1 \delta}{2}
    + \sum_{\statsample : \<\channelval, \statsample\> \le k}
    \channeldensity^-
    \frac{1 + \statsample_1 \delta}{2}
    \bigg]. \label{eqn:prob-z}
\end{align}
Now, if $\channelval_1 = 1$, then
\ifdefined\usejournalmargins
\begin{align*}
  \sum_{\statsample : \<\statsample, \channelval\> > k}
  \frac{1 + \statsample_1 \delta}{2}
  & = \sum_{\statsample : \<\statsample, \channelval\> > k,
    \statsample_1 = 1} \frac{1 + \delta}{2}
  + \sum_{\statsample : \<\statsample, \channelval\> > k,
    \statsample_1 = -1} \frac{1 - \delta}{2} \\
  & = \frac{1 + \delta}{2} C_{d-1}(k - 1) + \frac{1 - \delta}{2} C_{d-1}(k + 1)
\end{align*}
\else
\begin{equation*}
  \sum_{\statsample : \<\statsample, \channelval\> > k}
  \frac{1 + \statsample_1 \delta}{2}
  = \sum_{\statsample : \<\statsample, \channelval\> > k,
    \statsample_1 = 1} \frac{1 + \delta}{2}
  + \sum_{\statsample : \<\statsample, \channelval\> > k,
    \statsample_1 = -1} \frac{1 - \delta}{2}
  = \frac{1 + \delta}{2} C_{d-1}(k - 1) + \frac{1 - \delta}{2} C_{d-1}(k + 1)
\end{equation*}
\fi
and similarly
\begin{equation*}
  \sum_{\statsample : \<\statsample, \channelval\> \le k}
  \frac{1 + \statsample_1 \delta}{2}
  = \frac{1 + \delta}{2} \left(2^{d - 1} - C_{d-1}(k - 1)\right)
  + \frac{1 - \delta}{2} \left(2^{d - 1} - C_{d-1}(k + 1)\right).
\end{equation*}
On the other hand, we find that if $z_1 = -1$, then similar equalities hold,
but with the counters $C_{d-1}(k - 1)$ and $C_{d-1}(k + 1)$ flipped:
\begin{align*}
  \sum_{\statsample : \<\statsample, \channelval\> > k}
  \frac{1 + \statsample_1 \delta}{2}
  & = \frac{1 + \delta}{2} C_{d-1}(k + 1) + \frac{1 - \delta}{2} C_{d-1}(k - 1)
  \\
  \sum_{\statsample : \<\statsample, \channelval\> \le k}
  \frac{1 + \statsample_1 \delta}{2}
  & = \frac{1 + \delta}{2} \left(2^{d - 1} - C_{d-1}(k + 1)\right)
  + \frac{1 - \delta}{2} \left(2^{d - 1} - C_{d-1}(k - 1)\right).
\end{align*}
In particular, we find that so long as the first coordinate $\channelval_1 =
\channelval_1'$ of $\channelval$ remains constant, then
$\channelprob(\channelrv = \channelval \mid \cuberv = e_1) = 
\channelprob(\channelrv = \channelval' \mid \cuberv = e_1)$, and that
we thus have $\channelrv_2, \ldots, \channelrv_d$ are distributed
uniformly at random in $\{-1, 1\}^d$.

We now determine $\channeldensity^+$ and compute the
marginal value $\channelprob(\channelrv_1 = 1 \mid \cuberv = e_1)$. For
the first, we note that
\begin{equation*}
  C_d(k) \channeldensity^+ + (2^d - C_d(k)) \channeldensity^- = 1,
  ~~~ \mbox{or} ~~~
  C_d(k) \channeldensity^+ + e^{-\diffp} (2^d - C_d(k)) \channeldensity^+
  = 1,
\end{equation*}
which yields the expressions
\begin{equation*}
  \channeldensity^+ = \frac{e^\diffp}{(e^\diffp - 1) C_d(k) + 2^d}
  ~~~ \mbox{and} ~~~
  \channeldensity^- = \frac{1}{(e^\diffp - 1) C_d(k) + 2^d}.
\end{equation*}
By the expression~\eqref{eqn:prob-z} and calculations following, we thus
find that when $\channelval_1 = 1$, we have
\begin{subequations}
  \begin{align}
    \channeldensity(\channelval \mid e_1)
    & = \frac{1}{2^{d-1}} \cdot
    \Big[\channeldensity^+
      \Big(\frac{1 + \delta}{2} C_{d-1}(k - 1) + \frac{1 - \delta}{2}
      C_{d - 1}(k + 1)\Big)
      \nonumber \\
      & \qquad\qquad\quad ~
      + \channeldensity^-\Big(
      \frac{1 + \delta}{2}(2^{d-1} - C_{d-1}(k - 1))
      + \frac{1 - \delta}{2}(2^{d - 1} - C_{d-1}(k + 1))\Big)\Big]
    \nonumber \\
    & = \frac{1}{2^{d - 1}} \cdot
    \Big[2^{d-1} \channeldensity^-
      + \half(\channeldensity^+ - \channeldensity^-)
      (C_{d-1}(k - 1) + C_{d-1}(k + 1))
      \nonumber \\
      & \qquad\qquad\quad ~
      + \frac{\delta}{2}
      (\channeldensity^+ - \channeldensity^-)
      (C_{d-1}(k - 1) - C_{d-1}(k + 1))\Big],
    \label{eqn:diffp-channel-pos}
  \end{align}
  and similarly when $\channelval_1 = -1$ we have
  \begin{align}
    \channeldensity(\channelval \mid e_1)
    & = \frac{1}{2^{d - 1}} \cdot
    \Big[2^{d-1} \channeldensity^-
      + \half(\channeldensity^+ - \channeldensity^-)
      (C_{d-1}(k - 1) + C_{d-1}(k + 1)) \nonumber \\
      & \qquad\qquad\quad ~
      - \frac{\delta}{2}
      (\channeldensity^+ - \channeldensity^-)
      (C_{d-1}(k - 1) - C_{d-1}(k + 1))\Big].
    \label{eqn:diffp-channel-neg}
  \end{align}
\end{subequations}
Now note that
\begin{equation*}
  C_{d-1}(k - 1) - C_{d-1}(k + 1)
  = \sum_{i=0}^{\ceil{(d - k)/2} - 1} \binom{d - 1}{i}
  - \sum_{i=0}^{\ceil{(d - k)/2} - 2} \binom{d - 1}{i}
  = \binom{d - 1}{\ceil{(d - k)/2} - 1}
\end{equation*}
and that the difference
\begin{equation*}
  \channeldensity^+ - \channeldensity^-
  = \frac{e^\diffp - 1}{(e^\diffp - 1) C_d(k) + 2^d}.
\end{equation*}
Recalling the definition of the constant $\Delta$, we thus find from the
expansions~\eqref{eqn:diffp-channel-pos}
and~\eqref{eqn:diffp-channel-neg}---since they must sum to 1---that
\begin{equation}
  \channelprob(\channelrv = \channelval \mid \cuberv = e_1)
  = \frac{1}{2^{d - 1}} \cdot
  \begin{cases}
    \half + \frac{\Delta(\delta, \diffp, d, k)}{2} & \mbox{if~}
    \channelval_1 = 1 \\
    \half - \frac{\Delta(\delta, \diffp, d, k)}{2} & \mbox{if~}
    \channelval_1 = -1.
  \end{cases}
  \label{eqn:final-diffp-channel-probability}
\end{equation}
It is clear that similar statements hold in the other symmetric cases
(i.e.\ if $\cuberv = -e_2$, then the probabilities depend on $\channelval_2
= -1$ or $1$).

It remains to use the marginalized
representation~\eqref{eqn:final-diffp-channel-probability} to compute the
bound on the mutual information in the statement of the lemma.
Given $\cuberv = \cubecorner$, $\channelrv \in \{-M, M\}^d$ is uniform
except on the coordinate $j$ for which $\cubecorner_j \neq 0$, by
symmetry. (Marginally, $\channelrv$ is uniform on $\{-M, M\}^d$.)
By a direct calculation, we have
$H(\channelrv) = d \log 2$ and
\begin{equation*}
  H(\channelrv \mid \cuberv = e_1)
  = \sum_{j = 1}^d H(\channelrv_j \mid \channelrv_{1:j-1}, \cuberv = e_1)
  = H(\channelrv_1 \mid \cuberv = e_1)
  + (d - 1) \log 2,
\end{equation*}
and similarly for the other possible values of $\packrv$.  Therefore, using
the probabilities~\eqref{eqn:final-diffp-channel-probability}, we have the
mutual information bound
\begin{align*}
  \lefteqn{\information(\channelrv; \cuberv)
    = H(\channelrv) - H(\channelrv \mid \cuberv)
    \le d \log 2 - \frac{1}{2d} \sum_\cubecorner H(\channelrv \mid \cuberv
    = \cubecorner)} \\
  & = d \log 2
  - (d - 1) \log 2 \\
  & \quad ~ +
  \left(\half + \frac{\Delta(\delta, \diffp, d, k)}{2}\right)
  \log \left(\half + \frac{\Delta(\delta, \diffp, d, k)}{2}\right)
  + \left(\half - \frac{\Delta(\delta, \diffp, d, k)}{2}\right)
  \log \left(\half - \frac{\Delta(\delta, \diffp, d, k)}{2}\right) \\
  & \le \log 2 + \left(\half + \frac{\Delta(\delta, \diffp, d, k)}{2}\right)
  \left[\log \half + \Delta(\delta, \diffp, d, k)\right]
  + \left(\half - \frac{\Delta(\delta, \diffp, d, k)}{2}\right)
  \left[\log \half - \Delta(\delta, \diffp, d, k)\right] \\
  & = \Delta(\delta, \diffp, d, k)^2,
\end{align*}
where the inequality follows from the concavity of $p \mapsto \log(p)$.


\subsection{Bounds on total variation norm}
\label{appendix:tv-norm-bounding}

\begin{lemma}
  \label{lemma:tv-norm-bound}
  Let $\channelprob_1$ and $\channelprob_{-1}$ be distributions on $\{-1,
  1\}$, where
  \begin{equation*}
    \channelprob_1(\channelrv = \channelval)
    = \half + \half\cdot
    \begin{cases}
      ~ \delta & \mbox{if~} \channelval = 1 \\
      -\delta & \mbox{otherwise}
    \end{cases}
    ~~~ \mbox{and} ~~~
    \channelprob_{-1}(\channelrv = \channelval)
    = \half + \half\cdot
    \begin{cases}
      -\delta & \mbox{if~} \channelval = 1 \\
      ~ \delta & \mbox{otherwise}.
    \end{cases}
  \end{equation*}
  Let $\channelprob_i^n$ denote the $n$-fold product distribution
  of $\channelprob_i$. Then for $\delta \in [0, 1/3]$,
  \begin{equation*}
    \tvnorm{\channelprob_1^n - \channelprob_{-1}^n}
    \le \delta \sqrt{(3/2) n}.
  \end{equation*}
\end{lemma}
\begin{proof}
  For any two probability distributions $P, Q$, Pinsker's
  inequality~\cite{CoverTh06} asserts that the total variation norm is
  bounded as \mbox{$\tvnorm{P - Q} \le \sqrt{\dkl{P}{Q} / 2}$.}
  Applying this inequality in our setting, we find that
  \begin{align*}
    \tvnorm{\channelprob_1^n - \channelprob_{-1}^n} & \le \sqrt{\half
      \dkl{\channelprob_1^n}{\channelprob_{-1}^n}} = \frac{1}{\sqrt{2}}
    \sqrt{n \dkl{\channelprob_1}{\channelprob_{-1}}},
  \end{align*}
  where we have exploited the product nature of $\channelprob_i^n$.  Now
  we note that by the concavity of the $\log$, we have (via the
  first-order inequality) that $\log\frac{1 + \delta}{1 - \delta} \le 2
  \delta / (1 - \delta)$, so
  \begin{align*}
    \frac{1 + \delta}{2}
    \log \frac{\frac{1 + \delta}{2}}{\frac{1 - \delta}{2}}
    + \frac{1 - \delta}{2}
    \log \frac{\frac{1 - \delta}{2}}{\frac{1 + \delta}{2}}
    & = \frac{1 + \delta}{2}
    \log \frac{1 + \delta}{1 - \delta}
    + \frac{1 - \delta}{2} \log \frac{1 - \delta}{1 + \delta}
    = \delta \log \frac{1 + \delta}{1 - \delta} 
    \le \frac{2 \delta^2}{1 - \delta}.
  \end{align*}
  Assuming that $\delta \le 1/3$, the final term is upper bounded
  by $3 \delta^2$.
  But of course by definition of $\channelprob_1$ and $\channelprob_{-1}$,
  we have
  \begin{equation*}
    \dkl{\channelprob_1}{\channelprob_{-1}}
    = \frac{1 + \delta}{2}
    \log \frac{\frac{1 + \delta}{2}}{\frac{1 - \delta}{2}}
    + \frac{1 - \delta}{2}
    \log \frac{\frac{1 - \delta}{2}}{\frac{1 + \delta}{2}}
    \le 3 \delta^2,
  \end{equation*}
  which completes the proof.
\end{proof}

%% file: sgd-attainability.tex
\section{Achievability by stochastic mirror descent}
\label{appendix:sgd-achievability}

In this appendix, we provide further details on the
algorithm used to achieve the upper bounds in
Theorems~\ref{theorem:first-class} and~\ref{theorem:second-class}.  Both of
our achievability results rely on stochastic gradient mechanisms, and their
most important ingredient is a conditional distribution $\channel$ that
satisfies $\diffp$-local differential privacy.  In particular, if $g \in \R^d$
is a (sub)gradient of the loss $\loss(\statsample, \optvar)$, we construct
$\channelrv \in \R^d$ by perturbing $g$ in such a way that $\E[\channelrv \mid
  g] = g$. Thus, each of the achievability guarantees consists
of describing an $\diffp$-differentially private sampling distribution, then
bounding the expected norm of $\channelrv$ and applying 
one of the convergence guarantees~\eqref{eqn:stochastic-bounds}.

\subsection{Achievability in Theorem~\ref{theorem:first-class}}
\label{appendix:sgd-linf-achievability}

\newcommand{\sbound}{B}
\newcommand{\bernoulli}{\mathop{\rm Bernoulli}}
\newcommand{\uniform}{\mathop{\rm Uniform}}

The sampling strategy we use is essentially identical to that used in
Corollary~\ref{corollary:linf-privacy-diffp} (and the optimal $\diffp$-private
scheme of Proposition~\ref{proposition:linf-diffp-saddle-point}; see
also Strategy B in our paper~\cite{DuchiJoWa13_parametric}). Let
$\pi_\diffp \defeq e^\diffp / (e^\diffp + 1)$ and $T$ be a
$\bernoulli(\pi_\diffp)$-random variable, and let $\sbound \ge
\lipobj$ be a fixed constant (to be specified). Then given a vector $g \in
\R^d$ with $\linf{g} \le \lipobj$, construct $\wt{g} \in \R^d$ with
coordinates $\wt{g}_j$ sampled independently from $\{-\lipobj, \lipobj\}$ with
probabilities $1/2 - g_j / (2\lipobj)$ and $1/2 + g_j / (2 \lipobj)$.  Then
sample $T$ and set
\begin{equation}
  \label{eqn:linf-sampling}
  \channelrv \sim \begin{cases}
    \uniform(\channelval \in \{-\sbound, \sbound\}^d :
    \<\channelval, \wt{g}\> > 0)
    & \mbox{if~} T = 1 \\
    \uniform(\channelval \in \{-\sbound, \sbound\}^d :
    \<\channelval, \wt{g}\> \le 0)
    & \mbox{if~} T = 0.
  \end{cases}
\end{equation}
By inspection, the scheme~\eqref{eqn:linf-sampling} is $\diffp$-differentially
private. Moreover, we have by the calculations in the proof of
Corollary~\ref{corollary:linf-privacy-diffp} (see
Section~\ref{sec:corollary-rate-linf-diffp}) that by the sampling
strategy~\eqref{eqn:linf-sampling}
\begin{equation*}
  \E[\channelrv \mid g]
  = \E[\E[\channelrv \mid \wt{g}] \mid g]
  = g \frac{\sbound}{2^{d-1} \lipobj} \frac{e^\diffp - 1}{e^\diffp + 1}
  \cdot \begin{cases} \binom{d - 1}{\frac{d - 1}{2}}
    & d ~ \mbox{odd} \\
    \binom{d - 1}{d/2} & d ~ \mbox{even.} \end{cases}
\end{equation*}
Thus, w.l.o.g.\ assuming $d$ is odd, choosing
\begin{equation*}
  \sbound =
  2^{d-1}
  \lipobj\frac{e^\diffp + 1}{e^\diffp - 1} \binom{d - 1}{\frac{d - 1}{2}}^{-1}
  ~~~ \mbox{implies} ~~~
  \E[\channelrv \mid g] = g
  ~~ \mbox{and} ~~
  \linf{\channelrv} = \sbound = \order(1) (\lipobj / \diffp)
  \sqrt{d}.
\end{equation*}
Applying the mirror descent method to the gradients provided
from the sampling strategy~\eqref{eqn:linf-sampling},
we obtain the bound~\eqref{eqn:mirror-descent-bound}
with $M_\infty = \sbound = \order(1) (\lipobj / \diffp) \sqrt{d}$,
which is our desired result.

\subsection{Achievability in Theorem~\ref{theorem:second-class}}
\label{appendix:sgd-ltwo-achievability}

The achievability result for Theorem~\ref{theorem:second-class} is similar to
that for Theorem~\ref{theorem:first-class}, but we use a modified sampling
distribution.  Now, using the same notation as that for the
strategy~\eqref{eqn:linf-sampling}, we use the following.  Given a vector $g$
with $\ltwo{g} \le \lipobj$, set $\wt{g} = \lipobj g / \ltwo{g}$ with
probability $\half + \ltwo{g} / 2\lipobj$ and $\wt{g} = - \lipobj g /
\ltwo{g}$ with probability $\half - \ltwo{g} / 2 \lipobj$.  Then sample $T
\sim \bernoulli(\pi_\diffp)$ and set
\begin{equation}
  \label{eqn:ltwo-sampling}
  \channelrv \sim \begin{cases}
    \uniform(\channelval \in \R^d :
    \<\channelval, \wt{g}\> > 0, \ltwo{\channelval} = \sbound)
    & \mbox{if~} T = 1 \\
    \uniform(\channelval \in \R^d :
    \<\channelval, \wt{g}\> \le 0, \ltwo{\channelval} = \sbound)
    & \mbox{if~} T = 0.
  \end{cases}
\end{equation}
(This is Strategy A in our paper~\cite{DuchiJoWa13_parametric}.) Then we
have~\cite[Appendix D.1]{DuchiJoWa13_parametric} that
\begin{equation*}
  \E[\channelrv \mid g]
  = \frac{\sbound}{\lipobj} \frac{e^\diffp - 1}{e^\diffp + 1}
  \frac{\Gamma(\frac{d}{2} + 1)}{\sqrt{\pi} d \Gamma(\frac{d - 1}{2} + 1)},
\end{equation*}
so choosing
\begin{equation*}
  \sbound = \lipobj \frac{e^\diffp + 1}{e^\diffp - 1}
  \frac{\sqrt{\pi} d \Gamma(\frac{d - 1}{2} + 1)}{\Gamma(\frac{d}{2} + 1)}
  \le \lipobj \frac{e^\diffp + 1}{e^\diffp - 1} \frac{3 \sqrt{\pi} \sqrt{d}}{2}
\end{equation*}
implies that $\E[\channelrv \mid g] = g$ and $\ltwo{\channelrv} \le \sbound =
\order(1) (\lipobj / \diffp) \sqrt{d}$.
Applying the stochastic gradient descent method to the gradients
provided by the sampling scheme~\eqref{eqn:ltwo-sampling},
we obtain the bound~\eqref{eqn:sgd-bound} with $M_2 = \sbound$,
which implies that if $\Theta \subset \ball_2(\radius_2)$ then
\begin{equation*}
  \E[\risk(\what{\optvar}_n)] - \risk(\optvar^*)
  = \order\left(\frac{\sqrt{d}}{\diffp}
  \frac{\lipobj \radius_2}{\sqrt{n}}\right).
\end{equation*}
Noting that $\ball_q(\radius_q) \subset d^{\half - \frac{1}{q}}
\ball_2(\radius_q)$ completes the proof of the achievability result.

%% file: background-on-conditional-probability.tex
\section{Background on Conditional Probabilities}
\label{appendix:background-on-conditional-probability}

In this appendix, we present some basic lemmas on conditional independence
and regular conditional probabilities that will be useful in 
Appendix~\ref{appendix:saddle-point-proofs}.

We first recall the following classical data-processing inequality, which
 holds for essentially arbitrary random variables~\cite[Chapter 5]{Gray90}:
\begin{lemma}[Data processing]
  \label{lemma:data-processing}
  Let $\statrv \rightarrow \channelrv \rightarrow Y$ be a Markov
  chain. Then $\information(\statrv; Y) \le \information(\statrv;
  \channelrv)$, with equality if and only if $\statrv$ is
  conditionally independent of $Y$ given $\channelrv$.
\end{lemma}

This inequality, in conjunction with with Carath\'eodory and
Minkowski's finite-dimensional version of the Krein-Milman
theorem (e.g.~\cite{HiriartUrrutyLe96ab}), allows us to argue any
$\channelprob$ minimizing $\information(\statprob, \channelprob)$ must
be supported on the extreme points of $D$.  To make this point precise,
however, we need to address certain measurability issues involved in
the choice of the extreme points.

We begin with a precise definition of a regular conditional probability.
\begin{definition}
  \label{def:markov-kernel}
  Let $(\Omega, \mc{F})$ and $(T, \sigma(T))$ be measurable spaces.  A
  \emph{regular conditional probability}, also known as a Markov
  kernel or transition probability, is a function $\nu : T \times
  \mc{F} \rightarrow [0, 1]$ such that
  \begin{align*}
    t \mapsto \nu(t, A) & ~~ {\rm is~measurable~for~all~} A \in
    \mc{F} \\ \nu(t, \cdot) : \mc{F} \rightarrow [0, 1] & ~~ {\rm
      is~a~probability~measure~for~all~} t \in T.
  \end{align*}
\end{definition}
\noindent Any Markov chain has a transition probability; conversely,
any set of consistent transition probabilities define a Markov chain
(see, e.g., Chapter 5 of~\citet{Kallenberg97}).

Some difficulties with measurability arise in constructing the
appropriate Markov chain for our setting.  To deal with them, we use 
results from Choquet theory, which extend Krein-Milman theorems to 
integral representations~\cite{Phelps01}. We begin our proof by stating 
a measurable selection theorem~\cite[Theorem 11.4]{Phelps01}, 
though we restrict the theorem's statement to subsets of finite 
dimensional space.
\begin{proposition}
  \label{proposition:choquet}
  Let $D \subset \R^d$ be a compact convex set. For each
  $x$, there exists a probability measure $\mu_x$ supported on $\extreme(D)$
  such that $\int_D y d\mu_x(y) = x$. Moreover, the mapping
  $x \mapsto \mu_x$ can be taken to be measurable.
\end{proposition}
\noindent In the statement of this result, measurability is taken with
respect to the $\sigma$-field generated by the topology of weak
convergence. As a consequence of the proposition, however, it is clear
that since for any continuous function $f$ the mapping $x \mapsto \int
f d\mu_x$ is measurable, we have that for relatively open sets $A
\subset C$ the mapping $x \mapsto \mu_x(A)$ is measurable, whence for
any measurable set $A \subset C$ the mapping $x \mapsto \mu_x(A)$ is
measurable.  That is, we can define the Markov kernel $\nu : \R^d
\times \sigma(C) \rightarrow [0, 1]$ according to the mapping
specified by Proposition~\ref{proposition:choquet} (we take $\nu(x,
\cdot) = \mu_x$) with the additional properties that
\begin{equation*}
  \int_D y \nu(x, dy) = x
  ~~~ {\rm and} ~~~
  \nu(x, D \setminus \extreme(D)) = 0
  ~~ {\rm for~all~} x \in D.
\end{equation*}
In finite dimensions, a trivial extension of
Proposition~\ref{proposition:choquet} allows us to drop the assumption that
$D$ is convex. Indeed, we have that since $D$ is compact, then
$\extreme(D) = \extreme(\conv(D))$~\cite[Chapter III.2]{HiriartUrrutyLe96ab}.

Given this measure-theoretic background, we turn to a key lemma that 
we will need in Appendix~\ref{appendix:saddle-point-proofs}.  
In this lemma, we assume as usual that $C \subset D \subset \R^d$ 
are compact sets, and that $\channelprob \in \channeldistset(C, D)$ (recall the
definition~\eqref{eqn:channel-distribution-set}).
\begin{lemma}
  \label{lemma:extreme-points}
  Let $\statprob$ be a distribution supported on $C$. If there exists a set $A
  \subset C$ with $\statprob(A) > 0$ and a set $B \subset D \setminus
  \extreme(D)$ with $\channelprob(B \mid \statrv = x) > 0$ for $x \in A$,
  there exists a regular conditional probability distribution $\channelprob'
  \in \channeldistset(C, D)$ where $\channelprob'(\cdot \mid x)$ has support
  contained in $\extreme(D)$ and
  \begin{equation*}
    \information(\statprob, \channelprob) > \information(\statprob,
    \channelprob').
  \end{equation*}
\end{lemma}
\noindent
Paraphrasing the lemma slightly, we have that \emph{any} conditional 
distribution $\channelprob$ minimizing $\information(\statprob, \channelprob)$ 
must (outside of a set of measure zero) be completely supported on the extreme 
points $\extreme(D)$.

\begin{proof}
For any $y \in D$, Proposition~\ref{proposition:choquet} guarantees that 
we can represent $y$ as the (regular conditional) measure $\nu(y, \cdot)$.  
Thus we can define a random variable $Z_y$ distributed according to $\nu(y,
\cdot)$, whose existence we are guaranteed by standard
constructions~\cite{Billingsley86,Kallenberg97} with regular
conditional probability.  Then $\E[Z_y] = \int_D z \nu(y, dz) = y$,
and moreover, we can define the measurable version of the conditional
expectation $\E[Z_Y \mid Y]$ via
  \begin{equation*}
    \E[Z_Y \mid Y] = \int_D z \nu(Y, dz) = Y
  \end{equation*}
  so we have the (almost sure) chain of equalities
  \begin{align*}
    \E[Z_Y \mid X = x] & = \E[\E[Z_Y \mid Y] \mid X = x]
    = \int_D \E[Z_Y \mid Y = y] dQ(y \mid X = x) \\
    & = \int_D \int_D z \nu(y, dz) dQ(y \mid X = x)
    = \int_D y dQ(y \mid X = x) = x.
  \end{align*}
  By construction, $X \rightarrow Y \rightarrow Z$ is a valid Markov chain,
  and since the sets $A$ and $B$ satisfy $P(A) > 0$ and $\int_A Q(B \mid X =
  x) dP(x) > 0$, we see that $\information(X; Y) > \information(X; Z)$ by
  Lemma~\ref{lemma:data-processing}.
\end{proof}

We turn to an analogue of Lemma~\ref{lemma:extreme-points} in the 
differentially private setting.
\begin{lemma}
  \label{lemma:diffp-extreme-points}
  Let the conditions of Lemma~\ref{lemma:extreme-points} hold,
  and let $\statprob$ be a distribution
  supported on $C$. If there exists a set $A \subset C$ with
  $\statprob(A) > 0$ and a set $B \subset D \setminus \extreme(D)$ with
  $\channelprob(B \mid \statrv = x) > 0$ for $x \in A$, there exists a
  regular conditional probability distribution $\channelprob'
  \in \channeldistset[C, D]$ where
  $\channelprob'(\cdot \mid x)$ has support contained in $\extreme(D)$,
  satisfies
  \begin{equation*}
    \information(\statprob, \channelprob) > \information(\statprob,
    \channelprob'),
  \end{equation*}
  and has no worse differential privacy than $\channelprob$:
  \begin{equation*}
    \sup_{S \in \sigma(D)} \sup_{\statsample, \statsample' \in C}
    \frac{\channelprob'(S \mid \statrv = \statsample)}{
      \channelprob'(S \mid \statrv = \statsample')}
    \le \sup_{S \in \sigma(D)} \sup_{\statsample, \statsample' \in C}
    \frac{\channelprob(S \mid \statrv = \statsample)}{
      \channelprob(S \mid \statrv = \statsample')}.
  \end{equation*}
\end{lemma}
\begin{proof}
Let $\nu : \R^d \times \sigma(C)
\rightarrow [0, 1]$ be the Markov kernel defined in the proof of
Lemma~\ref{lemma:extreme-points}, and without loss of generality assume that
$\channelprob(\cdot \mid \statrv = \statsample)$ and $\channelprob(\cdot
\mid \statrv = \statsample')$ have density $\channeldensity$ with respect to
an underlying measure $\mu_{\statsample,\statsample'}$. Define the
distribution
\begin{equation*}
  \channelprob'(S \mid \statrv = \statsample)
  \defeq \int_D \int_D \nu(y, d\channelval) \channeldensity(y \mid
  \statsample) d\mu_{\statsample,\statsample'}(y).
\end{equation*}
By assumption, if $\channelprob$ is $\diffp$-differentially private, then
for $\mu$-almost all $y \in D$, we have $\channeldensity(y \mid \statsample)
\le e^\diffp \channeldensity(y \mid \statsample')$. We find that
\begin{align*}
  \channelprob'(S \mid \statrv = \statsample)
  & = \int_D \int_D \nu(y, d\channelval) \channeldensity(y \mid
  \statsample) d\mu_{\statsample,\statsample'}(y) \\
  & \le
  \int_D \int_D \nu(y, d\channelval) e^\diffp \channeldensity(y \mid
  \statsample') d\mu_{\statsample,\statsample'}(y)
  = e^\diffp \channelprob'(S \mid \statrv = \statsample'),
\end{align*}
so $\channelprob'$ is at least as differentially private as $\channelprob$.
\end{proof}

Finally, we will need the following standard maximum entropy result.
Let $\channelval$ denote a discrete random variable and let
$q(\channelval \mid x)$ denote the conditional probability mass 
function of $\channelrv \mid \statrv = x$.  Consider the finite 
dimensional entropy maximization problem
  \begin{align}
    \minimize_q ~ & \sum_\channelval q(\channelval \mid x) \log
    q(\channelval \mid x)
    \label{eqn:entropy-maximization} \\
    \subjectto ~ & \sum_\channelval \channelval q(\channelval \mid x)
    = x, ~~ \sum_\channelval q(\channelval \mid x) = 1, ~~
    q(\channelval \mid x) \ge 0 ~ {\rm for~all~} \channelval.
    \nonumber
  \end{align}
  We have the following lemma, which establishes the form of the
  solution to the problem~\eqref{eqn:entropy-maximization}.  We
  include a proof for completeness.
  \begin{lemma}
    \label{lemma:max-ent-q}
    The p.m.f.\ $q(\cdot \mid x)$ solving
    problem~\eqref{eqn:entropy-maximization} is given by
    \begin{equation}
      q(\channelval \mid x) = \frac{\exp(-\mu^\top
        \channelval)}{\sum_{\channelval'} \exp(-\mu^\top
        \channelval')},
      \label{eqn:max-ent-q}
    \end{equation}
    where $\mu \in \R^d$ is any vector chosen to satisfy the
    constraint $\sum_\channelval \channelval q(\channelval \mid x) =
    x$. Such
    a $\mu \in \R^d$ exists.
  \end{lemma}
\begin{proof}
  We may write the Lagrangian with dual variables $\mu \in \R^d$,
  $\lambda(\channelval) \ge 0$, and $\theta \in \R$,
  \ifdefined\usejournalmargins
  \begin{align*}
    \lefteqn{\mc{L}(q, \mu, \lambda, \theta)} \\
    & = \sum_\channelval q(\channelval \mid x) \log q(\channelval \mid x)
    + \mu^\top \!\bigg(\sum_\channelval \channelval q(\channelval \mid x)
    - x\bigg)
    + \theta \bigg(\sum_\channelval q(\channelval \mid x) - 1\bigg)
    - \sum_\channelval \lambda(\channelval) q(\channelval \mid x).
  \end{align*}
  \else
  \begin{equation*}
    \mc{L}(q, \mu, \lambda, \theta)
    = \sum_\channelval q(\channelval \mid x) \log q(\channelval \mid x)
    + \mu^\top \bigg(\sum_\channelval \channelval q(\channelval \mid x)
    - x\bigg)
    + \theta \bigg(\sum_\channelval q(\channelval \mid x) - 1\bigg)
    - \sum_\channelval \lambda(\channelval) q(\channelval \mid x).
  \end{equation*}
  \fi
  Since the problem~\eqref{eqn:entropy-maximization} has convex cost, linear
  constraints, and non-empty domain, strong duality obtains~\cite[Chapter
    5]{BoydVa04}, and the KKT conditions hold for the problem. Thus,
  minimizing $q$ out of $\mc{L}$ to find the dual, we take derivatives with
  respect to the $m$ variables $q(\channelval \mid x)$ for $\channelval = (1 +
  \alpha) u_i$ and find the optimal conditional p.m.f.\ $q$ must satisfy
  \begin{equation*}
    \log q(\channelval \mid x) + 1 + \mu^\top \channelval + \theta -
    \lambda(\channelval) = 0,
    ~~~ \mbox{or} ~~~
    q(\channelval \mid x)
    = \exp(\lambda(\channelval) - 1 - \theta) \exp(-\mu^\top \channelval).
  \end{equation*}
  In particular, we see that since $q(\channelval \mid x) > 0$, we must have
  $\lambda(\channelval) = 0$ by complementarity, and (satisfying the
  summability constraint $\sum_\channelval q(\channelval \mid x) = 1$) we see
  that
  \begin{equation*}
    q(\channelval \mid x)
    = \frac{\exp(-\mu^\top \channelval)}{\sum_{\channelval'}
      \exp(-\mu^\top \channelval')},
  \end{equation*}
  where $\mu \in \R^d$ is any vector chosen to satisfy the constraint
  $\sum_\channelval \channelval q(\channelval \mid x) = x$. The existence of
  such a $\mu$ is guaranteed by the attainment of the KKT conditions.
\end{proof}

%% file: saddle-point-proofs.tex
\section{Proofs of Minimax Mutual Information Characterizations}
\label{appendix:saddle-point-proofs}

In this section, we provide the proofs of the results stated in
Section~\ref{sec:saddle-points}, all of which follow a broadly similar
outline.  We make use of Lemma~\ref{lemma:extreme-points}
to guarantee that any conditional distribution
$\channelprob$ minimizing the mutual information
$\information(\statprob, \channelprob)$ must be supported on the
extreme points of the set $D$. This allows us to reduce
computing maximal entropies and minimal mutual information values to
finite dimensional convex programs, whose optimality we
can check using results from convex analysis and optimization.

\subsection{Proof of Theorem~\ref{theorem:finite-rotation-information}}
\label{appendix:finite-rotation-information-proof}

  We begin by considering $\sup_\statprob$, where $\channelprob^*$ is
  defined as in the statement of the theorem.  Since the support of
  $\channelprob^*$ is finite (there are $m$ extreme points of $D$), we
  have
  \begin{align*}
    \information(\statprob, \channelprob^*)
    = \information(\statrv; \channelrv)
    = H(\channelrv) - H(\channelrv \mid \statrv)
    & \le \log(m) - H(\channelrv \mid \statrv) \\
    & = \log(m) - \int H(\channelrv \mid \statrv = x) d\statprob(x).
  \end{align*}
  Now, for any distribution $\statprob$ on the set $C$ and for any $x \in
  \supp \statprob$, we can write $x$ as $x = \sum_i \beta_i(x) u_i$, where
  $u_i$ are the extreme points of $C$, and where $\beta_i(x) \ge 0$ and 
  $\sum_i \beta_i(x) = 1$ (using the Krein-Milman theorem). 
  Define the individual probability mass functions $q^i$ to be the
  maximum entropy p.m.f.~\eqref{eqn:max-ent-q} for each of the extreme points
  $u_i$. Then we can define the conditional probability mass function by
  \begin{equation*}
    q(\cdot \mid x) = \sum_i \beta_i(x) q^i(\cdot).
  \end{equation*}
  (Without loss of generality, we may assume the $\beta_i$ are
  continuous, since the set of extreme points is finite, and thus $q(\cdot
  \mid x)$ can be viewed as a regular conditional probability. We can make
  this formal using the techniques in the proof of
  Lemma~\ref{lemma:extreme-points}.) Denoting $H(q(\cdot \mid x)) \defeq
  H(\channelrv \mid \statrv = x)$, we can use the convexity of the negative
  entropy to see that
  \begin{equation}
    \label{eqn:entropy-convexity}
    \information(\statprob, \channelprob^*) \le
    \log(m) - \int \sum_i \beta_i(x)
    H(q^i(\cdot)) d\statprob(x).
  \end{equation}
  By symmetry, the entropy $H(q^i(\cdot)) = H(\channelprob^*(\cdot \mid
  \statrv = u_i))$ is a constant determined by the maximum entropy
  distribution~\eqref{eqn:max-ent-q}, and thus
  \begin{equation}
    \information(\statprob, \channelprob^*) \le \log(m)
    - H(\channelprob^*(\cdot \mid \statrv = u_i)).
    \label{eqn:general-finite-p-q-upper-bound}
  \end{equation}
  Equality in the upper bound~\eqref{eqn:general-finite-p-q-upper-bound} is
  attained by taking $\statprob^*$ to be the uniform distribution on the
  extreme points $\{u_i\}$ of $C$.

  It remains to establish an identical lower bound for
  $\information(\statprob^*, \channelprob)$ over all conditional distributions
  $\channelprob$ satisfying the constraints of the theorem statement. We know
  from Lemma~\ref{lemma:extreme-points} that $\channelprob$ must be supported
  on $(1 + \kappa) u_i$ for $i = 1, \ldots, m$. Denoting by $q(\channelval
  \mid x)$ the p.m.f.\ of $\channelprob$ conditional on $x$ (for $x$ in the
  finite set of extreme points of $C$ that make up the support $\supp
  \statprob^*$), we can write minimizing the mutual information as the
  parametric convex optimization problem
  \begin{align}
    \label{eqn:information-minimization}
    \minimize_q ~ & \sum_\channelval \left(\sum_x q(\channelval \mid x) p(x)
    \right) \log \left(\sum_x q(\channelval \mid x) p(x) \right)
    - \sum_x p(x) \sum_\channelval q(\channelval \mid x)
    \log q(\channelval \mid x) \\
    \subjectto ~ & \sum_\channelval q(\channelval \mid x)
    = 1 ~ {\rm for~all~}x,
    ~~ \sum_\channelval \channelval q(\channelval \mid x)
    = x ~ {\rm for~all~}x,
    ~~ q(\channelval \mid x) \ge 0 ~ {\rm for~all~} x, \channelval.
    \nonumber
  \end{align}
  In the problem~\eqref{eqn:information-minimization}, the sums over
  $x$ and $\channelval$ are over the extreme points of $C$ and $D$,
  respectively and $p$ is the uniform distribution with $p(x) =
  1/m$. Mutual information is convex in the conditional distribution
  $q$; moreover,
  it is strictly convex except when
  $q(\channelval \mid x) = \sum_{x'} q(\channelval \mid x') p(x')$ for
  all $x, \channelval$. (This can be seen by
  an inspection of the proof of Theorem~2.7.4 by~\citet{CoverTh06}.)
  In our case, since $\channelprob^*$ does not
  satisfy this equality, the uniqueness of $\channelprob^*$ as the
  minimizer of $\information(\statprob^*, \channelprob^*)$ will follow
  if we show that $\channelprob^*$ is a minimizer at all.

  We proceed to solve the
  problem~\eqref{eqn:information-minimization}. Writing $\information(p, q)$
  as a shorthand for the mutual information, we introduce Lagrange multiplers
  $\theta(x) \in \R$ for the normalization constraints, $\mu(x) \in \R^d$ for
  the conditional expectation constraints, and $\lambda(x, \channelval) \ge 0$
  for the nonnegativity constraints. This yields the Lagrangian
  \ifdefined\usejournalmargins
  \begin{align*}
    \lefteqn{\mc{L}(q, \mu, \lambda, \theta)} \\
    & = \information(p, q)
    - \sum_{x, \channelval} \lambda(x, \channelval) q(\channelval \mid x)
    + \sum_x \mu(x)^\top \!
    \bigg(\sum_\channelval \channelval q(\channelval \mid x) - x\bigg)
    + \sum_x \theta(x) \bigg(\sum_\channelval q(\channelval \mid x) - 1\bigg).
  \end{align*}
  \else
  \begin{equation*}
    \mc{L}(q, \mu, \lambda, \theta)
    = \information(p, q)
    - \sum_{x, \channelval} \lambda(x, \channelval) q(\channelval \mid x)
    + \sum_x \mu(x)^\top
    \bigg(\sum_\channelval \channelval q(\channelval \mid x) - x\bigg)
    + \sum_x \theta(x) \bigg(\sum_\channelval q(\channelval \mid x) - 1\bigg).
  \end{equation*}
  \fi
  If we can satisfy the Karush-Kuhn-Tucker (KKT) conditions (see,
  e.g.,~\cite{BoydVa04}) for optimality of the
  problem~\eqref{eqn:information-minimization}, we will be done. Taking
  derivatives with respect to $q(\channelval \mid x)$, we see
  \ifdefined\usejournalmargins
  \begin{align*}
    \lefteqn{\frac{\partial}{\partial q(\channelval \mid x)}
      \mc{L}(q, \mu, \lambda, \theta)
      = p(x)\left[\log(q(\channelval \mid x)) + 1\right]
      - p(x) \log\bigg(\sum_{x'} q(\channelval \mid x') p(x')\bigg)} \\
    & \qquad\qquad\qquad\qquad\qquad\qquad\qquad
    ~ - q(\channelval) \cdot \frac{1}{q(\channelval)} p(x)
    - \lambda(\channelval, x) + \theta(x) +
    \mu(x)^\top \channelval \\
    & \qquad\quad = p(x) \log q(\channelval \mid x)
    - p(x) \log\bigg(\sum_{x'} q(\channelval \mid x') p(x')\bigg)
    - \lambda(\channelval, x) + \theta(x) + \mu(x)^\top \channelval,
  \end{align*}
  \else
  \begin{align*}
    \frac{\partial}{\partial q(\channelval \mid x)}
    \mc{L}(q, \mu, \lambda, \theta)
    & = p(x)\left[\log(q(\channelval \mid x)) + 1\right]
    - p(x) \log\bigg(\sum_{x'} q(\channelval \mid x') p(x')\bigg) \\
    & \quad ~ - q(\channelval) \cdot \frac{1}{q(\channelval)} p(x)
    - \lambda(\channelval, x) + \theta(x) +
    \mu(x)^\top \channelval \\
    & = p(x) \log q(\channelval \mid x)
    - p(x) \log\bigg(\sum_{x'} q(\channelval \mid x') p(x')\bigg)
    - \lambda(\channelval, x) + \theta(x) + \mu(x)^\top \channelval,
  \end{align*}
  \fi
  where we set $q(\channelval) = \sum_{x'} q(\channelval \mid x')
  p(x')$ for shorthand.  Now, we use symmetry to note that since we
  have chosen $q$ to be the maximum entropy
  distribution~\eqref{eqn:max-ent-q} for each $x$ in the extreme
  points $\{u_i\}$ of $C$, the marginal $q(\channelval) = \sum_{x'}
  q(\channelval \mid x') p(x') = 1/m$ is uniform by the symmetry of
  the set $D$ and since $p$ is uniform.  In addition, since
  $q(\channelval \mid x) > 0$ strictly, we have $\lambda(\channelval,
  x) = 0$ by complementarity. Thus, at $q$ chosen to be the maximum
  entropy distribution, we can rewrite the derivative of the
  Lagrangian
  \begin{equation*}
    \frac{\partial}{\partial q(\channelval \mid x)}
    \mc{L}(q, \mu, \lambda, \theta)
    = \frac{1}{m} \log q(\channelval \mid x) - \frac{1}{m} \log \frac{1}{m}
    + \theta(x) + \mu(x)^\top \channelval.
  \end{equation*}
  Recalling the definition~\eqref{eqn:max-ent-q} of $q(\channelval \mid x)$,
  and denoting the maximum entropy parameters $\mu$ there by $\mu^*(x)$, we
  have
  \begin{equation*}
    \frac{\partial}{\partial q(\channelval \mid x)}
    \mc{L}(q, \mu, \lambda, \theta)
    = -\frac{1}{m} \mu^*(x)^\top \channelval + \frac{1}{m}
    \log\left(\sum_{\channelval'} \exp(-\mu^*(x)^\top \channelval')\right)
    - \frac{1}{m} \log \frac{1}{m} + \theta(x) + \mu(x)^\top \channelval.
  \end{equation*}
  Now, by inspection we may set
  \begin{equation*}
    \theta(x) = \frac{1}{m} \log \frac{1}{m}
    - \frac{1}{m} \log\left(\sum_{\channelval'}
    \exp(-\mu^*(x)^\top \channelval')\right)
    ~~~ \mbox{and} ~~~
    \mu(x) = \frac{1}{m} \mu^*(x),
  \end{equation*}
  and we satisfy the KKT conditions for the mutual information minimization
  problem~\eqref{eqn:information-minimization}.

  Summarizing, the conditional distribution $\channelprob^*$ specified in the
  statement of the theorem as the maximum entropy
  distribution~\eqref{eqn:max-ent-q} satisfies
  \begin{equation*}
    \inf_\channelprob \information(\statprob^*, \channelprob)
    \ge \information(\statprob^*, \channelprob^*),
  \end{equation*}
  which, when combined with the first part of the proof, gives the saddle
  point inequality
  \begin{equation*}
    \sup_\statprob \information(\statprob, \channelprob^*)
    \le \log(m) - H(q(\cdot \mid \statrv = u_i))
    = \information(\statprob^*, \channelprob^*)
    \le \inf_\channelprob \information(\statprob^*, \channelprob),
  \end{equation*}
as claimed.

\paragraph{Remarks}
In the proof of the theorem, we have defined $\channelprob^*(\cdot
\mid x)$ as a conditional distribution only for $x \in \extreme(C)$,
the extreme points of $C$. This can easily be remedied: take
$\channelprob^*(\cdot \mid x)$ to be the distribution maximizing the
entropy $H(\channelrv \mid \statrv = x)$ for each $x \in C$ under the
constraint that the support of $\channelrv$ be contained in
$\extreme(D)$.  This is equivalent to---for each $x \in C$---choosing
$\channelrv = z_i$ for $z_i \in \extreme(D)$, $i = 1, \ldots, m$, with
probability $q_i$, where $q \in \R^m$ solves the entropy maximization
problem
\begin{equation*}
  \maximize_{q \in \R^m} ~ -\sum_i q_i \log q_i
  ~~ \subjectto \sum_i z_i q_i = x,
  ~ \sum_i q_i = 1, ~ q_i \ge 0.
\end{equation*}
Inspecting the proof of Theorem~\ref{theorem:finite-rotation-information} (see
the bound~\eqref{eqn:entropy-convexity}) shows that this choice can only
decrease the mutual information $\information(\statrv; \channelrv)$.
Additionally, the strong convexity of the entropy over the simplex guarantees
that the solutions to this optimization problem are continuous
in $x$ (see Chapter~X of~\citet{HiriartUrrutyLe96ab}) so this
distribution $q(\cdot \mid x)$ defines a measurable random variable as
desired.


\subsection{Proof of Proposition~\ref{proposition:linf-saddle-point}}
\label{appendix:linf-saddle-point-proof}

By scaling, we may assume w.l.o.g.\ that $L = 1$ and $M \ge 1$.
Using Theorem~\ref{theorem:finite-rotation-information} (and the
remarks immediately following its proof), we can focus on maximizing
the entropy of the random variable $\channelrv$ conditional on
$\statrv = x$ for each fixed $x \in [-1, 1]^d$. Let $\channelrv_i$
denote the $i$th coordinate of the random vector $\channelrv$; we take
the conditional distribution of $\channelrv_i$ to be independent of
$\channelrv_j$ and let $\channelrv$ be distributed as
\begin{align}
  \label{eqn:z-given-x}
  \channelrv_i \mid \statrv & =
\begin{cases}
 M & \mbox{w.p.}~ \half + \frac{\statrv_i}{2M} \\
-M & \mbox{w.p.} ~
  \half - \frac{\statrv_i}{2M}. 
\end{cases}
\end{align}
Let us now verify that the distribution~\eqref{eqn:z-given-x} maximizes the
entropy $H(\channelrv \mid \statrv = x)$. Indeed, we may fix $x$ (leaving it
implicit in the vector $[q(z)]_z \defeq [q(z \mid x)]_z$), and
we solve the entropy maximization problem
\begin{equation}
  \label{eqn:discrete-problem}
  \minimize_q ~ - \! H(q) ~~~ \subjectto ~ \sum_\channelval
  q(\channelval) = 1, ~ q(\channelval) \ge 0, ~ \sum_\channelval
  \channelval q(\channelval) = x,
\end{equation}
where all sums are taken over $\channelval \in \extreme([-M, M]^d) =
\{-M, M\}^d$.  Introducing the Lagrange multipliers $\mu
\in \R^d$, $\lambda(\channelval) \ge 0$, and $\theta \in \R$, we find
that problem~\eqref{eqn:discrete-problem} has the Lagrangian
\begin{equation*}
  \mc{L}(q, \mu, \lambda, \theta)
  = - H(q)
  - \sum_\channelval \lambda(\channelval) q(\channelval)
  + \mu^\top \bigg(\sum_\channelval \channelval q(\channelval) - x\bigg)
  + \theta \bigg(\sum_\channelval q(\channelval) - 1\bigg).
\end{equation*}
To find the infimum of the Lagrangian with respect to $q$, we take
derivatives (since we make the identification $q \in \R^{2^d}$). We see that
\begin{align*}
  \frac{\partial}{\partial q(\channelval)} \mc{L}(q, \mu, \lambda, \theta)
  & = \log(q(\channelval)) + 1
  - \lambda(\channelval) + \theta +
  \mu^\top \channelval.
\end{align*}

With the definition~\eqref{eqn:z-given-x} of the probability mass function
$q$ (that $\channelval_i$ are independent Bernoulli random variables with parameters
$\half + x_i / 2M$), the coordinate conditional distributions are
\begin{equation*}
  q(\channelval_i \mid x_i) = \left(\half + \frac{1}{2M}\right)^{\half
    + \frac{x_i \channelval_i}{2M}} \left(\half -
  \frac{1}{2M}\right)^{\half - \frac{x_i \channelval_i}{2M}}.
\end{equation*}
Theorem~\ref{theorem:finite-rotation-information} says that without
loss of generality we may assume that $x \in \{-1, 1\}^d$, the full
probability mass function $q$ can be written
\begin{equation}
  \label{eqn:hypothesized-conditional}
  q(\channelval) = \left(\half + \frac{1}{2M}\right)^{\frac{d}{2} +
    \frac{x^\top \channelval}{2M}} \left(\half -
  \frac{1}{2M}\right)^{\frac{d}{2} - \frac{x^\top \channelval}{2M}}.
\end{equation}
Plugging the conditional~\eqref{eqn:hypothesized-conditional} results
in
\begin{align*}
  \lefteqn{\frac{\partial}{\partial q(\channelval)} \mc{L} (q, \mu,
    \lambda, \theta)} \\ & = \left(\frac{d}{2} + \frac{x^\top
    \channelval}{2M}\right) \log\left(\half + \frac{1}{2M}\right) +
  \left(\frac{d}{2} - \frac{x^\top \channelval}{2M}\right)
  \log\left(\half - \frac{1}{2M}\right) + 1 - \lambda(\channelval) +
  \theta + \mu^\top \channelval \\ & =
  \frac{d}{2}\left[\log\left(\half + \frac{1}{2M}\right) +
    \log\left(\half - \frac{1}{2M}\right)\right] + \frac{x^\top
    \channelval}{2M}\left[\log\left(\half + \frac{1}{2M}\right) -
    \log\left(\half - \frac{1}{2M}\right)\right] \\ & \qquad ~ + 1 -
  \lambda(\channelval) + \theta + \mu^\top \channelval.
\end{align*}
Performing a few algebraic manipulations with the logarithmic terms,
the final equality becomes
\begin{equation*}
  d \log\left(\frac{\sqrt{(M + 1)(M - 1)}}{M}\right)
  + \frac{x^\top \channelval}{M} \log\left(\sqrt{\frac{M + 1}{M - 1}}\right)
  + 1 - \lambda(\channelval) + \theta + \mu^\top \channelval.
\end{equation*}
The complementarity conditions for optimality~\cite{BoydVa04} imply
that $\lambda(\channelval) = 0$, and since the equality constraints in
the problem~\eqref{eqn:discrete-problem} are satisfied, we can choose
$\theta$ and $\mu$ arbitrarily. Taking
\begin{equation*}
  \theta = -d \log\left(\frac{\sqrt{(M + 1)(M - 1)}}{M}\right) - 1
  ~~~ \mbox{and} ~~~
  \mu = -x\frac{1}{M} \log\left(\sqrt{\frac{M + 1}{M - 1}}\right)
\end{equation*}
yields that the partial derivatives of $\mc{L}$ are 0, which shows that
indeed our choice of $\channelprob^*$ is optimal.


\subsection{Proof of Proposition~\ref{proposition:lone-saddle-point}}
\label{appendix:lone-saddle-point-proof}

The proof follows along lines similar to the $\ell_\infty$ case: we
compute the maximum entropy distribution subject to the constraint
that $\E[\channelrv] = x$ for some $x \in \R^d$ with $\lone{x} \le 1$,
and $\channelrv$ must be supported on the extreme points $\pm M e_i$
of the $\ell_1$-ball of radius $M$. (Recall that $e_i \in \R^d$ are
the standard basis vectors.)  Based on
Theorem~\ref{theorem:finite-rotation-information}, in order to find
the minimax mutual information, we need only consider the cases where
$x = \pm e_i$ for some $i \in \{1, \ldots, d\}$.

Following this plan, we recall the entropy maximization
problem~\eqref{eqn:discrete-problem}, where now $x = \pm e_i$ and
the sums are over $\channelval \in M \{\pm e_i\}_{i=1}^d$. As in the proof
of Proposition~\ref{proposition:linf-saddle-point}, we can write the
Lagrangian and take its derivatives, finding that for $\channelval = \pm M e_i$
we have
\begin{equation*}
  \frac{\partial}{\partial q(\channelval)} \mc{L}(q, \mu, \lambda,
  \theta) = \log(q(\channelval)) + 1 - \lambda(\channelval) + \theta -
  \mu^\top \channelval.
\end{equation*}
Solving for $q(\channelval)$, we find that
\begin{align*}
  q(\channelval) & = \exp(\lambda(\channelval) - 1 - \theta)
  \exp(\mu^\top \channelval),
\end{align*}
but complementarity~\cite{BoydVa04} guarantees
that $\lambda(\channelval) = 0$ since $q(\channelval) > 0$, and
normalizing we may write $q(\channelval) = \exp(-\mu^\top \channelval)
/ \exp(-\mu^\top \sum_{\channelval'} \channelval')$, where the sum is
over the extreme points of the $\ell_1$-ball of radius $M$.  In
particular, $q(Me_i) \propto e^{-\mu_i}$ and $q(-Me_i) \propto
e^{\mu_i}$.  Without loss of generality, let $x = e_i$. Symmetry
suggests we take (and we verify this to be true)
\begin{align}
  \label{eqn:symmetric-l1-q}
  q(\channelval) & = \exp(-1 - \theta) \begin{cases}
  \exp(\mu_i) & {\rm if~} \channelval = M e_i \\ \exp(-\mu_i) & {\rm
    if~} \channelval = -M e_i \\ \exp(0) & {\rm otherwise}.
  \end{cases}
\end{align}
Indeed, with the choice~\eqref{eqn:symmetric-l1-q} of $q$, we have
$q(M e_j) - q(-M e_j) = 0$ for $j \neq i$, while (setting $\gamma =
\mu_i$ and normalizing appropriately)
\begin{equation*}
  q(M e_i) - q(-M e_i) = \frac{e^{\gamma}}{e^{-\gamma} + e^\gamma + 2(d - 1)}
  - \frac{e^{-\gamma}}{e^{-\gamma} + e^\gamma + 2(d - 1)}.
\end{equation*}
Thus, if we can solve the equation $Mq(Me_i) - Mq(-M e_i) = 1$, we will be
nearly done. To that end, we write
\begin{equation*}
  \frac{e^\gamma - e^{-\gamma}}{e^\gamma + e^{-\gamma} + 2 (d - 1)}
  = \frac{1}{M}
  ~~~ \mbox{or} ~~~
  \beta - \beta^{-1}
  = \frac{1}{M} \left(\beta + \beta^{-1} + 2(d - 1)\right),
\end{equation*}
where we identified $\beta = e^\gamma$. Multiplying both sides by $\beta$,
we have a quadratic equation in $\beta$:
\begin{equation*}
  \beta^2 - 1 = \frac{1}{M} \left(\beta^2 + 2\beta(d - 1) + 1\right)
  ~~~ \mbox{or} ~~~
  (M - 1)\beta^2 - 2 (d - 1) \beta - (M + 1) = 0,
\end{equation*}
whose solution is the positive root of
\begin{equation*}
  \beta = \frac{2d - 2 \pm \sqrt{(2d - 2)^2 + 4(M^2 - 1)}}{2(M - 1)}
  ~~~ \mbox{or} ~~~
  \gamma = \log\left(\frac{2d - 2 + \sqrt{(2d - 2)^2 + 4(M^2 - 1)}}{2(M - 1)}
  \right).
\end{equation*}
By our construction, with $\gamma$ so defined, we satisfy the constraints
that $M\left[q(M e_i) - q(-M e_i)\right] = 1$ and $q(M e_j) - q(-M e_j) = 0$
for $j \neq i$. Since $q$ belongs to the exponential family and satisfies
the constraints, it maximizes the entropy $H(\channelrv)$ as
desired~\cite{CoverTh06}.

Algebraic manipulations and the computation of the conditional entropy
$H(\channelrv \mid \statrv = e_i)$ give the remainder of the statement of
the proposition.

\subsection{Proof of Proposition~\ref{proposition:linf-diffp-saddle-point}}
\label{appendix:proof-linf-diffp-saddle-point}

The outline of the proof of
Proposition~\ref{proposition:linf-diffp-saddle-point} is as
follows. Lemma~\ref{lemma:diffp-extreme-points} implies that any distribution
satisfying \olpd\ must be supported on the extreme points of the outer set $D$
(as in the proof of Theorem~\ref{theorem:finite-rotation-information}).  Given
this result, we reduce the problem of finding an optimally private
distribution to a linear program, using symmetry arguments to simplify the
LP. Finally, we show that the solution to the linear program is unique, which
means that we have found the unique distribution satisfying \olpd.

We begin by developing a reduction of the problem of finding a distribution 
with \olpd\ to a linear program.  Note that there is a non-increasing 
mapping between $M$---the radius of the larger $\ell_\infty$
ball---and $\optdiffp$. Indeed, whenever $M$ increases, the set of
distributions $\channelprob$ from which to choose a privacy channel increases,
so $\optdiffp$ decreases. Put inversely,
for a given differential privacy level $\diffp$, we can find the smallest
$M$ such that it is possible to construct an $\diffp$-differentially private
channel $\channelprob$ mapping from $[-1, 1]^d$ to $[-M, M]^d$.
(Lemma~\ref{lemma:two-level-diffp-solution} shows that
the mapping from $M$ to $\optdiffp$ is implicitly invertible.)

Thus, rather than solving for $\diffp$ as a function of $M$, we take the
converse view of finding the largest $M$ such that an $\diffp$-differentially
private distribution exists. Fix $d \in \N$ and (with some abuse of notation)
let $Z \in \{-1, 1\}^{d \times 2^d}$ be the matrix whose columns are the edges
of the hypercube $\{-1, 1\}^d$. For each $z, x \in \{-1, 1\}^d$, define the
variables $q(z \mid x) \ge 0$ to represent the conditional probability of
observing $Mz$ given $x$.  Let $q(\cdot \mid x) = [q(z \mid x)]_{z \in
  \{-1,1\}^d}$ denote the vector version of $q(z \mid x)$.  Then we have that
a $\diffp$-differentially private channel providing an unbiased perturbtation
of vectors in $[-1, 1]^d$ to $[-M, M]^d$, exists only if we can find settings
of $q(z \mid x)$ such that
\begin{equation*}
  Z q(\cdot \mid x) - \frac{1}{M} x = 0
  ~ \mbox{for~all~}x \in \{-1, 1\}^d
\end{equation*}
while additionally $q(z \mid x) \le e^\diffp q(z \mid x')$
and $\sum_z q(z \mid x) = 1$, $q(z \mid x) \ge 0$ for all $z, x, x'$.
Thus, if we make the change of variables $t = 1/M$, we see that
finding the smallest possible $M$---which corresponds to the
least perturbation possible for a given privacy level $\diffp$---can
be cast as solving the linear program
\begin{align}
  \minimize ~~ & -t \label{eqn:matrix-diffp-lp} \\
  \subjectto ~~ & Z q(\cdot \mid x) - t x = 0
  ~ \mbox{for~all~} x \in \{-1, 1\}^d
  \nonumber \\
  & q(z \mid x) \le e^\diffp q(z \mid x') ~
  \mbox{for~all~} x, x', z \in \{-1, 1\}^d \nonumber \\
  & \sum_z q(z \mid x) = 1, ~ q(\cdot \mid x) \succeq 0
  ~ \mbox{for~all~} x \in \{-1, 1\}^d. \nonumber
\end{align}
The solution vectors $q(\cdot \mid x)$, $x \in \{-1, 1\}^d$, give the
probability mass function for an $\diffp$-differentially private channel
perturbing from $[-1, 1]^d$ to $[-M, M]^d$, where $M = 1 / t^*$ and $t^*$
denotes the solution to the LP. This p.m.f.\ is then optimally locally
differentially private with $\diffp = \optdiffp([-1, 1]^d, [-M, M]^d)$.

It is possible to calculate the solution of the LP~\eqref{eqn:matrix-diffp-lp}
by hand, but it is tedious. We thus use the structure of \olpd\ to reduce the
problem to a single minimization problem over a vector $q \in \R^{2^d}$
(rather than a matrix $[q(z \mid x)] \in \R^{2^d \times 2^d}$). We have
\begin{lemma}
  \label{lemma:symmetry-lp-solutions}
  A distribution satisfying \olpd\ must, for each $x \in \{-1, 1\}^d$, satisfy
  $q(\cdot \mid x) = \Pi(x) q$, where $\Pi(x) \in \{0, 1\}^{2^d \times 2^d}$
  is a permutation matrix and $q$ is a fixed vector.
\end{lemma}
\begin{proof}
  Suppose for the sake of contradiction that this is not the case, but the
  vectors $q(x)$ and $t$ solve the linear
  program~\eqref{eqn:matrix-diffp-lp}. Let $Q_1$ denote the matrix of the
  vectors $q(\cdot \mid x)$.  Choose vectors $q(\cdot \mid x)$ and $q(\cdot
  \mid x')$ such that $q(\cdot \mid x) \neq \Pi q(\cdot \mid x')$ for any
  permutation matrix $\Pi$. Now construct vectors $q_2(\cdot \mid x)$ and
  $q_2(\cdot \mid x')$ such that $q_2(z \mid x) = q(z' \mid x')$, where $z'$
  is chosen so that $z_i' x_i' = z_i x_i$, and similarly choose $q_2$ so that
  $q_2(z \mid x') = q(z' \mid x)$, where $z_i x_i' = z_i' x_i$. Let $Q_2$
  denote the matrix of vectors $q$, but where $q_2(\cdot \mid x)$ and
  $q_2(\cdot \mid x')$ replace $q(\cdot \mid x), q(\cdot \mid x')$. Then by
  construction, all the constraints of the original linear
  program~\eqref{eqn:matrix-diffp-lp} are satisfied. By symmetry and the
  strict convexity of the mutual information in the channel distribution
  $\channelprob$, however, we see that
  \begin{equation*}
    \information(\statprob, \channelprob_1)
    = \information(\statprob, \channelprob_2)
    = \half\left(\information(\statprob, \channelprob_1)
    + \information(\statprob, \channelprob_2)\right)
    > \information\left(\statprob,
    \half(\channelprob_1 + \channelprob_2)\right).
  \end{equation*}
  The decrease in mutual information gives the necessary contradiction.
\end{proof}

With Lemma~\ref{lemma:symmetry-lp-solutions} in hand, we can now turn to the
smaller linear program---in a single vector $q$ and for a single vector $x \in
\{-1, 1\}^d$---that will give us the locally optimal differentially private
channel. Indeed, we consider the linear program in the variables $t \in \R$
and $\channeldensity \in \R^{2^d}$, where
we let $\channeldensity(z)$ denote the
entry of $\channeldensity$ corresponding to column $\channelval$ of $Z$:
\begin{equation}
  \label{eqn:single-diffp-lp}
  \begin{split}
    \minimize ~~ & -t \\
    \subjectto ~~ & Z \channeldensity - t \statsample = 0,
    ~~ \channeldensity(\channelval) \le e^\diffp
    \channeldensity(\channelval') ~ \mbox{for~all~} \channelval, \channelval',
    ~~
    \sum_{\channelval} \channeldensity(\channelval) = 1,
    ~ \channeldensity \ge 0.
  \end{split}
\end{equation}
Define the constants
\ifdefined\usejournalmargins
\begin{equation*}
  K_d = \sum_{i = 0}^{\floor{d/2}} (d - 2 i) \binom{d}{i}
  ~~ \mbox{and} ~~
  C_d = \card\left\{z \in \{-1, 1\}^d : z^\top x > 0\right\}
  = \begin{cases}
  2^{d - 1} & d ~\mbox{odd} \\
  2^{d - 1} - \half \binom{d}{d/2} & d ~ \mbox{even}.
  \end{cases}
\end{equation*}
\else
\begin{equation*}
  K_d = \sum_{i = 0}^{\floor{d/2}} (d - 2 i) \binom{d}{i}
  ~~~ \mbox{and} ~~~
  C_d = \card\left\{z \in \{-1, 1\}^d : z^\top x > 0\right\}
  = \begin{cases}
  2^{d - 1} & \mbox{if~} d ~\mbox{odd} \\
  2^{d - 1} - \half \binom{d}{d/2} & \mbox{if~} d ~ \mbox{even}.
  \end{cases}
\end{equation*}
\fi
We have the following lemma, which characterizes the structure of the
solution vector $q$.
\begin{lemma}
  \label{lemma:two-level-diffp-solution}
  Define $\diffp^* = \log\frac{K_d + 2^d - C_d}{K_d - C_d}$. For any
  $\diffp < \diffp^*$, the unique solution to the linear
  program~\eqref{eqn:single-diffp-lp} is given by
  \begin{equation*}
    q(z) = \begin{cases} \frac{e^\diffp}{e^\diffp C_d + 2^d - C_d}
      & \mbox{if}~ \<z, x\> > 0 \\
      \frac{1}{e^\diffp C_d + 2^d - C_d}
      & \mbox{otherwise.}
    \end{cases}
  \end{equation*}
\end{lemma}
\begin{proof}
  First, problem~\eqref{eqn:single-diffp-lp} is clearly equivalent to
  the linear program
  \begin{equation}
    \label{eqn:single-diffp-lp-max}
    \begin{split}
      \minimize ~~ & -t \\
      \subjectto ~~ & Z \channeldensity - t \statsample = 0 
      , ~~ \max_{\channelval} \{ \channeldensity(\channelval) \}
      + e^\diffp \max_{\channelval} \{ -\channeldensity(\channelval) \} \le 0
      , ~~ \sum_{\channelval} \channeldensity(\channelval) = 1,
      ~ \channeldensity \ge 0.
    \end{split}
  \end{equation}
  Our proof proceeds in two large steps: first, we argue that a $q$
  of the form specified in the lemma is indeed the solution to the
  problem~\eqref{eqn:single-diffp-lp-max}, then we use results on
  uniqueness of solutions to linear programs due to~\citet{Mangasarian79}.

  For the first step, we begin by writing the Lagrangian to the
  problem~\eqref{eqn:single-diffp-lp-max}. We introduce dual variables $\theta
  \in \R^{2^d}$ for the constraint $Zq - tx = 0$, $\lambda \ge 0$ for the
  first inequality, $\tau \in \R$ for the sum constraint, and $\beta \in
  \R_+^{2^d}$ for the non-negativity of $q$. With this, we have Lagrangian
  \begin{equation}
    \label{eqn:lp-lagrangian}
    \mc{L}(q, t, \theta, \lambda, \tau, \beta)
    = -t + \theta^\top\left(\sum_z q(z) - tx \right)
    + \lambda \max_z\{q(z)\} + e^\diffp \max_z \{-q(z)\}
    + \tau(\onevec^\top q - 1) - \beta^\top q.
  \end{equation}
  Recall the generalized subgradient KKT conditions for optimality of the
  solution to an optimization problem~\cite[Chapter
    VII]{HiriartUrrutyLe96ab}. A vector $q > 0$ is optimal for the
  problem~\eqref{eqn:single-diffp-lp-max} if the constraints $\max_i \{q_i\}
  \le e^\diffp \min_i \{q_i\}$ and $\sum_i q_i = 1$ hold, there is a $t \ge 0$
  such that $Zq - tx = 0$, and we can find $\theta$, $\lambda$, and $\tau$
  such that
  \begin{equation}
    Z^\top \theta + \lambda\left[v_+
      - e^\diffp v_-\right]
    + \tau \onevec = 0,
    ~~ \beta = 0, ~~
    ~~ \mbox{and} ~~
    \theta^\top x = -1,
    \label{eqn:kkt-diffp-lp}
  \end{equation}
  where $v_+$ and $v_-$ are vectors satisfying
  \begin{equation*}
    v_+ \in \conv\Big\{e_i : q_i = \max_j\{q_j\}\Big\}
    ~~~ \mbox{and} ~~~
    v_- \in \conv\Big\{e_i : q_i = \min_j \{q_j\}\Big\}.
  \end{equation*}
  That $\beta = 0$ follows by complementarity (recall that $q > 0$ is assumed).

  If we can find settings for the vectors $\theta, \lambda, \tau,$ and
  $v_{\pm}$ satisfying the KKT conditions~\eqref{eqn:kkt-diffp-lp}, we are
  done.  To that end, set $\theta = -x / d$. Then by inspection $\theta^\top x
  = -\ltwo{x}^2 / d = -1$, and we can rewrite the remaining KKT condition by
  noting that we must find vectors $v_+$, $v_-$, and $\tau \in \R$ such that
  \begin{equation}
    \label{eqn:kkt-vs}
    \begin{split}
      & -\frac{1}{d} Z^\top x + v_+ - e^\diffp v_-
      + \tau \onevec = 0,
      ~~~ v_+^\top \onevec = v_-^\top \onevec, ~~~
      v_+ \ge 0, v_- \ge 0, \\
      & v_+(z) = 0 ~ \mbox{if~} q(z) < \max_z\{q(z)\},
      ~~~ \mbox{and} ~~~
      v_-(z) = 0 ~\mbox{if~} q(z) > \min_z \{q(z)\}.
    \end{split}
  \end{equation}
  Note that we have eliminated $\lambda$ as it is a non-negative homogeneous
  scaling term on $v_+$ and $v_-$. We choose values
  $q^+, q^-$ with $0 < q^- < q^+$ and set $q(z) = q^+$ when
  $z^\top x > 0$ and $q(z) = q^-$ when $z^\top x \le 0$, where
  $q^+, q^-$ are chosen so that
  $\sum_z q(z) = 1$. We now choose the values
  of $v_+$, $v_-$, and $\tau$ satisfying the KKT conditions in
  expression~\eqref{eqn:kkt-vs} based on the values $q^+, q^-$. Indeed, set
  \begin{equation}
    \label{eqn:convex-hull-vectors}
    v_+(z) = \left\{\begin{array}{ll}
    \frac{z^\top x}{d} - \tau & \mbox{if~} z^\top x > 0 \\
    0 & \mbox{otherwise}\end{array}\right.
    ~~~ \mbox{and} ~~~
    v_-(z) = \left\{\begin{array}{ll}
    -e^{-\diffp} \frac{z^\top x}{d}
    + e^{-\diffp} \tau & \mbox{if~} z^\top x \le 0 \\
    0 & \mbox{otherwise}\end{array}\right.
  \end{equation}
  By inspection, we see that $-Z^\top x / d + v_+ - e^\diffp v_- + \tau = 0$,
  so the only question remaining is whether we can choose $\tau$ such that
  $v_{\pm} \ge 0$ and $v_+^\top \onevec = v_-^\top \onevec$.

  To that end, we recall the definition of the constant $K_d$, and we seek
  $\tau$ such that
  \begin{equation*}
    \sum_z v_+(z) = \frac{1}{d} K_d - \tau C_d
    = e^{-\diffp} \frac{1}{d} K_d + e^{-\diffp} \tau (2^d - C_d)
    = \sum_z v_-(z)
  \end{equation*}
  by the symmetry in the sums. Rewriting the equation, we find that for
  equality we must have
  \begin{equation*}
    \tau\left(e^{-\diffp}(2^d - C_d)
    + C_d\right)
    = \frac{1}{d} K_d (1 - e^{-\diffp})
    ~~~ \mbox{or} ~~~
    \tau = \frac{K_d}{d} \cdot \frac{e^\diffp - 1}{
      e^\diffp C_d + 2^d - C_d}
    = \frac{K_d}{d C_d} \cdot \frac{e^\diffp - 1}{
      e^\diffp + 2^d / C_d - 1}.
  \end{equation*}
  Thus we find that if $\diffp$ is such that
  \begin{equation}
    \frac{K_d}{d C_d} \frac{e^\diffp - 1}{e^\diffp + 2^d / C_d - 1}
    < \frac{1}{d},
    \label{eqn:alpha-inequality}
  \end{equation}
  then by our choice~\eqref{eqn:convex-hull-vectors} of the vectors
  $v_+$ and $v_-$, we have $v_+(z) > 0$ whenever $z^\top x > 0$,
  and $v_-(z) > 0$ whenever $z^\top x \le 0$. Noting that by our setting
  of $q(z)$, we have by symmetry of $Z$ that
  there exists a $t > 0$ such that $Zq = tx$, we find that our choice of $q$
  is optimal (since the KKT conditions~\eqref{eqn:kkt-vs} hold).

  We have two arguments remaining in the proof. The first is to show that for
  $\diffp < \diffp^*$ defined in the statement of the lemma, the
  inequality~\eqref{eqn:alpha-inequality} holds. Rewriting the inequality,
  we solve
  \begin{equation*}
    e^\diffp - 1 =
    \frac{C_d}{K_d} \left(e^\diffp + 2^d / C_d - 1\right)
    ~~ \mbox{or} ~~
    e^\diffp\left(1 - \frac{C_d}{K_d}\right)
    = \frac{2^d - C_d}{K_d} + 1,
    ~~ \mbox{i.e.} ~~
    \diffp^* = \log\frac{K_d + 2^d - C_d}{K_d - C_d}.
  \end{equation*}
  For any $\diffp < \diffp^*$, the strict
  inequality~\eqref{eqn:alpha-inequality} holds, so the
  setting~\eqref{eqn:convex-hull-vectors} of $v_+$ and $v_-$ satisfy the KKT
  conditions.

  Our last argument regards the uniqueness of the two-valued solution vector
  $q$. For that, we apply Mangasarian's result~\cite[Theorem 1]{Mangasarian79}
  that if there exists an $\epsilon > 0$ such that for any vector $u \in
  \R^{2^d}$ with $\ltwo{u} = 1$, $q$ is a solution of the linear
  program~\eqref{eqn:single-diffp-lp} when the objective is $-t + \epsilon
  u^\top q$, then $q$ is unique.  Luckily, this is not difficult given our
  previous work. The Lagrangian~\eqref{eqn:lp-lagrangian} for the modified
  linear program becomes
  \begin{equation*}
    \epsilon u^\top q - t + \theta^\top\left(\sum_z q(z) - tx \right)
    + \lambda \max_z\{q(z)\} + e^\diffp \max_z \{-q(z)\}
    + \tau(\onevec^\top q - 1) - \beta^\top q.
  \end{equation*}
  The only modification in our KKT conditions~\eqref{eqn:kkt-diffp-lp}
  is that the first equality becomes
  \begin{equation*}
    \epsilon u + Z^\top \theta + \lambda\left[v_+ - e^\diffp v_-\right]
    + \tau \onevec = 0.
  \end{equation*}
  By the strictness of the inequalities
  $v_+(z) > 0$ for $z$ such that $z^\top x > 0$ (and similarly for $v_-$)
  in the definitions~\eqref{eqn:convex-hull-vectors}
  whenever $\diffp < \diffp^*$, we see that for suitably small $\epsilon > 0$,
  the vectors $v_+$ and $v_-$ can be perturbed so that the KKT conditions
  are still satisfied. This proves the uniqueness of the two-valued solution
  vector $q$.
\end{proof}

\paragraph{Remarks}
Following an argument with completely the same structure as the proof, we see
that for any $d \in \N$ (say $d \ge 3$), there are different ``regimes'' of
$\diffp$, that is, there exists a sequence $\diffp_0^*, \diffp_2^*, \ldots,
\diffp_{d-1}^*$ (or $\diffp_{d-2}^*$ if $d$ is even) such that for $\diffp \in
(\diffp_{2i}^*, \diffp_{2i + 2}^*)$, the unique optimal solution to the linear
program~\eqref{eqn:single-diffp-lp} is given by taking
\begin{equation*}
  q(z) \propto \begin{cases}\exp(\diffp) & \mbox{for}~ z ~\mbox{s.t.}~
    \<z, x\> > 2(i + 1) \\
    1 & \mbox{for~} z ~\mbox{s.t.}~ \<z, x\> \le 2(i + 1)
  \end{cases}
\end{equation*}
(for $\diffp < \diffp_0^*$, we say $i = -1$ above).
For $\diffp = \diffp_{2i}^*$, the set of solutions is given by
the convex combinations of the solution vectors
\begin{equation*}
  q_<(z) \propto \begin{cases}\exp(\diffp) & \mbox{for}~ z ~\mbox{s.t.}~
    \<z, x\> > 2i \\
    1 & \mbox{for~} z ~\mbox{s.t.}~ \<z, x\> \le 2i
  \end{cases}
  ~~ \mbox{and} ~~
  q_>(z) \propto \begin{cases}\exp(\diffp) & \mbox{for}~ z ~\mbox{s.t.}~
    \<z, x\> > 2(i + 1) \\
    1 & \mbox{for~} z ~\mbox{s.t.}~ \<z, x\> \le 2(i + 1),
  \end{cases}
\end{equation*}
which follows from arguments similar to our application of Mangasarian's
results~\cite{Mangasarian79}.

Now we may complete the proof of
Proposition~\ref{proposition:linf-diffp-saddle-point}. Indeed, we see from
Lemma~\ref{lemma:two-level-diffp-solution} that the distribution satisfying
\olpd\ must assign probability masses at two levels---at least when the point
being perturbed comes from $\{-1, 1\}^d$. Now let $\channelprob$ be a
distribution specified in the lemma.  An argument identical to that in our
proof of Proposition~\ref{proposition:linf-saddle-point}---by symmetry---shows
that the distribution $\statprob$ maximizing the mutual information
$\information(\statprob, \channelprob)$ is uniform on $\{-1, 1\}^d$. The
uniqueness of $\channelprob$ then follows from
Lemmas~\ref{lemma:symmetry-lp-solutions}
and~\ref{lemma:two-level-diffp-solution}, which show that such $\channelprob$
is the only distribution that minimizes the radius $M$ of the ball
$[-M, M]^d$; inverting this bound gives the proposition.
\qed




%% file: lone-asymptotic-expansion.tex
\subsection{Proof of Corollary~\ref{corollary:lone-asymptotic-expansion}}
\label{appendix:lone-asymptotic-expansion}

By scaling, we may assume that $M \ge L = 1$ in the proof of the corollary.
First, we claim that as $\gamma \rightarrow 0$, the following
expansion holds:
\begin{equation}
  \log(2d) - \log\left(e^\gamma + e^{-\gamma}
  + 2d - 2\right)
  + \gamma \frac{e^\gamma}{e^\gamma + e^{-\gamma} + 2d - 2}
  - \gamma \frac{e^{-\gamma}}{e^\gamma + e^{-\gamma} + 2d - 2}
  = \frac{\gamma^2}{2d} + \Theta\left(\frac{\gamma^4}{d}\right).
  \label{eqn:lone-asymptotic-expansion}
\end{equation}
Before proving this, we use the
expansion~\eqref{eqn:lone-asymptotic-expansion} to prove
Corollary~\ref{corollary:lone-asymptotic-expansion}. Noting that
\begin{equation*}
  \frac{2d - 2 + \sqrt{(2d - 2)^2 + 4(M^2 - 1)}}{2(M - 1)}
  = \sqrt{\frac{M + 1}{M - 1}}
  + \frac{d - 1}{M - 1} + \Theta(d^2 / M^2),
\end{equation*}
we see that since $\log(1 + x) = x - x^2/2 + \Theta(x^3)$, we have
  $\gamma = \frac{d}{M} + \Theta\left(\frac{d^2}{M^2}\right).$
Thus the mutual information in
Proposition~\ref{proposition:lone-saddle-point} is
\begin{align*}
  \information(\statprob^*, \channelprob^*)
  & = \frac{\log^2(\sqrt{(M + 1) / (M - 1)} + d / M + \Theta(d^2/M^2))}{2d}
  + \Theta\left(\frac{\log^4(1 + d / M)}{d}\right) \nonumber \\
  & = \frac{d}{2M^2} + \Theta\left(\min\left\{\frac{d^3}{M^4},
  \frac{\log^4(d)}{d}\right\}\right).
\end{align*}

Now we return to showing the
claim~\eqref{eqn:lone-asymptotic-expansion}. Indeed, define
$f(\gamma) = \log(e^\gamma + e^{-\gamma} + 2d - 2)$. Letting $f^{(i)}$
denote the $i$th derivative of $f$, we have
\begin{equation*}
  f^{(1)}(\gamma) =
  \frac{e^\gamma - e^{-\gamma}}{e^\gamma + e^{-\gamma} + 2d - 2},
  ~~~~
  f^{(2)}(\gamma) =
  \frac{(e^\gamma + e^{-\gamma})(2d - 2) + 4}{
    (e^\gamma + e^{-\gamma} + 2d - 2)^2},
\end{equation*}
and
\begin{equation*}
  f^{(3)}(\gamma) =
  \frac{-(e^{2\gamma} - e^{-2\gamma})(2d - 2)
    - 8(e^\gamma - e^{-\gamma}) + (2d - 2)^2(e^\gamma - e^{-\gamma})}{
    (e^\gamma + e^{-\gamma} + 2d - 2)^3}.
\end{equation*}
Via a Taylor expansion, we have
$f(0) = f(\gamma) - \gamma f^{(1)}(\gamma)
+ \frac{\gamma^2}{2} f^{(2)}(\gamma)
+ \order(f^{(3)}(\gamma) \gamma^3)$, and so
substituting values for $f(\gamma)$ and $f^{(1)}(\gamma)$, we have
\begin{align*}
  & \log(2d) - \log\left(e^\gamma + e^{-\gamma}
  + 2d - 2\right)
  + \gamma \frac{e^\gamma}{e^\gamma + e^{-\gamma} + 2d - 2}
  - \gamma \frac{e^{-\gamma}}{e^\gamma + e^{-\gamma} + 2d - 2} \\
  & \qquad\qquad ~ = \frac{(e^\gamma + e^{-\gamma})(2d - 2) + 4}{
    (e^\gamma + e^{-\gamma} + 2d - 2)^2} \cdot \frac{\gamma^2}{2} 
  + \order\left(f^{(3)}(\gamma) \gamma^3\right).
\end{align*}
A few simpler Taylor expansions yield that $f^{(3)}(\gamma) = \order(\gamma
/ d)$, which means that all we have left to tackle is $f^{(2)}(\gamma)$.
But noting that
\begin{equation*}
  2 \left(e^\gamma + e^{-\gamma}\right)
  = 4 \left(1 + \frac{\gamma^2}{2!} + \frac{\gamma^4}{4!} + \cdots\right)
  = 4 + \order(\gamma^2)
\end{equation*}
implies that $f^{(2)}(\gamma) = (4d + \order(d \gamma^2)) / 4d^2$, and hence
$(\gamma^2 / 2) f^{(2)}(\gamma) = \gamma^2 / 2d + \order(\gamma^4 / d)$, which
yields the result. \qed